\def\final{}
\documentclass[11pt]{airesearch}

\usepackage[T1]{fontenc}
\usepackage{XCharter}

\usetikzlibrary{fit}

\ifdefined\final
  \usepackage[disable]{todonotes}
\else
  \usepackage[textsize=tiny]{todonotes}
\fi

\newcommand{\R}{\mathbb{R}}
\newcommand{\Norm}{\operatorname{RMSNorm}}
\newcommand{\Merge}{\operatorname{Merge}}
\newcommand{\Attn}{\operatorname{Attn}}
\newcommand{\FFN}{\operatorname{FFN}}
\newcommand{\softmax}{\operatorname{softmax}}
\newcommand{\sg}{\operatorname{stopgrad}}

\hypersetup{pdftitle={Recurrent Looped Transformer},pdfauthor={Yifan Zhang, Jichen Feng, Shihan Qin, Sun-Yuan Kung, Mengdi Wang}}
\usepackage{titling}
\posttitle{\par\end{center}\vskip .2em}
\postauthor{\end{tabular}\par\end{center}\vskip -.6em}
\postdate{\par\end{center}\vskip -.5em}
\title{\huge\bfseries Recurrent Looped Transformer}
\author{\textbf{Yifan Zhang$^1$~~~~Jichen Feng$^2$~~~~Shihan Qin$^2$}\\[0.5mm]\textbf{Sun-Yuan Kung$^1$~~~~Mengdi Wang$^1$}\\[1.5mm]$^1$Princeton University~~~~$^2$University of Pennsylvania}
\date{}

\begin{document}
\maketitle

\begin{abstract}
State tracking requires an update at every input, but the depth a Transformer applies to each token is fixed regardless of sequence length.
We introduce the Recurrent Looped Transformer (RLT), which splits its layers between a parallel causal encoder and a recurrent decoder.
At each token, the decoder merges the encoder output with the previous token's final decoder state, so the computation path grows with sequence length at a fixed per-token cost.
On six algorithmic tasks, we compare five splits of eight layers with an eight-layer Transformer over three seeds.
Trained on at most 40 bits, two RLT splits generalize parity to 256 bits with 100\% accuracy in every seed, while the Transformer stays at chance.
On swap-based $S_5$ permutation tracking at eight times the training length, RLT reaches 97\% final-state accuracy versus under 1\% for the Transformer, and accuracy increases with decoder depth.
On modular arithmetic beyond the training lengths, RLT reaches up to 93\% versus 33\% for the Transformer.
Ablations show that these gains depend on the feedback: removing it drops parity and swap-based $S_5$ to chance at every split.
Updating the feedback once per four-token chunk lets known tokens in a chunk run in parallel and keeps 64-bit parity at 99\%, while permutation tracking depends on per-token feedback: chunking lowers length-64 swap-based $S_5$ from 100\% to 20\%.
\end{abstract}
\projectpage{https://github.com/yifanzhang-pro/recurrent-looped-tranformer}
\vspace{1ex}

% Keep Figure 1 on the title page; begin the introduction on page two.
\begin{center}
\begin{minipage}{\textwidth}
\captionsetup{type=figure}
\centering
% Native TikZ: local colors/styles do not alter the document theme.
\begingroup
\definecolor{rltInk}{HTML}{172B4D}
\definecolor{rltBlue}{HTML}{007AFF}
\definecolor{rltOrange}{HTML}{E87500}
\definecolor{rltGray}{HTML}{667085}
\definecolor{rltLine}{HTML}{DCE3EB}
\resizebox{\linewidth}{!}{%
\begin{tikzpicture}[x=1cm,y=1cm,>=Latex,font=\sffamily\small,
 b/.style={draw=rltBlue!65,fill=white,rounded corners=2pt,minimum height=.49cm,inner sep=4pt,text=rltInk,align=center},
 d/.style={draw=rltOrange!75,fill=rltOrange!7,rounded corners=3pt,minimum height=.7cm,minimum width=1.4cm,inner sep=4pt,text=rltInk,align=center},
 flow/.style={->,line width=.75pt,draw=rltBlue},
 state/.style={->,line width=1.8pt,draw=rltOrange},
 muted/.style={text=rltGray,font=\sffamily\scriptsize}]
\path[use as bounding box] (0,0) rectangle (16,6.5);
\fill[rounded corners=6pt,rltBlue!4] (0,.56) rectangle (5.75,6.5);
\fill[rounded corners=6pt,rltOrange!3] (5.95,.56) rectangle (16,6.5);
\node[anchor=west,font=\sffamily\bfseries\large,text=rltBlue] at (.25,6.1) {Causal encoder};
\node[anchor=west,muted] at (.25,5.65) {Parallel causal encoding of observed tokens};
\node[anchor=west,font=\sffamily\bfseries\large,text=rltOrange] at (6.2,6.1) {Recurrent decoder};
\node[anchor=west,muted] at (6.2,5.65) {Prompt and response use the same transition};
% Causal encoder computes features independently of the decoder state.
\foreach \xx/\lab in {.65/{x_1},1.75/{x_2},4.95/{x_T}} {
 \node[b,minimum width=.72cm] (p\xx) at (\xx,4.95) {$\lab$};
 \draw[flow] (\xx,4.67)--(\xx,4.3);
}
\node[text=rltGray] at (3.25,4.95) {$\cdots$};
\node[b,fill=rltBlue!9,minimum width=5.05cm,minimum height=.73cm] (enc) at (2.85,3.92) {\textbf{Causal encoder}\quad $E_\theta$};
\node[muted] at (2.85,3.29) {$e_1,\ e_2,\ \ldots,\ e_T$};
\draw[flow] (2.85,3.02)--(2.85,2.64);
\node[b,minimum width=5.05cm,minimum height=.76cm] (mem) at (2.85,2.22) {\textbf{Encoder KV memory}\quad $M^E_{\le T}$};
\node[muted,align=center] at (2.85,1.18) {Encoder memory is available\\before decoder replay};
% Last two prompt updates followed by one response update; earlier states omitted.
\foreach \xx/\t/\out [count=\i] in {7.35/{T-1}/T,10.55/T/{T+1},13.75/{T+1}/{T+2}} {
 \node[b,minimum width=1.16cm] (e\i) at (\xx,4.55) {$e_{\t}$};
 \node[muted] at (\xx,4.98) {\ifnum\i=1 prompt\else\ifnum\i=2 last prompt token\else response token\fi\fi};
 \node[circle,draw=rltOrange,line width=.8pt,fill=white,minimum size=.42cm,inner sep=0pt,text=rltOrange] (m\i) at (\xx,3.66) {$\oplus$};
 \node[d] (d\i) at (\xx,2.65) {$D_\phi$\\[-1pt]{\scriptsize $L_D$ layers}};
 \draw[flow] (e\i)--(m\i);
 \draw[state] (m\i)--(d\i);
 \node[muted] at (\xx,1.79) {predict $x_{\out}$};
 \draw[->,draw=rltGray!70] (d\i.south)--(\xx,2.02);
}
\draw[state] (6.15,3.66)--node[above,text=rltOrange,font=\small]{$H_{T-2}$}(m1.west);
% Route recurrent-state arrows above the decoder blocks.
\draw[state] (d1.east)--(8.75,2.65)--(8.75,3.66)--node[above,text=rltOrange,font=\small]{$H_{T-1}$}(m2.west);
\draw[state] (d2.east)--(11.95,2.65)--(11.95,3.66)--node[above,text=rltOrange,font=\small]{$H_T$}(m3.west);
\draw[state] (d3.east)--(15.25,2.65)--(15.25,3.66)--(15.85,3.66);
\node[text=rltOrange] at (15.65,4.02) {$\cdots$};
% A memory bus, prefix-restricted separately at every decoder query.
\draw[draw=rltBlue,line width=.7pt] (mem.south)--(2.85,1.52)--(5.82,1.52)--(5.82,.82)--(14.75,.82);
\foreach \xx/\side [count=\i] in {7.35/6.42,10.55/9.62,13.75/12.82} {
 \draw[flow] (\side,.82)--(\side,2.3)--(d\i.south west);
}
\node[fill=rltOrange!3,text=rltBlue,font=\scriptsize,inner sep=2pt] at (10.9,.82) {Read encoder KV up to the current token};
\node[anchor=west,muted] at (.18,.18) {\textcolor{rltBlue}{\rule{.35cm}{1.4pt}}\ Context KV};
\node[anchor=west,muted] at (2.6,.18) {\textcolor{rltOrange}{\rule{.35cm}{2pt}}\ State + SWA KV};
\node[anchor=east,font=\sffamily\bfseries\fontsize{8.5}{10}\selectfont,text=rltInk] at (15.85,.18) {Merge uses $s$; SWA uses layerwise cached KV.};
\end{tikzpicture}%
}
\endgroup
\caption{\textbf{Recurrence across prompt and response.}
The encoder builds causal key--value (KV) memory; the decoder processes tokens in order.
Orange arrows carry the decoder state $H_t=(s_t,C_t^D)$: the merge reads $s_t$, and each SWA layer reads its own cache.
Blue arrows supply encoder memory up to the current position.
The drawing shows the last two prompt updates and the first response update.}
\label{fig:architecture}
\end{minipage}
\end{center}
\clearpage

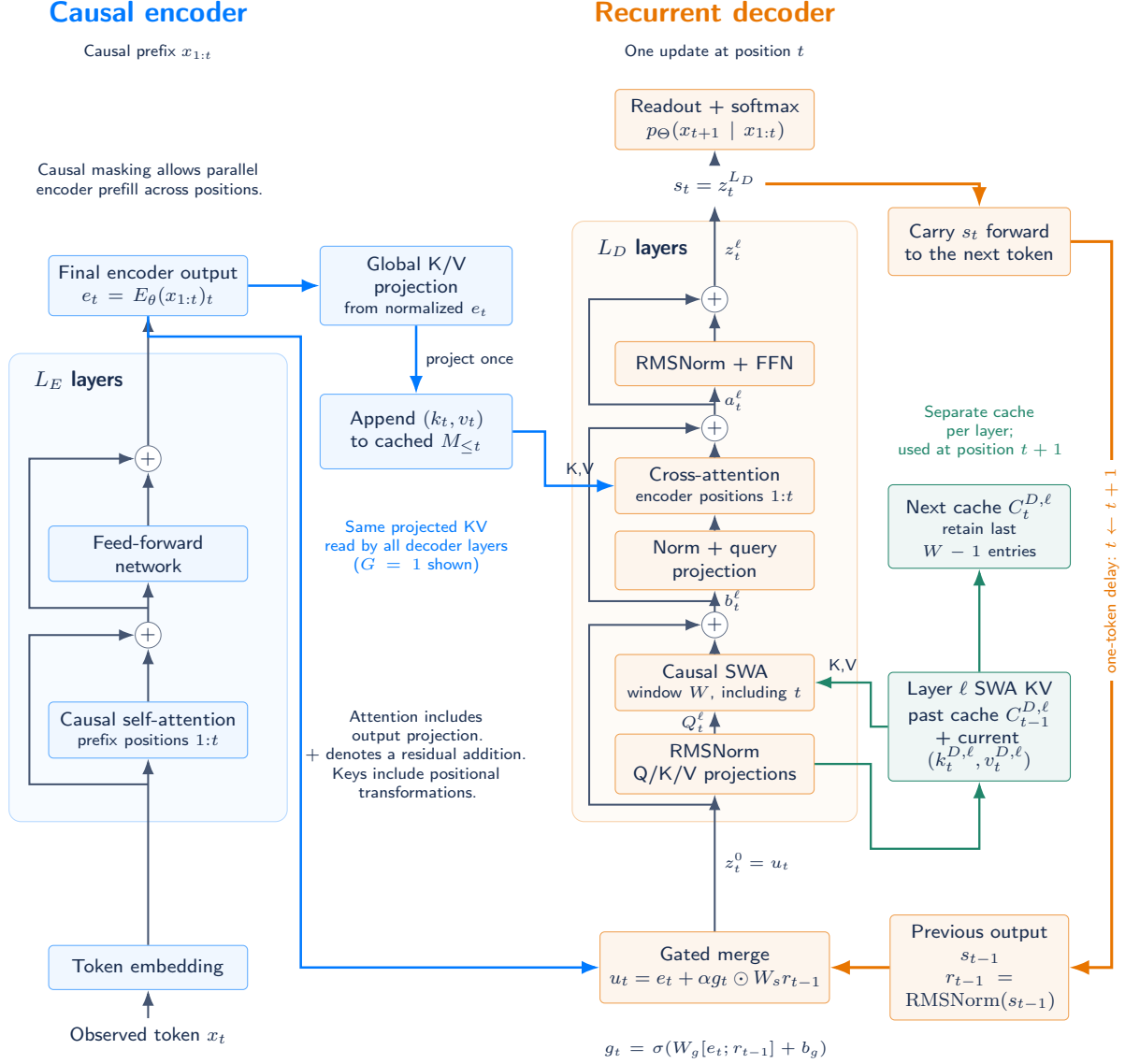
\begin{figure}[p]
\centering
% One token update; shared global memory (G=1). All geometry is in cm.
\begingroup
\definecolor{detailInk}{HTML}{172B4D}
\definecolor{detailBlue}{HTML}{007AFF}
\definecolor{detailOrange}{HTML}{E87500}
\definecolor{detailGreen}{HTML}{20856A}
\resizebox{\linewidth}{!}{%
\begin{tikzpicture}[x=1cm,y=1cm,>=Latex,font=\sffamily\fontsize{8}{9.5}\selectfont,text=detailInk,
 box/.style={draw=detailInk!45,fill=white,rounded corners=2pt,minimum height=.65cm,text width=2.65cm,align=center,inner sep=4pt},
 enc/.style={box,draw=detailBlue!65,fill=detailBlue!6},
 dec/.style={box,draw=detailOrange!70,fill=detailOrange!7},
 mem/.style={box,draw=detailGreen,fill=detailGreen!7},
 add/.style={circle,draw=detailInk!65,fill=white,minimum size=.38cm,inner sep=0pt},
 flow/.style={->,line width=.8pt,draw=detailInk!80},
 global/.style={->,line width=1pt,draw=detailBlue},
 recurrent/.style={->,line width=1.2pt,draw=detailOrange},
 local/.style={->,line width=.9pt,draw=detailGreen},
 note/.style={font=\sffamily\fontsize{6.8}{8}\selectfont,align=center}]
\path[use as bounding box] (-.2,-.25) rectangle (17,16.3);
% Stack outlines are drawn first, behind their contents.
\filldraw[rounded corners=5pt,fill=detailBlue!3,draw=detailBlue!30] (.25,3.8) rectangle (4.35,10.7);
\filldraw[rounded corners=5pt,fill=detailOrange!3,draw=detailOrange!35] (8.55,3.8) rectangle (12.75,12.65);
\node[font=\sffamily\bfseries\fontsize{12}{14}\selectfont,text=detailBlue] at (2.3,15.75) {Causal encoder};
\node[font=\sffamily\bfseries\fontsize{12}{14}\selectfont,text=detailOrange] at (10.65,15.75) {Recurrent decoder};
\node[note] at (2.3,15.15) {Causal prefix $x_{1:t}$};
\node[note] at (10.65,15.15) {One update at position $t$};
\node[anchor=west,font=\sffamily\bfseries\fontsize{9}{11}\selectfont] at (.48,10.3) {$L_E$ layers};
\node[anchor=west,font=\sffamily\bfseries\fontsize{9}{11}\selectfont] at (8.78,12.25) {$L_D$ layers};
% Encoder, bottom to top.
\node (token) at (2.3,.65) {Observed token $x_t$};
\node[enc] (embed) at (2.3,1.65) {Token embedding};
\node[enc] (esa) at (2.3,5.15) {Causal self-attention\\[-1pt]{\fontsize{6.8}{8}\selectfont prefix positions $1{:}t$}};
\node[add] (eadd1) at (2.3,6.55) {$+$};
\node[enc] (effn) at (2.3,7.75) {Feed-forward network};
\node[add] (eadd2) at (2.3,9.15) {$+$};
\node[enc] (feature) at (2.3,11.7) {Final encoder output\\$e_t=E_\theta(x_{1:t})_t$};
\draw[flow] (token)--(embed);
\draw[flow] (embed)--(esa);
\draw[flow] (esa)--(eadd1);
\draw[flow] (eadd1)--(effn);
\draw[flow] (effn)--(eadd2);
\draw[flow] (eadd2)--(feature);
\draw[flow] (2.3,4.35)--(.55,4.35)|-(eadd1);
\draw[flow] (2.3,6.95)--(.55,6.95)|-(eadd2);
\node[note,text width=3.3cm] at (2.3,13.25) {Causal masking allows parallel\\encoder prefill across positions.};
% Projected memory lives between stacks.
\node[enc,text width=2.55cm] (proj) at (6.25,11.7) {Global K/V projection\\[-1pt]{\fontsize{6.8}{8}\selectfont from normalized $e_t$}};
\node[enc,text width=2.55cm,minimum height=1.1cm] (memory) at (6.25,9.55) {Append $(k_t,v_t)$\\to cached $M_{\le t}$};
\draw[global] (feature)--(proj);
\draw[global] (proj)--node[right,note] {project once}(memory);
\node[note,text width=2.7cm,text=detailBlue] at (6.25,7.85) {Same projected KV\\read by all decoder layers\\($G=1$ shown)};
% Merge and recurrent input.
\node[dec,text width=3.1cm,minimum height=1.05cm] (merge) at (10.65,1.65) {Gated merge\\$u_t=e_t+\alpha g_t\odot W_s r_{t-1}$};
\node[dec,text width=2.35cm] (prev) at (14.55,1.65) {Previous output\\$s_{t-1}$\\$r_{t-1}=\mathrm{RMSNorm}(s_{t-1})$};
\draw[recurrent] (prev)--(merge);
\draw[global] (feature.south)--(2.3,10.95)--(4.55,10.95)--(4.55,1.65)--(merge.west);
\node[note,text width=3.6cm] at (10.65,.45) {$g_t=\sigma(W_g[e_t;r_{t-1}]+b_g)$};
% Decoder: three pre-normalized residual sublayers.
\node[dec] (sqkv) at (10.65,4.65) {RMSNorm\\Q/K/V projections};
\node[dec] (swa) at (10.65,5.85) {Causal SWA\\[-1pt]{\fontsize{6.8}{8}\selectfont window $W$, including $t$}};
\node[add] (dadd1) at (10.65,6.7) {$+$};
\node[dec] (mq) at (10.65,7.65) {Norm + query projection};
\node[dec] (cross) at (10.65,8.75) {Cross-attention\\[-1pt]{\fontsize{6.8}{8}\selectfont encoder positions $1{:}t$}};
\node[add] (dadd2) at (10.65,9.6) {$+$};
\node[dec] (dffn) at (10.65,10.55) {RMSNorm + FFN};
\node[add] (dadd3) at (10.65,11.5) {$+$};
\draw[flow] (merge)--node[right,note] {$z_t^0=u_t$}(sqkv);
\draw[flow] (sqkv)--node[left,note] {$Q_t^\ell$}(swa);
\draw[flow] (swa)--(dadd1);
\draw[flow] (dadd1)--node[right,note] {$b_t^\ell$}(mq);
\draw[flow] (mq)--(cross);
\draw[flow] (cross)--(dadd2);
\draw[flow] (dadd2)--node[right,note] {$a_t^\ell$}(dffn);
\draw[flow] (dffn)--(dadd3);
\draw[flow] (10.65,4.05)--(8.8,4.05)|-(dadd1);
\draw[flow] (10.65,7.05)--(8.8,7.05)|-(dadd2);
\draw[flow] (10.65,9.95)--(8.8,9.95)|-(dadd3);
\draw[global] (memory.east)--(8.15,9.55)--(8.15,8.75)--node[above,note] {K,V}(cross.west);
% Local state is independent at every layer.
\node[mem,text width=2.4cm,minimum height=1.15cm] (cache) at (14.55,5.2) {Layer $\ell$ SWA KV\\past cache $C_{t-1}^{D,\ell}$\\+ current $(k_t^{D,\ell},v_t^{D,\ell})$};
\draw[local] (sqkv.east)--(12.95,4.65)--(12.95,3.4)--(14.55,3.4)--(cache.south);
\draw[local] (cache.west)--(12.95,5.2)--(12.95,5.85)--node[above,note] {K,V}(swa.east);
\node[mem,text width=2.4cm,minimum height=1cm] (nextcache) at (14.55,8.15) {Next cache $C_t^{D,\ell}$\\[-1pt]{\fontsize{6.8}{8}\selectfont retain last $W-1$ entries}};
\draw[local] (cache)--(nextcache);
\node[note,text width=2.5cm,text=detailGreen] at (14.55,9.55) {Separate cache per layer;\\used at position $t+1$};
% Output and delay path back to the next token's merge.
\node (state) at (10.65,13.2) {$s_t=z_t^{L_D}$};
\node[dec] (head) at (10.65,14.15) {Readout + softmax\\$p_\Theta(x_{t+1}\mid x_{1:t})$};
\draw[flow] (dadd3)--node[right,note,pos=.45] {$z_t^\ell$}(state);
\draw[flow] (state)--(head);
\node[dec,text width=2.4cm,minimum height=1cm] (delay) at (14.55,12.35) {Carry $s_t$ forward\\to the next token};
\draw[recurrent] (state.east)-|(delay.north);
\draw[recurrent] (delay.east)--(16.5,12.35)--(16.5,1.65)--(prev.east);
\node[note,text=detailOrange,rotate=90,fill=white,inner sep=3pt] at (16.5,7.5) {one-token delay: $t\leftarrow t+1$};
\node[note,text width=3.4cm] at (6.25,4.8) {Attention includes output projection.\\$+$ denotes a residual addition.\\Keys include positional\\transformations.};
\end{tikzpicture}%
}
\endgroup
\caption{\textbf{RLT-1 architecture for one token update.}
Computation flows upward; each outlined stack repeats for the indicated depth.
The encoder output supplies both the gated merge and the global KV projection.
With $G=1$, all decoder layers use their own queries to read the same cached, projected global KV.
Each decoder layer constructs separate SWA KV from its own input, attends to its local window, then reads global memory and applies an FFN.
Residual paths bypass each attention or FFN sublayer.
The final decoder output $s_t$ feeds the next token's merge, while each layer retains its own SWA cache.
Blue arrows carry encoder features and global memory, green arrows carry local KV, and orange arrows carry recurrent hidden states.}
\label{fig:architecture-detail}
\end{figure}
\clearpage

\section{Introduction}
Algorithmic state tracking requires a model to update a state with each input: parity accumulates bits, and permutation tracking composes group elements.
A Transformer computes each position with a fixed number of layers, and unless $\mathsf{TC}^0=\mathsf{NC}^1$, fixed-depth Transformers cannot track compositions of $S_5$ permutations over arbitrarily long sequences \citep{merrill2024illusion}.
A recurrent network instead passes the result of each update to the next, so its computation path grows with the input.
Adding this connection to a Transformer lets later tokens build on earlier computation along a path that grows with sequence length.

We introduce the Recurrent Looped Transformer (RLT), which divides its layers between a causal encoder and a recurrent Transformer decoder (Figure~\ref{fig:architecture}).
The encoder processes known tokens in parallel and builds token representations and global key--value memory.
At each position, a gated merge combines the current encoder representation with the previous token's final decoder output.
The decoder then attends to the encoder prefix and its own recent activations before predicting the next token.
Prompt and response tokens follow the same transition.
With decoder depth $L_D$, processing $t$ tokens creates a recurrent path through $tL_D$ decoder blocks, while each token passes through only $L_E+L_D$ layers.

This design introduces two choices that a decoder-only Transformer lacks.
The first is the split between encoder and decoder depth: at a fixed total layer count, moving layers to the decoder lengthens the recurrent path but adds decoder work that follows token order during training and prefill.
The second is the feedback interval $B$, the number of tokens between feedback updates.
RLT-1 feeds back at every token ($B=1$), RLT-2 shares one feedback state across each chunk of $B$ tokens, and RLT-0 removes the feedback ($B=\infty$).
For $T$ known tokens, the decoder needs $TL_D$ sequential block stages in RLT-1, $\lceil T/B\rceil L_D$ in RLT-2, and $L_D$ in RLT-0 (Table~\ref{tab:rlt-efficiency}); RLT-2 uses the same parameters as RLT-1.

We compare five allocations of eight layers with an eight-layer decoder-only Transformer on six algorithmic tasks: addition, parity, modular arithmetic with and without brackets, and standard and swap-based $S_5$ state tracking.
All comparisons use three initialization seeds and shared training and test data, and evaluate lengths well beyond the training range.
Our contributions are as follows.
\begin{itemize}
\item \textbf{Architecture.} RLT combines encoder-memory reuse \citep{sun2024yoco} with feedback through the full decoder at every token.
The appendices specify training by backpropagation through the full recurrent history, cache reuse across turns, and exact policy replay.
\item \textbf{Length generalization.} Trained on at most 40 bits, RLT splits $5+3$ and $7+1$ reach $100\pm0\%$ parity accuracy at 256 bits in every seed, compared with $50.07\pm1.63\%$ for the Transformer.
On swap-based $S_5$ at 256 operations, eight times the training length, $4+4$ reaches $97.30\pm2.76\%$ final-state accuracy, compared with $0.85\pm0.30\%$.
After 5,000 training steps, $6+2$ reaches $93.36\pm5.69\%$ on flat mod-5 expressions of length 63, compared with $33.20\pm2.33\%$ (Section~\ref{sec:length-generalization}).
\item \textbf{Depth allocation.} The best split depends on the task: parity generalizes best with one or three decoder layers, whereas swap-based $S_5$ accuracy at 256 operations rises with decoder depth, from chance with one decoder layer to 97\% with four.
\item \textbf{Feedback interval.} Without feedback, parity and swap-based $S_5$ stay near chance from length 64 at every split.
With four-token chunks, RLT-2 keeps 64-bit parity at $98.99\pm1.66\%$ but lowers swap-based $S_5$ at 64 operations from 100\% to $19.60\pm6.54\%$, so permutation tracking benefits most from per-token feedback.
\end{itemize}

Feedback Transformer, Recurrent Transformer, Full-bandwidth Transformer, T$^2$MLR, and Latent Recurrent Transformer also pass information across tokens through feedback \citep{fan2020feedback,oncescu2026recurrent,wang2026fullbandwidth,cai2026t2mlr,huang2026latentrecurrent}.
RLT feeds back through the full decoder, keeps this recurrence separate from parallel encoder memory, and treats the allocation of depth between encoding and recurrent updates as a design variable (Section~\ref{sec:related-work}).

\section{Method}
This section defines RLT-1 and two controls that change its feedback interval.
\subsection{Sequence and state}
Let $x_{1:S}$ be an independent sequence beginning with $x_1=\mathrm{BOS}$.
At inference, $x_{1:T}$ is the prompt and later tokens form the response.
The same conditional distribution applies before and after $T$.
Let $L_E$ and $L_D$ be the encoder and decoder depths, and let $d$ be the residual width.
We write the encoder representation $e_t\in\R^d$ and recurrent output $s_t\in\R^d$ as column vectors.

The complete decoder state is $H_t=(s_t,C_t^D)$.
The cache $C_t^D$ holds the key/value projections retained by each decoder SWA layer.
The window size $W\ge1$ includes the current token, so each layer retains at most $W-1$ past positions after an update.
We collect all parameters, including the learned initial state $s_\star$, in $\Theta$.
The encoder and decoder are denoted by $E_\theta$ and $D_\phi$.

\subsection{Causal encoder and memory}
For an observed prefix, compute
\begin{equation}
 e_{1:T}=E_\theta(x_{1:T}). \label{eq:encoder}
\end{equation}
Positions within each encoder layer can be processed together using a causal mask.
For decoder memory group $g\in\{1,\ldots,G\}$, construct
\begin{equation}
 k_t^g=\mathcal P_K^g(e_t,t),\qquad
 v_t^g=W_V^g\Norm_E(e_t),\qquad
 M_{\le t}^g=\{(k_j^g,v_j^g)\}_{j=1}^t. \label{eq:memory}
\end{equation}
The key map includes normalization, projection, and any positional transformation.
Decoder layer $\ell$ reads group $g(\ell)$: $G=1$ shares memory across layers, while $G=L_D$ allows separate projections for each layer.
Separate attention sublayers read the two stores: cross-attention reads encoder memory $M_{\le t}$ over the full prefix, and causal SWA reads decoder KV within a bounded window.

\subsection{One transition for every token}
Initialize once, before BOS:
\begin{equation}
 H_0=(s_\star,\varnothing). \label{eq:init}
\end{equation}
For every observed or sampled token, apply
\begin{align}
 u_t&=\Merge(e_t,s_{t-1}),\\
 H_t=(s_t,C_t^D)&=D_\phi(u_t;M_{\le t},C_{t-1}^D,t),
 \qquad t\ge1, \label{eq:recurrence}\\
 p_\Theta(x_{t+1}\mid x_{1:t})
 &=\softmax\bigl(W_o\Norm_o(s_t)\bigr)_{x_{t+1}}.
 \label{eq:distribution}
\end{align}
Define $F_t(H)=D_\phi(\Merge(e_t,s);M_{\le t},C^D,t)$ for
$H=(s,C^D)$.
After the encoder features are computed, the decoder processes every prompt position to predict the first response token:
\begin{equation}
 H_T=F_T\circ F_{T-1}\circ\cdots\circ F_1(H_0).
 \label{eq:prompt-state}
\end{equation}
Generation continues from both components of $H_T$.
After sampling $x_{T+1}$ from $s_T$, encode it incrementally, append its encoder-derived KV, and compute $H_{T+1}=F_{T+1}(H_T)$, including every decoder SWA cache update.

\paragraph{RNN and state-space interpretations.}
RLT-1 can be viewed as a recurrent neural network (RNN) whose cell is the gated merge followed by the full Transformer decoder.
The cell reuses its parameters across tokens, while its decoder layers may have distinct parameters.
For a decoder state $H=(s,C^D)$, encoder representation $e$, and encoder memory $M$, define the state-transition function and output function by
\begin{align}
 \mathcal F_\Theta(H;e,M,t)
 &:=D_\phi\bigl(\Merge(e,s);M,C^D,t\bigr),
 \label{eq:ssm-transition-function}\\
 \mathcal G_\Theta(H)
 &:=\softmax\bigl(W_o\Norm_o(s)\bigr).
 \label{eq:ssm-output-function}
\end{align}
The nonlinear state-space equations are then
\begin{align}
 H_t&=\mathcal F_\Theta(H_{t-1};e_t,M_{\le t},t),
 \label{eq:ssm-state-update}\\
 p_\Theta(\,\cdot\mid x_{1:t})&=\mathcal G_\Theta(H_t).
 \label{eq:ssm-readout}
\end{align}
The transition updates both the feedback vector $s_t$ and the decoder cache $C_t^D$; the output function maps $s_t$ to a distribution over the next token.
The encoder supplies $e_t$ and $M_{\le t}$ as conditioning inputs, so the earlier notation is $F_t(H)=\mathcal F_\Theta(H;e_t,M_{\le t},t)$.
At fixed width, decoder depth, and SWA window, this decoder state has bounded size, while the encoder memory $M_{\le t}$ grows with sequence length.

The state-space model (SSM) cores of S4 and Mamba use updates that are affine in the previous state once the layer inputs are fixed; Mamba makes the coefficients input-dependent \citep{gu2021s4,gu2023mamba}.
S4 exploits time invariance for convolutional evaluation, while Mamba uses parallel scans of its input-dependent affine updates.
RLT-1's gated merge and decoder instead depend nonlinearly on the previous decoder state, so these affine scan algorithms do not directly apply to its transition.

\subsection{Prompt prefill and incremental decoding}
Prefill first computes the causal encoder representations and memory of the prompt, then evaluates $H_1,\ldots,H_T$ in order.
At decoder position $t$, attention is restricted to $M_{\le t}$ even though the entire prompt memory is available.
Generation uses
\begin{equation}
 (e_t,C_t^E)=E_\theta^{\mathrm{step}}(x_t,C_{t-1}^E)
\end{equation}
with encoder cache $C^E$, followed by a memory append and decoder update.
Observed tokens may be encoded in chunks if the encoder cache is preserved.
The decoder still processes every token in order: each update produces the hidden state and SWA KV needed by later positions.

\subsection{Merge and decoder blocks}
We use the following gated merge:
\begin{align}
 r_{t-1}&=\Norm_s(s_{t-1}),\\
 g_t&=\sigma\bigl(W_g[e_t;r_{t-1}]+b_g\bigr),\\
 u_t&=e_t+\alpha\,g_t\odot W_sr_{t-1}. \label{eq:merge}
\end{align}
Here $W_g\in\R^{d\times2d}$, $W_s\in\R^{d\times d}$, and $\alpha$ controls the feedback scale.
The main comparison uses $\alpha=0.1$.

Starting from $z_t^0=u_t$, each decoder block applies causal SWA, encoder-memory cross-attention, and a feed-forward network (FFN):
\begin{align}
 q_t^{D,\ell}&=\mathcal P_Q^{D,\ell}(z_t^{\ell-1},t),\\
 k_t^{D,\ell}&=\mathcal P_K^{D,\ell}(z_t^{\ell-1},t),\qquad
 v_t^{D,\ell}=W_V^{D,\ell}\Norm_{S,\ell}(z_t^{\ell-1}),\\
 b_t^\ell&=z_t^{\ell-1}
 +\Attn_\ell^D\!\left(q_t^{D,\ell},
 \{(k_j^{D,\ell},v_j^{D,\ell})\}_{j=\max(1,t-W+1)}^t\right),\\
 a_t^\ell&=b_t^\ell+
 \Attn_\ell^M\!\left(\mathcal P_Q^{M,\ell}(b_t^\ell,t),
 M_{\le t}^{g(\ell)}\right),\\
 z_t^\ell&=a_t^\ell+\FFN_\ell(\Norm_{D,\ell}(a_t^\ell)),
 \qquad s_t=z_t^{L_D}. \label{eq:block}
\end{align}
The attention operators include output projections; query/key maps include their normalizations and positional transformations.
Each layer forms current KV from its input before SWA, so attention to the current position introduces no circular dependency.
Historical decoder KV comes from $C_{t-1}^D$.
After the update, retain positions $\max(1,t-W+2),\ldots,t$ in $C_t^D$; this set is empty for $W=1$.
Figure~\ref{fig:architecture-detail} expands the encoder and decoder stacks for one token update.

\subsection{Recurrent depth and execution cost}
After a prompt of length $T$ and $n$ processed response tokens, the feedback path has passed through $(T+n)L_D$ decoder blocks, while each token evaluates $L_E+L_D$ blocks (Figure~\ref{fig:temporal-depth}).
The influence of earlier states depends on the gates, projections, and products of transition Jacobians.
A prompt of length $T$ requires $T$ sequential decoder updates; updates from independent sequences can share a batch, with separate states and caches.
Appendix~\ref{app:execution-details} derives the work and storage costs and proves that states and next-token distributions are invariant to the prompt--response split.

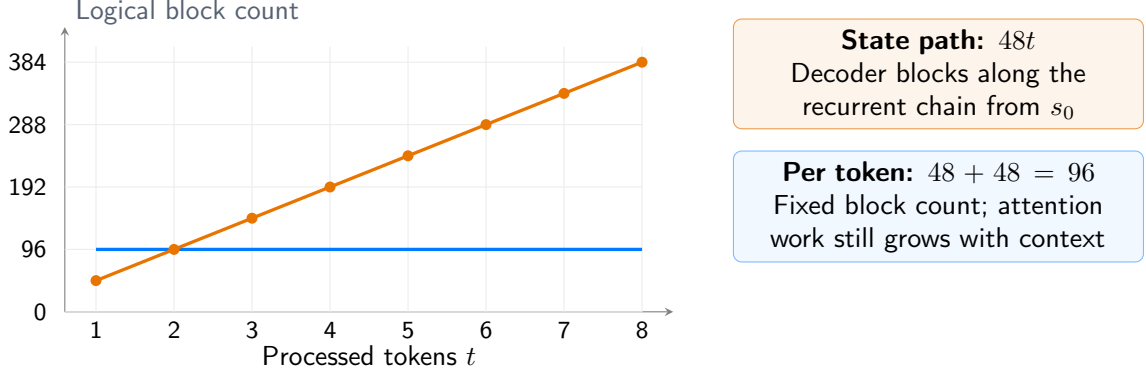
\begin{figure}[t]
\centering
\begingroup
\definecolor{rblue}{HTML}{007AFF}
\definecolor{rorange}{HTML}{E87500}
\definecolor{rgray}{HTML}{566172}
\resizebox{\linewidth}{!}{%
\begin{tikzpicture}[font=\sffamily\small,>=stealth,
box/.style={draw=rblue!50,fill=rblue!5,rounded corners=3pt,align=center,minimum height=.7cm},
state/.style={box,draw=rorange!70,fill=rorange!8}]

\path[use as bounding box] (0,0) rectangle (16,6.1);
\node[anchor=west,font=\sffamily\bfseries] at (0,5.8) {Block count per token and along the recurrent path};
\node[anchor=west,text=rgray] at (0,5.25) {Illustrative 48-layer encoder + 48-layer decoder; counts exclude merge and readout.};
\draw[->,gray] (1.6,1) -- (9.4,1) ;
\node at (5.5,.45) {Processed tokens $t$};
\draw[->,gray] (1.6,1) -- (1.6,4.65);
\node[anchor=west,text=rgray] at (1.6,4.85) {Logical block count};
\foreach \t/\x in {1/2,2/3,3/4,4/5,5/6,6/7,7/8,8/9}{
\draw[gray!15] (\x,1)--(\x,4.4);
\node[below,font=\sffamily\footnotesize] at (\x,1) {\t};}
\foreach \v/\y in {0/1,96/1.8,192/2.6,288/3.4,384/4.2}{
\draw[gray!15] (1.6,\y)--(9,\y);
\node[left,font=\sffamily\footnotesize] at (1.5,\y) {\v};}
\draw[rblue,very thick] (2,1.8)--(9,1.8);
\draw[rorange,very thick] (2,1.4)--(9,4.2);
\foreach \x/\y in {2/1.4,3/1.8,4/2.2,5/2.6,6/3,7/3.4,8/3.8,9/4.2}{\fill[rorange] (\x,\y) circle (2pt);}
\node[state,text width=5cm] at (12.8,4.05) {\textbf{State path: $48t$}\\Decoder blocks along the recurrent chain from $s_0$};
\node[box,text width=5cm] at (12.8,2.35) {\textbf{Per token: $48+48=96$}\\Fixed block count; attention work still grows with context};
\end{tikzpicture}}
\endgroup
\caption{\textbf{Recurrent path length and block count per token.}
For $L_E=L_D=48$, the path traverses $48t$ decoder blocks after $t$ tokens.
Each token evaluates 96 encoder and decoder blocks.
Orange counts decoder blocks along the recurrent path; blue counts blocks in both stacks per token.}
\label{fig:temporal-depth}
\end{figure}

\subsection{Feedback interval: RLT-1, RLT-2, and RLT-0}
\label{sec:feedback-variants}
RLT-1 and its two controls differ in one quantity, the feedback interval $B$: the number of tokens between updates of the state that enters the decoder input.
RLT-1 updates this state at every token ($B=1$), RLT-2 once per chunk of $B$ tokens, and RLT-0 never ($B=\infty$).
All three keep the causal encoder, encoder memory, decoder SWA, prefix-restricted memory attention, and readout of RLT-1.

\emph{RLT-2} holds the feedback state fixed within a chunk of $B$ tokens and updates it at the chunk boundary (Figure~\ref{fig:architecture-chunk}).
Boundaries are anchored at BOS, so chunk $k$ covers the known positions $I_k=\{(k-1)B+1,\ldots,\min(kB,T)\}$.
Starting from the feedback register $h_0=s_\star$, every position in chunk $k$ merges its encoder output with the same boundary state through its own gate:
\begin{align}
 r_{k-1}&=\Norm_s(h_{k-1}),\qquad
 g_t=\sigma\bigl(W_g[e_t;r_{k-1}]+b_g\bigr),\\
 z_t^0&=e_t+\alpha\,g_t\odot W_s r_{k-1},\qquad t\in I_k,
 \label{eq:chunk-merge}\\
 h_k&=y_{kB}\quad\text{when }kB\le T.
 \label{eq:chunk-transition}
\end{align}
Here $y_t=z_t^{L_D}$ is the final decoder output at position $t$; every $y_t$ predicts $x_{t+1}$, but only the last output of a complete chunk becomes the next feedback state.
The decoder applies Equation~\eqref{eq:block} to all known positions of a chunk together, and causal SWA also reads keys retained from earlier chunks.
Because $h_{k-1}$ depends only on tokens before the chunk, the causal masks keep each $y_t$ a function of $x_{1:t}$.
A partial final chunk leaves the register at the preceding complete boundary, so moving the prompt--response split does not change the computation.
With $B=1$, every token closes a chunk, and RLT-2 reduces exactly to RLT-1.

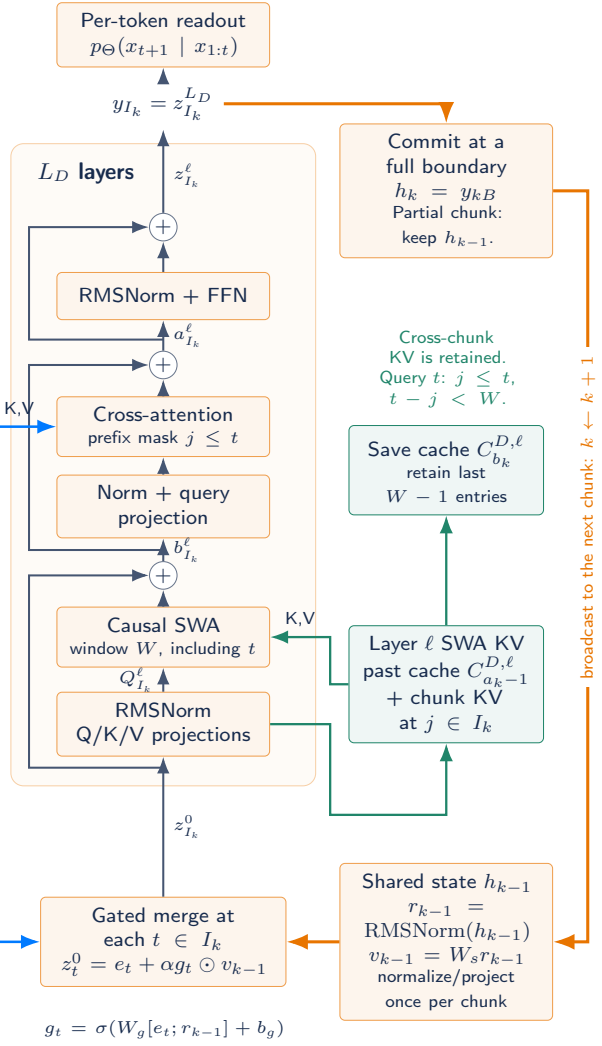
\begin{figure}[p]
\centering
% One RLT-2 chunk; parallel known positions and shared encoder memory (G=1).
\begingroup
\definecolor{detailInk}{HTML}{172B4D}
\definecolor{detailBlue}{HTML}{007AFF}
\definecolor{detailOrange}{HTML}{E87500}
\definecolor{detailGreen}{HTML}{20856A}
\resizebox{\linewidth}{!}{%
\begin{tikzpicture}[x=1cm,y=1cm,>=Latex,font=\sffamily\fontsize{8}{9.5}\selectfont,text=detailInk,
 box/.style={draw=detailInk!45,fill=white,rounded corners=2pt,minimum height=.65cm,text width=2.65cm,align=center,inner sep=4pt},
 enc/.style={box,draw=detailBlue!65,fill=detailBlue!6},
 dec/.style={box,draw=detailOrange!70,fill=detailOrange!7},
 mem/.style={box,draw=detailGreen,fill=detailGreen!7},
 add/.style={circle,draw=detailInk!65,fill=white,minimum size=.38cm,inner sep=0pt},
 flow/.style={->,line width=.8pt,draw=detailInk!80},
 global/.style={->,line width=1pt,draw=detailBlue},
 recurrent/.style={->,line width=1.2pt,draw=detailOrange},
 local/.style={->,line width=.9pt,draw=detailGreen},
 note/.style={font=\sffamily\fontsize{6.8}{8}\selectfont,align=center}]
\path[use as bounding box] (-.2,-.25) rectangle (17,16.3);
% Stack outlines are drawn first, behind their contents.
\filldraw[rounded corners=5pt,fill=detailBlue!3,draw=detailBlue!30] (.25,3.8) rectangle (4.35,10.7);
\filldraw[rounded corners=5pt,fill=detailOrange!3,draw=detailOrange!35] (8.55,3.8) rectangle (12.75,12.65);
\node[font=\sffamily\bfseries\fontsize{12}{14}\selectfont,text=detailBlue] at (2.3,15.75) {Causal encoder};
\node[font=\sffamily\bfseries\fontsize{12}{14}\selectfont,text=detailOrange] at (10.65,15.75) {RLT-2 decoder};
\node[note] at (2.3,15.15) {Known prefix through $b_k$};
\node[note] at (10.65,15.15) {Chunk $I_k$: parallel positions within each layer};
\node[anchor=west,font=\sffamily\bfseries\fontsize{9}{11}\selectfont] at (.48,10.3) {$L_E$ layers};
\node[anchor=west,font=\sffamily\bfseries\fontsize{9}{11}\selectfont] at (8.78,12.25) {$L_D$ layers};
% Encoder, bottom to top.
\node (token) at (2.3,.65) {Known chunk $x_{I_k}$};
\node[enc] (embed) at (2.3,1.65) {Token embeddings};
\node[enc] (esa) at (2.3,5.15) {Causal self-attention\\[-1pt]{\fontsize{6.8}{8}\selectfont prefix mask $j\le t$}};
\node[add] (eadd1) at (2.3,6.55) {$+$};
\node[enc] (effn) at (2.3,7.75) {Feed-forward network};
\node[add] (eadd2) at (2.3,9.15) {$+$};
\node[enc] (feature) at (2.3,11.7) {Encoder outputs\\$e_{I_k}=E_\theta(x_{1:b_k})_{I_k}$};
\draw[flow] (token)--(embed);
\draw[flow] (embed)--(esa);
\draw[flow] (esa)--(eadd1);
\draw[flow] (eadd1)--(effn);
\draw[flow] (effn)--(eadd2);
\draw[flow] (eadd2)--(feature);
\draw[flow] (2.3,4.35)--(.55,4.35)|-(eadd1);
\draw[flow] (2.3,6.95)--(.55,6.95)|-(eadd2);
\node[note,text width=3.3cm] at (2.3,13.25) {Causal masking allows parallel\\encoder prefill across positions.};
% Projected memory lives between stacks.
\node[enc,text width=2.55cm] (proj) at (6.25,11.7) {Global K/V projection\\[-1pt]{\fontsize{6.8}{8}\selectfont from normalized $e_{I_k}$}};
\node[enc,text width=2.55cm,minimum height=1.1cm] (memory) at (6.25,9.55) {Append $(k_j,v_j)$, $j\in I_k$\\to encoder memory $M_{\le b_k}$};
\draw[global] (feature)--(proj);
\draw[global] (proj)--node[right,note] {once per position}(memory);
\node[note,text width=2.7cm,text=detailBlue] at (6.25,7.85) {Same projected KV\\read by all decoder layers\\($G=1$ shown)};
% The same boundary state enters every position-specific gate in the chunk.
\node[dec,text width=3.1cm,minimum height=1.05cm] (merge) at (10.65,1.65) {Gated merge at each $t\in I_k$\\$z_t^0=e_t+\alpha g_t\odot v_{k-1}$};
\node[dec,text width=2.65cm,minimum height=1.5cm] (prev) at (14.55,1.65) {Shared state $h_{k-1}$\\$r_{k-1}=\mathrm{RMSNorm}(h_{k-1})$\\$v_{k-1}=W_s r_{k-1}$\\[-1pt]{\fontsize{6.8}{8}\selectfont normalize/project once per chunk}};
\draw[recurrent] (prev)--(merge);
\draw[global] (feature.south)--(2.3,10.95)--(4.55,10.95)--(4.55,1.65)--(merge.west);
\node[note,text width=3.6cm] at (10.65,.45) {$g_t=\sigma(W_g[e_t;r_{k-1}]+b_g)$};
% Decoder: three pre-normalized residual sublayers.
\node[dec] (sqkv) at (10.65,4.65) {RMSNorm\\Q/K/V projections};
\node[dec] (swa) at (10.65,5.85) {Causal SWA\\[-1pt]{\fontsize{6.8}{8}\selectfont window $W$, including $t$}};
\node[add] (dadd1) at (10.65,6.7) {$+$};
\node[dec] (mq) at (10.65,7.65) {Norm + query projection};
\node[dec] (cross) at (10.65,8.75) {Cross-attention\\[-1pt]{\fontsize{6.8}{8}\selectfont prefix mask $j\le t$}};
\node[add] (dadd2) at (10.65,9.6) {$+$};
\node[dec] (dffn) at (10.65,10.55) {RMSNorm + FFN};
\node[add] (dadd3) at (10.65,11.5) {$+$};
\draw[flow] (merge)--node[right,note] {$z_{I_k}^0$}(sqkv);
\draw[flow] (sqkv)--node[left,note] {$Q_{I_k}^\ell$}(swa);
\draw[flow] (swa)--(dadd1);
\draw[flow] (dadd1)--node[right,note] {$b_{I_k}^\ell$}(mq);
\draw[flow] (mq)--(cross);
\draw[flow] (cross)--(dadd2);
\draw[flow] (dadd2)--node[right,note] {$a_{I_k}^\ell$}(dffn);
\draw[flow] (dffn)--(dadd3);
\draw[flow] (10.65,4.05)--(8.8,4.05)|-(dadd1);
\draw[flow] (10.65,7.05)--(8.8,7.05)|-(dadd2);
\draw[flow] (10.65,9.95)--(8.8,9.95)|-(dadd3);
\draw[global] (memory.east)--(8.15,9.55)--(8.15,8.75)--node[above,note] {K,V}(cross.west);
% Local state is independent at every layer.
\node[mem,text width=2.4cm,minimum height=1.15cm] (cache) at (14.55,5.2) {Layer $\ell$ SWA KV\\past cache $C_{a_k-1}^{D,\ell}$\\+ chunk KV at $j\in I_k$};
\draw[local] (sqkv.east)--(12.95,4.65)--(12.95,3.4)--(14.55,3.4)--(cache.south);
\draw[local] (cache.west)--(12.95,5.2)--(12.95,5.85)--node[above,note] {K,V}(swa.east);
\node[mem,text width=2.4cm,minimum height=1cm] (nextcache) at (14.55,8.15) {Save cache $C_{b_k}^{D,\ell}$\\[-1pt]{\fontsize{6.8}{8}\selectfont retain last $W-1$ entries}};
\draw[local] (cache)--(nextcache);
\node[note,text width=2.5cm,text=detailGreen] at (14.55,9.55) {Cross-chunk KV is retained.\\Query $t$: $j\le t$, $t-j<W$.};
% All outputs drive predictions; only a complete boundary commits feedback.
\node (state) at (10.65,13.2) {$y_{I_k}=z_{I_k}^{L_D}$};
\node[dec] (head) at (10.65,14.15) {Per-token readout\\$p_\Theta(x_{t+1}\mid x_{1:t})$};
\draw[flow] (dadd3)--node[right,note,pos=.45] {$z_{I_k}^\ell$}(state);
\draw[flow] (state.north)--(head.south);
\node[dec,text width=2.65cm,minimum height=1.35cm] (delay) at (14.55,12.0) {Commit at a full boundary\\$h_k=y_{kB}$\\[-1pt]{\fontsize{6.8}{8}\selectfont Partial chunk: keep $h_{k-1}$.}};
\draw[recurrent] (state.east)-|(delay.north);
\draw[recurrent] (delay.east)--(16.5,12.0)--(16.5,1.65)--(prev.east);
\node[note,text=detailOrange,rotate=90,fill=white,inner sep=3pt] at (16.5,7.5) {broadcast to the next chunk: $k\leftarrow k+1$};
\node[note,text width=3.4cm] at (6.25,4.8) {Positions in $I_k$ run together.\\Layers remain sequential.\\$+$ denotes residual addition.\\$h_0=s_\star$ at sequence start.};
\end{tikzpicture}%
}
\endgroup
\caption{\textbf{RLT-2 architecture for one chunk.}
Computation flows upward; the stacks repeat for the indicated depths.
Each position $t\in I_k$ uses its own gate to merge its encoder output with the preceding boundary state $h_{k-1}$.
State normalization and projection are shared across the chunk.
Each decoder layer applies causal SWA, prefix-restricted encoder-memory attention, and an FFN, with residual connections around each sublayer.
For $G=1$, all decoder layers read the same projected encoder KV; each layer retains its own SWA KV across chunk boundaries.
Every output $y_t$ predicts the next token; only $y_{kB}$ at a complete boundary becomes the next feedback state.
A partial chunk preserves $h_{k-1}$ during continuation.
Blue arrows carry encoder features and memory, green arrows carry decoder KV, and orange arrows carry chunk-level feedback.}
\label{fig:architecture-chunk}
\end{figure}

\emph{RLT-0} is the $B=\infty$ limit: no chunk boundary is reached, so no decoder output is fed back.
RLT-2 with $B=\infty$ would still merge every position with the constant initial state $s_\star$; RLT-0 drops this constant merge together with $s_\star$, state normalization, the gate, and the feedback projection, and feeds $z_t^0=e_t$ to the decoder (Figure~\ref{fig:architecture-no-feedback}).

Table~\ref{tab:rlt-efficiency} and Figure~\ref{fig:chunk-schedule} compare the three intervals at matched stack dimensions.
For a known prefix of length $T$, the decoder needs $\lceil T/B\rceil L_D$ sequential block stages, with up to $\min(B,T)$ positions in each batched matrix operation: $TL_D$ stages for RLT-1 and $L_D$ for RLT-0.
All three evaluate $TL_D$ decoder blocks and generate one token per step.
Larger $B$ therefore trades feedback frequency for known-token parallelism.
Appendix~\ref{sec:rlt2} gives the chunked training, prefill, and incremental-generation procedures.

Because RLT-2 uses the same parameters for every $B$, the chunk size can also vary during training.
Larger chunks shorten the sequential decoder path, while smaller chunks extrapolate better on parity and swaps-$S_5$ (Section~\ref{sec:experiments}).
In seed-42 CPU timing at $4+4$, chunk4 training steps run 2.27 times as fast as RLT-1 steps, and RLT-0 steps 4.17--4.30 times as fast (Appendix~\ref{app:mod5-variants}).
Pretraining can therefore begin with a large $B$, processing most tokens with high parallelism, and mid-training and post-training can reduce $B$ toward $B=1$ to adapt the model to frequent feedback.
Changing $B$ changes the model's computation, so inference and RL replay use the chunk size of the final training stage.
Our experiments train each chunk size from scratch and do not test such a schedule.

\begin{table}[t]
\centering
\small
\setlength{\tabcolsep}{4pt}
\begin{tabular}{p{.5\linewidth}ccc}
\toprule
Quantity & \makecell{RLT-1\\($B=1$)} & \makecell{RLT-2\\($B$)} & \makecell{RLT-0\\($B=\infty$)} \\
\midrule
Positions per decoder batch, known tokens & $1$ & up to $B$ & $T$ \\
Decoder block stages, training forward / prefill & $TL_D$ & $KL_D$ & $L_D$ \\
Backward dependency depth, full BPTT & $O(TL_D)$ & $O(KL_D)$ & $O(L_D)$ \\
Final-state feedback updates over $T$ tokens & $T$ & $\lfloor T/B\rfloor$ & $0$ \\
Decoder block evaluations over $T$ tokens & $TL_D$ & $TL_D$ & $TL_D$ \\
Merge gate evaluations over $T$ tokens & $T$ & $T$ & $0$ \\
Shared feedback projections, known tokens & $T$ & $K$ & $0$ \\
Autoregressive token steps for $N$ tokens & $N$ & $N$ & $N$ \\
Blocks per consumed generation token & $L_E+L_D$ & $L_E+L_D$ & $L_E+L_D$ \\
Persistent feedback register & $d$ values & $d$ values & none \\
\bottomrule
\end{tabular}
\caption{\textbf{Execution costs at matched stack dimensions.}
$T$ counts known input tokens, $K=\lceil T/B\rceil$, and $N$ counts generated tokens; all three variants share the $L_E$-stage causal encoder prefill.
Counts omit common readout and encoder-memory projection work.
RLT-2 reuses the normalized and projected boundary state; gates depend on each token.
Backward counts cover the decoder dependency graph, with the encoder's backward pass common to all variants.
Counts describe dependencies and arithmetic; wall-clock speed depends on the implementation and hardware.}
\label{tab:rlt-efficiency}
\end{table}

\begin{figure}[t]
\centering
\begingroup
\definecolor{scheduleBlue}{HTML}{007AFF}
\definecolor{scheduleOrange}{HTML}{E87500}
\resizebox{\linewidth}{!}{%
\begin{tikzpicture}[x=1cm,y=1cm,>=stealth,font=\sffamily\small,
token/.style={draw=scheduleBlue!60,fill=white,minimum width=1.0cm,minimum height=.65cm,rounded corners=2pt}]
\path[use as bounding box] (0,0) rectangle (16.4,5.6);
\node[anchor=west,font=\sffamily\bfseries] at (0,5.25) {Known-token decoder batches};
\foreach \i in {1,...,8}{
 \pgfmathsetmacro{\x}{3+(\i-1)*1.45}
 \fill[scheduleBlue!8,rounded corners=4pt] (\x-.6,3.7) rectangle (\x+.6,4.65);
}
\fill[scheduleBlue!8,rounded corners=4pt] (2.35,2.15) rectangle (7.9,3.1);
\fill[scheduleBlue!8,rounded corners=4pt] (8.15,2.15) rectangle (13.65,3.1);
\fill[scheduleBlue!8,rounded corners=4pt] (2.35,.6) rectangle (13.65,1.55);
\foreach \name/\interval/\y/\row in {RLT-1/{B=1}/4.175/1,RLT-2/{B=4}/2.625/2,RLT-0/{B=\infty}/1.075/0}{
 \node[anchor=west,align=left] at (0,\y) {\textbf{\name}\\{\footnotesize$\interval$}};
 \foreach \i in {1,...,8}{
  \pgfmathsetmacro{\x}{3+(\i-1)*1.45}
  \node[token] (t\row\i) at (\x,\y) {$x_\i$};
 }
}
\foreach \i/\j in {1/2,2/3,3/4,4/5,5/6,6/7,7/8}{
 \draw[->,scheduleOrange,thick] (t1\i)--(t1\j);
}
\draw[scheduleOrange,thick] (t24.north)--(7.35,3.25)--(13.15,3.25);
\foreach \i in {5,...,8}{
 \pgfmathsetmacro{\x}{3+(\i-1)*1.45}
 \draw[->,scheduleOrange,thick] (\x,3.25)--(t2\i.north);
}
\node[text=scheduleOrange,font=\sffamily\scriptsize] at (10.3,3.5) {$y_4$: shared feedback for positions 5--8};
\node[anchor=west,align=left] at (14,4.175) {8 batches\\$8L_D$ stages};
\node[anchor=west,align=left] at (14,2.625) {2 batches\\$2L_D$ stages};
\node[anchor=west,align=left] at (14,1.075) {1 batch\\$L_D$ stages};
\node[anchor=west,font=\sffamily\footnotesize] at (0,.1) {Within a batch: positions in parallel at each layer. Between batches: preserve causal SWA caches.};
\end{tikzpicture}}
\endgroup
\caption{\textbf{Decoder scheduling for eight known tokens as the feedback interval grows.}
Rows show RLT-1 ($B=1$), RLT-2 ($B=4$), and RLT-0 ($B=\infty$).
Each shaded group is one layerwise decoder batch, containing $L_D$ sequential blocks.
Orange arrows show feedback dependencies; RLT-2 broadcasts $y_4$ to all four positions in its second chunk, and RLT-0 has none.
Causal SWA applies in every row, with caches retained between batches.
During autoregressive generation, all three variants consume tokens one at a time.}
\label{fig:chunk-schedule}
\end{figure}

\subsection{Training and state replay}
Training uses teacher forcing and cross-entropy on the selected next-token targets.
Full backpropagation through time (BPTT) follows recurrent outputs, decoder KV, and encoder memory; masking a target loss does not skip its state update.
For reinforcement learning (RL), evaluating a sampled response after a parameter update requires replaying its history to rebuild the state under the current parameters.
Appendix~\ref{app:training-objectives} gives the objectives and replay conditions for pretraining, supervised fine-tuning, and policy gradients.

\section{Algorithmic Experiments}
\label{sec:experiments}

We study how decoder depth and feedback frequency affect length generalization.
The experiments cover addition, parity, modular arithmetic with and without brackets, and standard and swaps-based $S_5$ state tracking.
We first compare encoder--decoder allocations, then vary the feedback interval at a fixed split.
Appendix~\ref{app:depth8-complete} gives the complete six-task study at 2,000 steps; Appendices~\ref{app:feedback-variants} and~\ref{app:mod5-variants} give all feedback variants and the 5,000-step modular-arithmetic comparisons.

\subsection{Models and evaluation}
All models have eight logical layers, width 512, FFN width 1,365, and four attention heads.
We compare untied RLT-1 splits $4+4$, $5+3$, $6+2$, $7+1$, and $8+0$ with an eight-layer decoder-only Transformer (GPT Layer 8).
RLT-1 uses an SWA window of eight, one shared encoder-memory group, and feedback scale $\alpha=0.1$.
The $8+0$ variant retains the gated recurrent merge despite having no decoder blocks.
Equal layer counts do not match parameters or compute: RLT-1 has 26.10--28.73M parameters, compared with 25.31M for the Transformer (Table~\ref{tab:experiment-parameters}).

All eight-layer accuracy comparisons use initialization seeds 42, 43, and 44, with a shared training stream and held-out examples within each task.
Global batch size is 512 and microbatch size is 32.
AdamW uses $(\beta_1,\beta_2)=(0.9,0.95)$, weight decay 0.1, and gradient clipping at norm 1.
The base learning rate warms up to $10^{-4}$ over 200 steps, then decays to $5\times10^{-6}$ at the end of each run.
Parity, addition, and $S_5$ use 2,000 steps; the mod-5 comparisons in this section use 5,000 steps with a cosine schedule spanning that budget.
Models use the same width-$\mu$P recipe and four-thread CPU FP32 execution.
For RLT-1 and RLT-2, TBPTT128 covers every training sequence, so no gradients are truncated.

Parity trains on 3--40 bits; $S_5$ trains on 32 operations, using either all 120 permutations or the identity and ten single transpositions (swaps), following \citet{grazzi2024unlocking}.
Flat mod-5 expressions use $+$, $-$, and $\times$ with multiplication precedence and odd training lengths 3--39; bracketed expressions use trees of lengths 3--40.
We score the final answer or state on 1,024 shared test examples per length.
Addition trains on randomly sampled 1--8-digit operands and uses teacher-forced answer-token accuracy on 256 shared pairs per test width.
For each run, length generalization uses the checkpoint with minimum in-distribution (ID) validation example loss; exact ties select the earliest step.
Test scores do not enter checkpoint selection.
Means and sample standard deviations (SD, $n=3$) describe initialization variability on the fixed data stream.
Appendices~\ref{app:experiment-details} and~\ref{app:length-generalization} specify generators, metrics, and checkpoint steps.

\subsection{Length extrapolation and decoder allocation}
\label{sec:length-generalization}
RLT-1 sustains parity and swaps-$S_5$ accuracy well beyond the training lengths (Figure~\ref{fig:main-depth-allocation}).
At 256 bits, parity splits $5+3$ and $7+1$ retain $100\pm0\%$ accuracy, compared with $50.07\pm1.63\%$ for the Transformer.
The advantage also appears during learning: at step 500, $6+2$ reaches $99.44\pm0.98\%$ validation accuracy versus $48.48\pm0.53\%$ for the Transformer (Figure~\ref{fig:parity-errorbars}).
In a separate seed-42 sixteen-layer series, RLT-1 Layer $8 + 8$, $9+7$, $11+5$, and $16+0$ reach 100\% at 256 bits, compared with 49.41\% for GPT Layer 16 (Appendix~\ref{app:parity-depth16}).

Swaps-$S_5$ favors a larger decoder allocation.
At 256 operations, eight times the training length, $4+4$ reaches $97.30\pm2.76\%$ final-state accuracy, compared with $0.85\pm0.30\%$ for the Transformer.
At 512 operations, $4+4$ still reaches $55.70\pm25.78\%$, while the Transformer stays at $0.85\pm0.30\%$.
Splits $7+1$ and $8+0$ are already near the uniform reference at 256 operations, despite high training-length accuracy.
Figure~\ref{fig:generalization-s5} also reports prefix-token and whole-sequence accuracy.

After 5,000 steps, RLT-1 leads on modular arithmetic at intermediate test lengths.
On flat expressions of length 63, RLT-1 Layer $6 + 2$ reaches $93.36\pm5.69\%$, compared with $33.20\pm2.33\%$ for the Transformer.
On bracketed expressions of length 64, $5+3$ reaches $67.97\pm2.91\%$, compared with $46.71\pm1.21\%$.
Accuracy falls on longer expressions, and several flat mod-5 splits vary widely across seeds.

\begin{figure}[p]
\centering
\includegraphics[width=\linewidth]{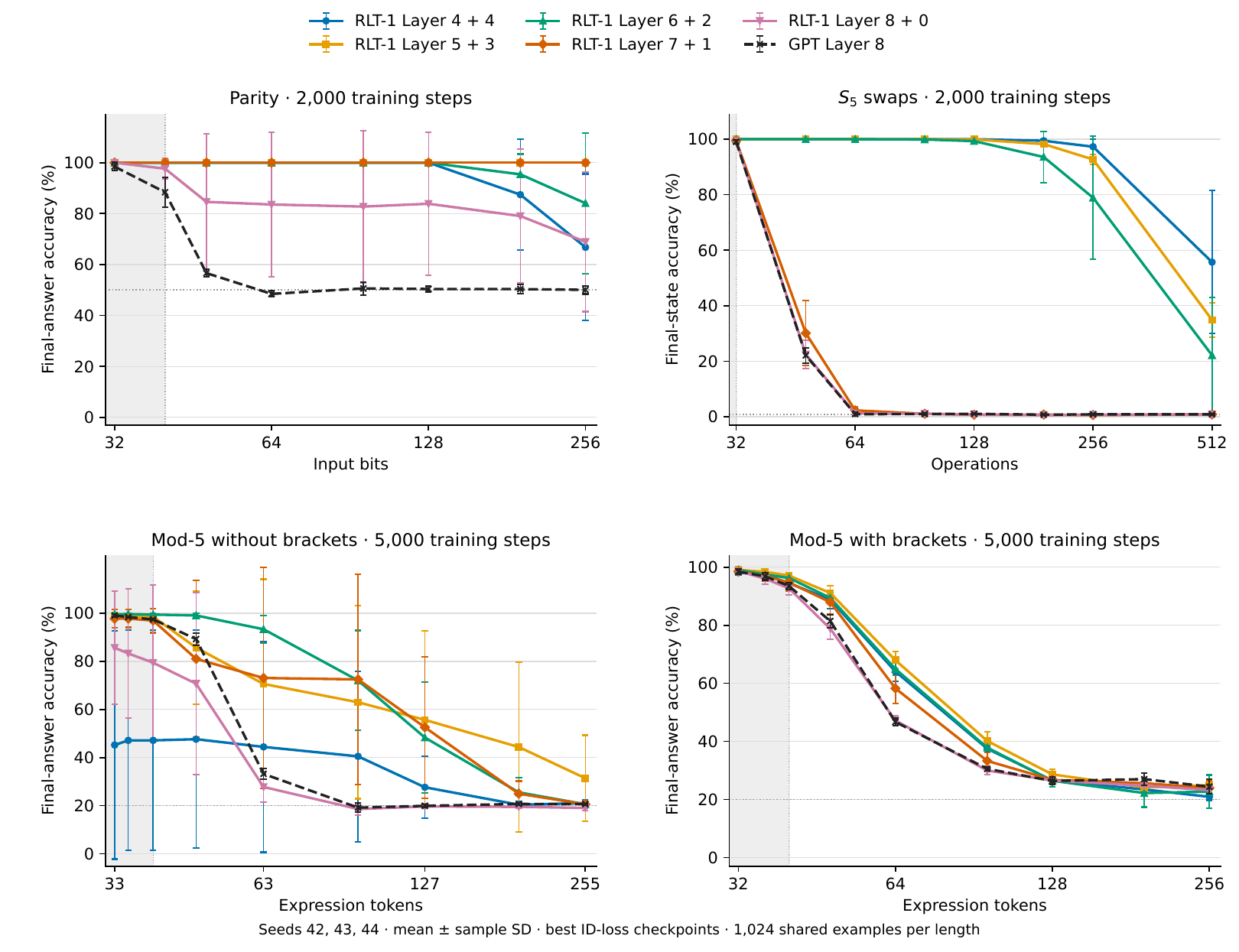}
\caption{\textbf{Length generalization across encoder--decoder allocations.}
All five RLT-1 splits and GPT Layer 8 are shown across the full evaluated length grids.
The top row uses 2,000-step parity and swaps-$S_5$ runs; the bottom row uses 5,000-step mod-5 runs for every model.
Each run selects its best ID-loss checkpoint, with earliest-step tie breaking.
Points and untrimmed error bars show mean $\pm$ sample SD over seeds 42, 43, and 44 on 1,024 shared examples per length.
Gray regions mark trained lengths; dotted horizontal lines mark uniform-prediction accuracy.
Flat mod-5 labels actual odd expression lengths.
The complete six-task 2,000-step comparison is in Figure~\ref{fig:length-generalization}.}
\label{fig:main-depth-allocation}
\end{figure}

\begin{figure}[t]
\centering
\includegraphics[width=\linewidth]{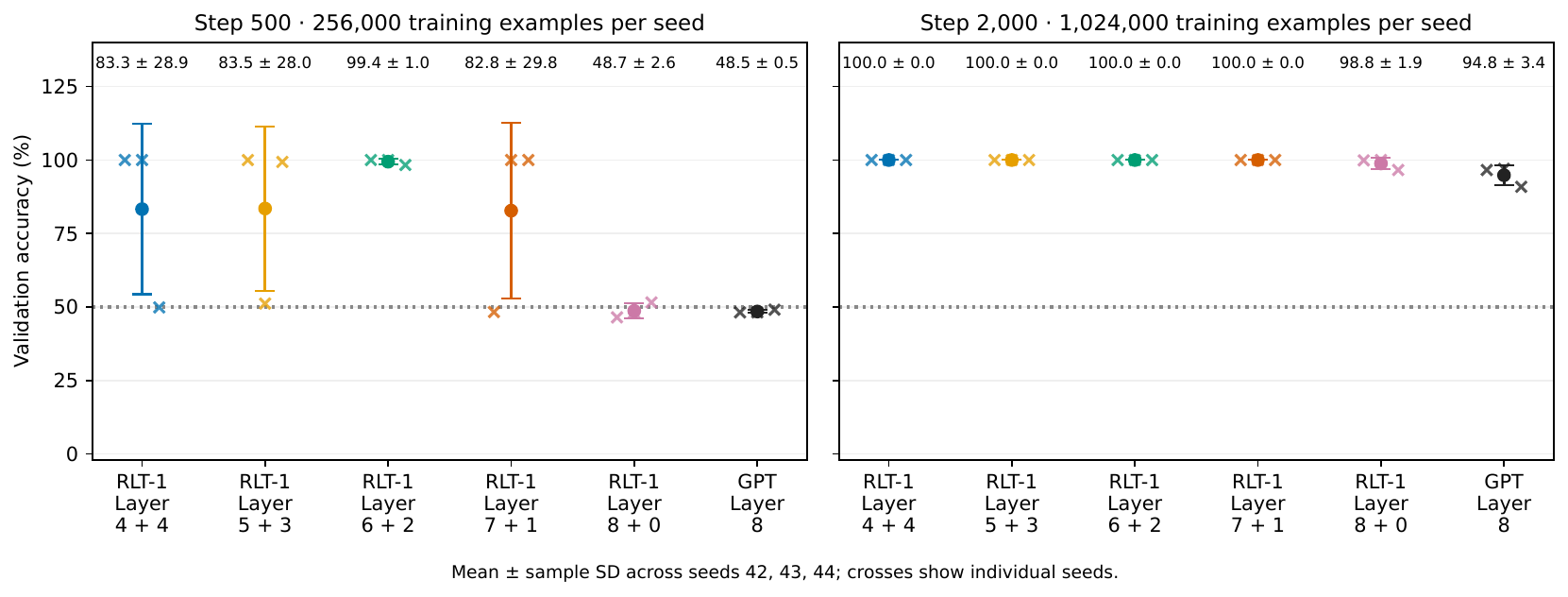}
\caption{\textbf{Parity at 500 and 2,000 training steps.}
Circles and whiskers show the mean and sample SD across seeds 42, 43, and 44; crosses show individual seeds with a small horizontal offset.
SD whiskers can extend beyond 100\%.
Every model uses the same 768 validation examples; the panels correspond to 256,000 and 1,024,000 training examples per seed.}
\label{fig:parity-errorbars}
\end{figure}

\subsection{Feedback frequency and parallelism}
We compare feedback intervals $B=1$ (RLT-1), $B=4$ (RLT-2 chunk4), and $B=\infty$ (RLT-0) at every split, using the data, optimizer, and seeds of the RLT-1 runs (Section~\ref{sec:feedback-variants}).
RLT-0 has 787,968 fewer parameters than RLT-1 at each split; RLT-2 has the same parameters.
Figure~\ref{fig:main-feedback-frequency} compares all variants at $4+4$, and Appendix~\ref{app:feedback-variants} gives every split.

At 64 bits, RLT-1 reaches $100\pm0\%$ parity accuracy and chunk4 reaches $98.99\pm1.66\%$, while RLT-0 remains near chance at $50.23\pm3.80\%$.
Swaps-$S_5$ is more sensitive to the feedback interval: at 64 operations, RLT-1 reaches $100\pm0\%$ and chunk4 $19.60\pm6.54\%$.
RLT-0 and the Transformer are near the $1/120$ uniform reference at this length.
Relative to RLT-1, four-token chunks lose about one percentage point on 64-bit parity and about 80 points on 64-operation swaps-$S_5$.

The effect of chunking also depends on the split (Figures~\ref{fig:variant-parity} and~\ref{fig:variant-s5-swaps}).
Chunk4 $4+4$ retains $69.34\pm24.17\%$ parity accuracy at 256 bits, whereas chunk4 $7+1$ and $8+0$ are already near chance at 64 bits.
On swaps-$S_5$, chunk4 $5+3$ reaches $79.92\pm10.89\%$ at 48 operations and $50.16\pm18.77\%$ at 64, compared with $19.60\pm6.54\%$ at 64 for chunk4 $4+4$.

\begin{figure}[t]
\centering
\includegraphics[width=\linewidth]{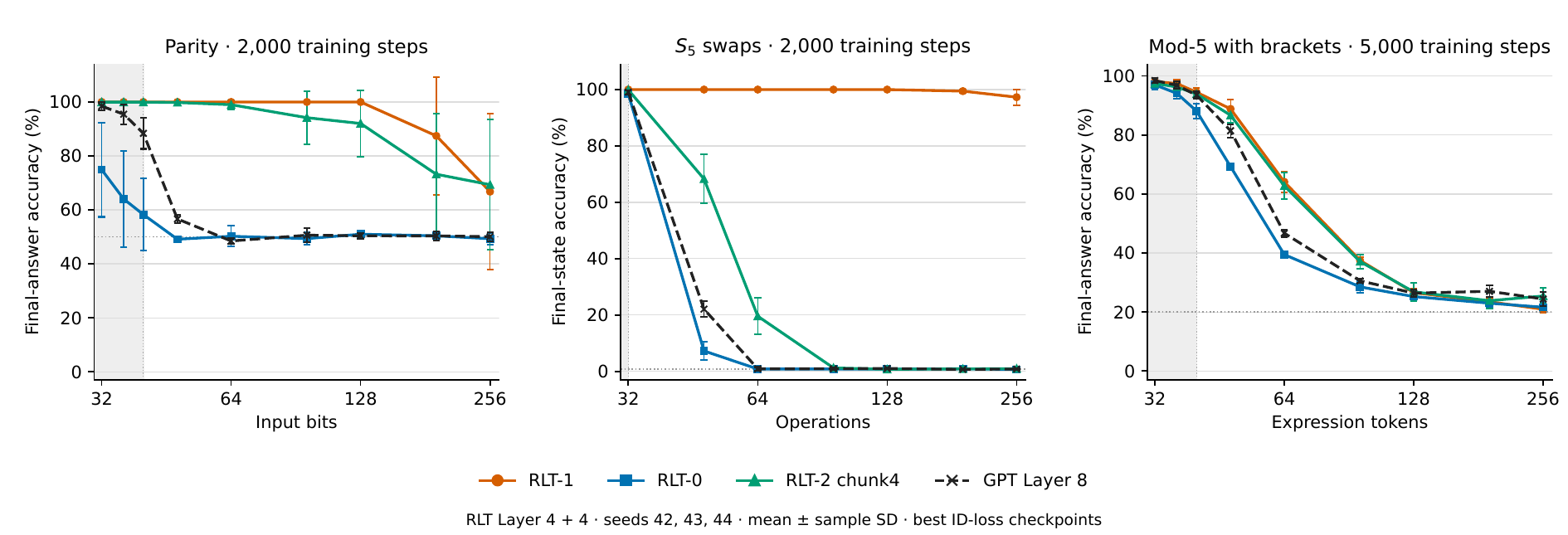}
\caption{\textbf{Feedback frequency at a fixed $4+4$ split.}
Final-answer/state accuracy for RLT-1 ($B=1$), RLT-2 chunk4 ($B=4$), RLT-0 ($B=\infty$), and GPT Layer 8, across all native test lengths.
Parity and swaps-$S_5$ use 2,000 training steps; bracketed mod-5 uses 5,000.
All models use seeds 42--44 and their best ID-loss checkpoints.
Error bars show sample SD without clipping; gray regions mark trained lengths and dotted lines mark uniform prediction.}
\label{fig:main-feedback-frequency}
\end{figure}

\subsection{Addition, standard \texorpdfstring{$S_5$}{S5}, and scope}
All models reach 100\% teacher-forced accuracy on the training-width addition validation set, but accuracy drops beyond eight digits; the eight-layer model means at 32 digits range from 14.89\% to 16.84\%.
Standard $S_5$ remains difficult, with final-state accuracy near the $1/120$ reference across feedback variants (Figure~\ref{fig:variant-s5-standard}).
Appendix~\ref{app:experiment-details} gives the complete six-task curves, endpoint tables, and individual parity trajectories; Appendix~\ref{app:addition-feedback} gives the addition feedback-scale ablation.
These experiments measure supervised algorithmic performance; RL performance is not evaluated.

\section{Related Work}
\label{sec:related-work}
\paragraph{Hybrid Transformer--RNN models.}
\citet{chen2018best} combine a Transformer encoder with an LSTM-based RNMT+ decoder for machine translation.
RLT-1 also combines parallel encoding with recurrent decoding.
Its decoder consists of Transformer blocks with encoder-memory cross-attention and SWA, and it processes both prompt and response tokens in a causal language model.

\paragraph{Encoder-derived memory.}
YOCO builds reusable KV memory for an upper cross-decoder and allows early exit during prefill \citep{sun2024yoco}.
DeepSeek-V4.1-Flash projects global decoder KV from final encoder states and maintains separate decoder SWA \citep{deepseek2026v41}.
RLT-1 also builds global memory from encoder states, but its recurrent decoder must process every prompt token.

\paragraph{Temporal feedback.}
Feedback Transformer forms a learned weighted sum of each token's representations across layers and lets subsequent tokens attend to this shared memory at every layer \citep{fan2020feedback}.
RLT-1 feeds the preceding final decoder output into a gated merge at the decoder input, with separate encoder-derived global memory and layerwise decoder SWA caches.
Recurrent Transformer constructs each layer's persistent KV from that layer's output and provides an exact tiling schedule to improve memory movement \citep{oncescu2026recurrent}.
RLT-1's recurrent dependency spans the full decoder, whereas Recurrent Transformer's feedback is layerwise.

\paragraph{Cross-token latent feedback.}
Full-bandwidth Transformer combines the previous top-layer hidden state with the next token embedding through a gated linear unit, retaining the Transformer stack and KV cache \citep{wang2026fullbandwidth}.
Its multi-pass training shifts hidden states between passes to allow token-parallel teacher forcing, with a prefix mixin to address differences between prompts and generation.
RLT-1 places output-to-input feedback at the encoder--decoder interface and trains by replaying the decoder sequentially over the full history.

T$^2$MLR feeds a cached middle-layer representation from the previous token into an earlier layer at the current position \citep{cai2026t2mlr}.
Its experiments find that recurrence between middle layers can outperform recurrence through the full network.
Training approximates temporal states with a fixed number of Jacobi iterations and controls backward depth separately.
RLT-1 feeds back through the entire decoder and trains by replaying the full history with BPTT, optionally truncated.

\paragraph{Latent recurrent language models.}
Latent Recurrent Transformer (LRT) reuses a high-level state from the previous token through KV projection and residual injection, keeping a decoder-only backbone and one forward pass per generated token \citep{huang2026latentrecurrent}.
Its interleaved parallel training refines subsets of positions from a shared state buffer.
For RL, LRT initializes this buffer with detached rollout states to reduce differences between rollout and recomputation.
RLT-1 separates encoder memory from the recurrent decoder and rebuilds the full history under current parameters for exact policy replay.

\paragraph{Continuous latent computation.}
Coconut feeds the last hidden state back as the next input embedding during a latent reasoning phase, using a curriculum that replaces textual reasoning steps with continuous states \citep{hao2024coconut}.
PonderLM-2 inserts latent steps between ordinary tokens during pretraining and uses Jacobi iterations to approximate their recurrent dependencies in parallel \citep{zeng2025ponderlm2}.
RLT-1 advances its state once per ordinary prompt or response token, without adding latent positions.

\paragraph{Block- and segment-level recurrence.}
Block-Recurrent Transformers update persistent state vectors with attention and gates, processing a block of tokens in parallel at each recurrent step \citep{hutchins2022blockrecurrent}.
Their block-feedback variant lets all layers cross-attend to the recurrent state from the preceding block.
Recurrent Memory Transformer appends write-memory tokens to each segment and passes their final representations to the next segment as memory inputs, with BPTT through this connection \citep{bulatov2022memory}.
RLT-2 is the closest variant: it updates feedback at chunk boundaries from the last token's final decoder output and gates this state into every decoder input of the next chunk, without dedicated memory tokens or a separate bank of recurrent state vectors.
RLT-1, the $B=1$ case, feeds back at every token and therefore requires sequential decoder work within each block.

\paragraph{Weight sharing across depth.}
Universal Transformers share weights across depth and optionally adapt the number of refinement steps by position \citep{dehghani2018universal}.
\citet{saunshi2025looped} study how repeated applications of a shared Transformer stack increase effective depth for reasoning, while recurrent-depth language models vary latent computation at inference time \citep{geiping2025latent}.
DeepLoop analyzes the effect of repeated parameter visits on residual scaling in weight-tied looped Transformers and derives a loop-aware scaling rule for the Post-LN architecture \citep{li2026deeploop}.
These models recur over depth at each position; RLT-1 recurs across tokens, and sharing weights between its encoder and decoder is optional.

\section{Conclusion}
We introduced the Recurrent Looped Transformer (RLT), which pairs a parallel causal encoder with a decoder that feeds its final state back at every token, so the computation path grows with sequence length at a fixed per-token cost.
With eight layers, RLT generalizes parity and swap-based $S_5$ tracking far beyond the training lengths, where an eight-layer Transformer is at chance, and its best splits outperform the Transformer on modular arithmetic beyond the training lengths.
Removing the feedback eliminates the gains on parity and swap-based $S_5$.
The encoder--decoder split is one design axis: parity generalizes best with one or three decoder layers, whereas swap-based $S_5$ needs at least two decoder layers and improves with each additional one up to the four tested.
The feedback interval $B$ is a second axis that trades parallelism against accuracy: $B=4$ processes known tokens in parallel within each chunk and keeps 64-bit parity at 99\%, while swap-based $S_5$ benefits most from per-token feedback.
Because the chunk size does not change the parameters, a model can be pretrained with large chunks for parallelism and continue with smaller chunks, down to $B=1$, in mid- and post-training.

\section*{Acknowledgement}

We used large language models to improve the wording of this work.

\vspace{2ex}
\setlength{\bibsep}{2pt}
\bibliographystyle{plainnat}
\bibliography{reference}
\clearpage
\appendix
\renewcommand{\appendixpagename}{\centering\huge Appendix}
\appendixpage
\counterwithin{theorem}{section}
\startcontents[section]
\begingroup
\small
\printcontents[section]{l}{1}{\setcounter{tocdepth}{2}}
\endgroup
\clearpage
\section{Architecture and Execution Details}
\subsection{Sharing encoder and decoder weights}
The reference tied RLT-1 sets $L_E=L_D=L$.
Encoder self-attention at layer $\ell$ and decoder SWA at layer $\ell$ share compatible query, key, value, and output projections; the corresponding FFNs also share weights.
The encoder attends to its causal context; decoder SWA attends to decoder activations within its window.
Decoder cross-attention uses separate query/output projections and the memory projections of Equation~\eqref{eq:memory}, adding computation beyond the shared attention and FFN.
The encoder and decoder retain separate normalizations, and the merge and readout are separate modules.

The two passes share weights but compute separate activations.
An untied $E_\theta,D_\phi$ uses separate weights and keeps the same recurrence.
Memory groups can be shared in either version; each decoder layer still maintains its own SWA cache.

\section{Computational Properties}
\label{app:execution-details}
\subsection{Prompt--response consistency}
\begin{proposition}[Invariance to the serving split]
Fix the parameters, token sequence, position convention, and initial state for an independent sequence.
Assume exact arithmetic, mathematically equivalent causal encoder execution, identical SWA windows and cache updates, and deterministic decoder operations.
Processing a prefix with batched encoder prefill followed by recurrent decoder updates gives the same states and next-token distributions as processing it incrementally.
The conditional distribution for a fixed token history is therefore unchanged when the prompt--response split moves.
\end{proposition}
\begin{proof}
Causal encoder equivalence gives the same $e_t$ and $M_{\le t}$ in both schedules.
Both start from $H_0=(s_\star,\varnothing)$.
If their states agree at $t-1$, Equation~\eqref{eq:recurrence} applies the same operations to the same inputs at $t$, so their states agree at $t$.
The result follows by induction, since the transition does not depend on the serving split.
\end{proof}

\subsection{Prefill work and sequential depth}
Let $C_E^{\mathrm{pf}}(T)$ be encoder prefill work, $C_M(T)$ the memory projection work, and $C_D^{\mathrm{step}}(t)$ a decoder evaluation over $t$ encoder-memory entries and at most $W$ decoder positions per layer, including merge overhead.
Then
\begin{equation}
 C_{\mathrm{prefill}}(T)=C_E^{\mathrm{pf}}(T)+C_M(T)
 +\sum_{t=1}^T C_D^{\mathrm{step}}(t). \label{eq:cost}
\end{equation}
For dense attention and width-proportional KV, a coarse arithmetic estimate is
\begin{equation}
 O\bigl((L_E+L_D)(Td^2+T^2d)+GTd^2
 +L_D T\min(W,T)d\bigr).
\end{equation}
Prefill requires $T$ sequential decoder transitions, each containing $L_D$ blocks.
These sequential dependencies can reduce hardware utilization even when arithmetic complexity is of the same order as a dense Transformer's.

Each new token evaluates $L_E+L_D$ blocks, plus merge and memory projection.
Attention work still grows with context length.
With effective KV width $d_{\mathrm{KV}}$, inference cache storage is approximately
\begin{equation}
 O\bigl((L_E+G)t\,d_{\mathrm{KV}}
 +L_D\min(t,W-1)d_{\mathrm{KV}}^D+d\bigr).
\end{equation}
Here $d_{\mathrm{KV}}^D$ is the effective decoder SWA KV width, and the $O(d)$ term stores the current recurrent output.
The SWA term counts retained history; current-position KV and training activations need additional storage.
Weight sharing reduces parameter storage but retains both passes and their caches.

\section{Hardware Execution}
\subsection{Batching independent sequences}
For a known training sequence or prompt, the encoder computes features and memory projections in parallel across positions, then the decoder processes positions in order.
Independent sequences can batch their next decoder updates together (Figure~\ref{fig:hardware-schedule}).
Each sequence keeps its own encoder prefix, SWA cache, recurrent output, and position.

At inference, each step consists of an encoder update, memory append, merge, decoder pass, and readout.
Batching requests increases matrix-operation sizes and allows weight reuse.
Small batches and uneven sequence lengths may limit utilization.

\begin{figure}[!htbp]
\centering
\begingroup
\definecolor{rblue}{HTML}{007AFF}
\definecolor{rorange}{HTML}{E87500}
\definecolor{rgray}{HTML}{566172}
\resizebox{\linewidth}{!}{%
\begin{tikzpicture}[font=\sffamily\small,>=stealth,
box/.style={draw=rblue!50,fill=rblue!5,rounded corners=3pt,align=center,minimum height=.7cm},
state/.style={box,draw=rorange!70,fill=rorange!8}]

\path[use as bounding box] (0,0) rectangle (16,5.65);
\node[anchor=west,font=\sffamily\bfseries] at (0,5.35) {Batching decoder updates from independent sequences};
\node[box,text width=3.1cm,minimum height=2.9cm] (enc) at (1.8,2.9) {\textbf{Causal encoder}\\[6pt]Encode known tokens\\in parallel\\[6pt]Build each sequence's\\prefix KV memory};
\foreach \col/\x in {1/6,2/8.4,3/10.8,4/13.2}{
\fill[rblue!5,rounded corners=4pt] (\x-.7,1.35) rectangle (\x+.7,4.4);
\node[text=rgray,font=\sffamily\footnotesize] at (\x,4.65) {Batch \col};}
\foreach \row/\y in {A/3.85,B/2.9,C/1.95}{
\node[anchor=east] at (5,\y) {Seq. \row};
\foreach \col/\x in {1/6,2/8.4,3/10.8,4/13.2}{
\node[state,minimum width=1cm] (n\row\col) at (\x,\y) {$D_{\phi}$};}
\draw[->,rorange,thick] (n\row1)--(n\row2);
\draw[->,rorange,thick] (n\row2)--(n\row3);
\draw[->,rorange,thick] (n\row3)--(n\row4);
\draw[->,rorange,thick] (n\row4)--(14.7,\y);}
\draw[->,rblue,thick] (enc.north) -- (1.8,4.65) -- (3.8,4.65);
\node[anchor=west,text=rblue,font=\sffamily\footnotesize] at (3.9,4.98) {Each update reads its encoder prefix and decoder SWA window};
\node[anchor=west,text=rgray] at (4.7,.85) {Columns: independent updates can share one batched kernel.};
\node[anchor=west,text=rgray] at (0,.15) {Rows: updates follow token order within each sequence.};
\end{tikzpicture}}
\endgroup
\caption{\textbf{Batching decoder updates across sequences.}
Columns group independent updates into a batch; rows follow each sequence in token order.
Each update reads its own encoder prefix and decoder SWA window.
Known tokens can be encoded in parallel before replay; generated tokens are encoded as they arrive.}
\label{fig:hardware-schedule}
\end{figure}

\subsection{Memory traffic and parameter reuse}
Encoder KV stays fixed while the prefix and parameters are unchanged.
Decoder layers reuse that memory and append their own KV to separate SWA caches.
Fewer memory groups reduce KV storage and projection work; more groups allow different transformations for each layer.
Sharing encoder and decoder weights may also help keep weights resident on the accelerator.

Normalization, gating, state projection, and residual addition can be fused into one kernel.

\subsection{Training memory and numerical agreement}
Activation checkpointing reduces stored activations by recomputing them during backpropagation through time (BPTT).

\citet{zhang2026kernelmismatch} describe how the executed policy depends on precision, cache construction, reductions, and sampling transforms as well as weights.
Sampler and trainer implementations must align positional conventions, stochastic behavior, and cache contents, and check numerical agreement across execution modes.

\section{Training Objectives}
\label{app:training-objectives}
\subsection{Autoregressive pretraining}
Full-sequence next-token prediction is the base objective:
\begin{equation}
 \mathcal L_{\mathrm{PT}}(\Theta)=
 -\mathbb E_{x_{1:S}}\left[\frac{1}{S-1}\sum_{t=1}^{S-1}
 \log p_\Theta(x_{t+1}\mid x_{1:t})\right]. \label{eq:pretrain-loss}
\end{equation}
Initialize $H_0=(s_\star,\varnothing)$, compute causal encoder features, and update the decoder at positions $1,\ldots,S-1$.
Every non-BOS target contributes to the loss; the encoder, memory projections, merge, and decoder train jointly.
At each independent document, reset both decoder and encoder state, reset positions, and prevent attention across document boundaries.
When training on segments, specify the initial context and state: dropping an earlier state changes the history on which the likelihood is conditioned.

\subsection{Supervised fine-tuning}
Let $m_{t+1}=1$ for assistant targets and $0$ for user, system, tool, or padding targets.
For examples with at least one selected target, optimize
\begin{equation}
 \mathcal L_{\mathrm{SFT}}(\Theta)=
 -\mathbb E_x\left[\frac{1}{\sum_{t=1}^{S-1}m_{t+1}}
 \sum_{\substack{1\le t<S\\m_{t+1}=1}}
 \log p_\Theta(x_{t+1}\mid x_{1:t})\right]. \label{eq:sft-loss}
\end{equation}
The mask selects loss terms; user, system, and tool tokens still update the decoder state.
Gradients from assistant losses flow through these updates, including encoder memory and decoder KV.
The model continues from the existing state when an assistant turn begins.

\subsection{Policy gradients and replay}
Let $c=x_{1:T}$ be a prompt and $y=(y_1,\ldots,y_N)$ a sampled response, with $y_i=x_{T+i}$.
The policy is
\begin{equation}
 \pi_\Theta(y\mid c)=\prod_{i=1}^{N}
 p_\Theta(y_i\mid c,y_{<i}).
\end{equation}
For a sequence reward $R(c,y)$ independent of $\Theta$, define $J(\Theta)=\mathbb E_{c,\,y\sim\pi_\Theta}[R(c,y)]$.
Its on-policy score-function gradient is
\begin{equation}
 \nabla_\Theta J=
 \mathbb E_{c,\,y\sim\pi_\Theta}\left[
 (R(c,y)-b(c))\sum_{i=1}^{N}
 \nabla_\Theta\log p_\Theta(y_i\mid c,y_{<i})\right],
 \label{eq:policy-gradient}
\end{equation}
where $b(c)$ is a response-independent baseline, held constant when taking the policy gradient.
During replay, the sampled tokens are fixed and gradients pass through the recurrent computation used to evaluate their log-probabilities.

Figure~\ref{fig:rl-replay} shows the replay procedure.
The sampler records each action's log-probability under the distribution $\mu$ that sampled it.
The trainer rebuilds encoder features, recurrent outputs, and decoder SWA caches under the current parameters $\Theta$, starting from the initial state and processing the full prompt and sampled response prefix.
The resulting action probabilities give the ratios
\begin{equation}
 r_i(\Theta)=\exp\left(
 \log p_\Theta(y_i\mid c,y_{<i})
 -\log\mu(y_i\mid c,y_{<i})\right). \label{eq:ratio}
\end{equation}
Here the target is the raw model policy $p_\Theta$.
The behavior probabilities must include any temperature scaling, truncation, and renormalization used by the sampler.
Exact importance sampling requires $p_\Theta(\cdot\mid h)\ll\mu(\cdot\mid h)$: at each relevant history, every action with positive target probability must also have positive behavior probability.
Top-$k$ or top-$p$ sampling generally excludes actions with positive probability under an untruncated softmax target.
If the target policy itself uses a sampling transform, apply it consistently in the objective and ratios.

To preserve both forward probabilities and their parameter derivatives, the trainer differentiates the recurrent computation or an equivalent implementation \citep{zhang2026kernelmismatch}.
The recorded behavior log-probabilities stay fixed across parameter updates, since they describe the policy that sampled the actions.

For fixed prompts and sequence rewards, exact trajectory importance sampling uses $\prod_i r_i$, subject to support and integrability conditions.
Tokenwise clipping and other surrogate losses may introduce bias even with correct replay probabilities.

\begin{figure}[!htbp]
\centering
\begingroup
\definecolor{rblue}{HTML}{007AFF}
\definecolor{rorange}{HTML}{E87500}
\definecolor{rgray}{HTML}{566172}
\resizebox{\linewidth}{!}{%
\begin{tikzpicture}[font=\sffamily\small,>=stealth,
box/.style={draw=rblue!50,fill=rblue!5,rounded corners=3pt,align=center,minimum height=.7cm},
state/.style={box,draw=rorange!70,fill=rorange!8}]

\path[use as bounding box] (0,0) rectangle (16,6.6);
\node[anchor=west,font=\sffamily\bfseries] at (0,6.3) {Replaying a sampled history under current parameters};
\node[box,text width=6.7cm,minimum height=.8cm] at (3.7,5.4) {\textbf{SAMPLER}\quad behavior policy $\mu$};
\node[state,text width=6.7cm,minimum height=.8cm] at (12,5.4) {\textbf{TRAINER}\quad current parameters $\Theta$};
\node[box,text width=6.7cm,minimum height=1.05cm] (sample) at (3.7,4.05) {Process the full prompt; sample a response\\Record tokens $y$ and $\log\mu(y_i\mid c,y_{<i})$};
\node[state,text width=6.7cm,minimum height=1.05cm] (replay) at (12,4.05) {Rebuild encoder KV, $s_t$, and decoder SWA KV\\Replay full prompt + fixed response tokens};
\draw[->,thick] (sample)-- node[above,font=\sffamily\footnotesize] {$c,y$} (replay);
\node[box,text width=6.7cm,minimum height=.85cm] (old) at (3.7,2.6) {Keep recorded behavior log-probabilities\\as the ratio denominator};
\node[state,text width=6.7cm,minimum height=.85cm] (cur) at (12,2.6) {Compute current $\log p_\Theta(y_i\mid c,y_{<i})$\\Full BPTT; checkpoint as needed};
\draw[->,rblue,thick] (sample)--(old);
\draw[->,rorange,thick] (replay)--(cur);
\node[draw=rgray!50,rounded corners=3pt,fill=gray!3,minimum width=11cm,minimum height=.85cm] (ratio) at (7.85,1.1) {$r_i(\Theta)=\exp\!\left[\log p_\Theta(y_i\mid c,y_{<i})-\log\mu(y_i\mid c,y_{<i})\right]$};
\draw[->,rblue,thick] (old.south)--(3.7,1.75)--(4.5,1.75)--(4.5,1.55);
\draw[->,rorange,thick] (cur.south)--(12,1.75)--(11.2,1.75)--(11.2,1.55);
\node[anchor=west,text=rgray,font=\sffamily\footnotesize] at (0,.15) {After every parameter update: rebuild current-policy KV and states; retain original behavior log-probabilities.};
\end{tikzpicture}}
\endgroup
\caption{\textbf{Replaying a sampled history under current parameters.}
The sampler records behavior probabilities.
The trainer rebuilds the full prefix, including recurrent outputs and decoder SWA caches, to evaluate the current policy.
The ratio compares the current action probability with the recorded behavior probability.}
\label{fig:rl-replay}
\end{figure}

\section{Multi-turn Inference and Cache Reuse}
A conversation forms one token history beginning at BOS.
User messages, tool results, role delimiters, and assistant tokens all receive encoder and decoder updates.
For example, when a tool returns, encode its output using the cached prefix, then process those tokens through the decoder in order before resuming generation.
Reset state only for a new independent sequence or an explicitly defined context reset.

To resume from a saved prefix, retain the encoder cache $C_t^E$, cross-attention memory $M_{\le t}$, decoder state $H_t=(s_t,C_t^D)$, tokens and positions, SWA window convention, and model version.
These values can be reused at fixed weights regardless of where the prompt ends.
Reproducing the same sampled outputs also requires the sampler's random state.
If only text is saved, replay it to rebuild the state.

Editing or truncating an earlier token invalidates the later cached state.
Recompute from a valid checkpoint before the edit; this also applies to chat templates that rewrite earlier tokens.
In multi-turn RL, user and tool tokens update the state and carry gradients, but receive no action importance-ratio factors because the policy did not sample them.

\clearpage
\section{Execution Procedures}
The complete decoder state is $H_t=(s_t,C_t^D)$, initialized as $H_0=(s_\star,\varnothing)$.
Encoder continuation state and encoder-derived cross-attention memory are maintained separately.
All schedules use the block order and window convention in Equation~\eqref{eq:block}.

\subsection{Recurrent prompt prefill}
\begin{enumerate}
\item Start an independent sequence with $x_1=\mathrm{BOS}$ and
      $H_0=(s_\star,\varnothing)$; initialize encoder caches and positions.
\item Encode $x_{1:T}$ causally in parallel and construct its encoder KV
      memory.
\item For $t=1,\ldots,T$, compute
      $H_t=D_\phi(\Merge(e_t,s_{t-1});M_{\le t},C_{t-1}^D,t)$.
      Each decoder layer reads current KV and the retained SWA history.
      Both attention masks exclude positions after $t$.
\item Predict the first response token from $s_T$. Preserve $H_T$, encoder
      caches, memory, and positional metadata for the next update.
\end{enumerate}

\subsection{Generation and new external inputs}
Sample $x_{t+1}$ from the distribution predicted by $s_t$, then update the encoder cache, memory, and complete decoder state to $H_{t+1}$.
This state predicts $x_{t+2}$.
Each consumed token receives exactly one recurrent update and one KV insertion at every decoder SWA layer.

Multiple tokens from a user or tool can be encoded in a causal batch conditioned on the existing prefix.
The decoder processes them in token order from the saved state, updating both $s_t$ and $C_t^D$.

If generation stops at a length limit before consuming its last emitted token, record that token separately.
The cache represents the consumed prefix; process the pending token once before continuing with later tokens.

\subsection{Pretraining and SFT}
\begin{enumerate}
\item Compute $e_{1:S-1}$ using a causal encoder and construct memory.
\item Initialize $H_0=(s_\star,\varnothing)$; unroll all positions
      $1,\ldots,S-1$, including every layerwise SWA cache update.
\item Accumulate all valid next-token losses for pretraining, or only
      assistant-target losses for SFT. All context tokens receive state
      updates.
\item Normalize each example by its number of selected targets, then
      average examples as in Equations~\eqref{eq:pretrain-loss}
      and~\eqref{eq:sft-loss}. Backpropagate through the complete
      computation, using checkpointing if required.
\end{enumerate}
Examples without selected targets are excluded from the loss average.
Token averaging would give more weight to examples with more selected targets.

Packed independent sequences need separate recurrent states, SWA caches, encoder caches, positions, and attention masks.
Both attention mechanisms must stay within document boundaries, even where losses are masked.

\subsection{Current-policy RL replay}
\begin{enumerate}
\item Read the rollout token history, action mask, actual behavior
      log-probabilities, and sampling metadata. Behavior probabilities
      include all sampling transforms and renormalization.
\item Hold current parameter values fixed throughout forward replay and
      backward. Use the target policy's positional, SWA, and execution
      conventions, with parameter gradients enabled.
\item Recompute encoder representations of the known history. From
      $H_0=(s_\star,\varnothing)$, rebuild the recurrent output and every
      decoder SWA cache through all prompt tokens.
\item Replay subsequent tokens in order. Before consuming each sampled
      action, read its current-policy log-probability from the preceding
      state. Include EOS if sampled. Consume every intervening external
      token needed for later predictions.
\item Form the chosen RL loss from action log-probabilities, rewards or
      advantages, and any required behavior ratios. Apply action factors
      only to tokens sampled by the policy.
\item Backpropagate, update parameters, and invalidate parameter-dependent
      encoder and decoder caches before the next exact current-policy
      replay.
\end{enumerate}
The action mask selects loss terms while preserving gradients through user and tool tokens.
Prompt replay under \texttt{no\_grad}, or detaching prompt KV, preserves forward values but drops the prompt-state gradients required by full BPTT.

A deterministic forward pass, for example with dropout disabled, gives a reproducible probability for each action.
For a stochastic policy, specify whether action probabilities condition on internal randomness or average over it.
Each replay evaluates one realization; marginal probabilities require averaging over that randomness.

\section{Causality and Gradient Paths}
\begin{proposition}[Causality]
With a causal encoder, prefix-restricted encoder memory, and causal
decoder SWA, $H_t=(s_t,C_t^D)$ depends only on $x_{1:t}$ and the parameters,
including $s_\star$, under deterministic execution.
\end{proposition}
\begin{proof}
Encoder causality implies that each $e_j$ depends only on $x_{1:j}$, hence
$M_{\le t}$ depends only on $x_{1:t}$. The initial state contains no
future-token information. Suppose $H_{t-1}$ depends only on $x_{1:t-1}$.
The next update uses this state, $e_t$, and $M_{\le t}$. At each decoder
layer, current KV is formed from the causally available layer input,
and SWA reads no position greater than $t$. Appending current KV and
evicting old entries introduce no future information. Thus both $s_t$
and $C_t^D$ depend only on $x_{1:t}$, completing the induction.
\end{proof}

For teacher-forced encoder features held fixed, use a fixed-slot
representation of the decoder cache, with validity masks during warm-up,
and define
\begin{equation}
 \mathcal J_t=\frac{\partial H_t}{\partial H_{t-1}}.
\end{equation}
Then
\begin{equation}
 \frac{\partial H_t}{\partial H_j}
 =\mathcal J_t\mathcal J_{t-1}\cdots\mathcal J_{j+1},\qquad j<t,
\end{equation}
where
\begin{equation}
 \mathcal J_t=
 \begin{pmatrix}
 \dfrac{\partial s_t}{\partial s_{t-1}} &
 \dfrac{\partial s_t}{\partial C_{t-1}^D}\\[6pt]
 \dfrac{\partial C_t^D}{\partial s_{t-1}} &
 \dfrac{\partial C_t^D}{\partial C_{t-1}^D}
 \end{pmatrix}.
\end{equation}
The block Jacobian includes paths through both the recurrent output and decoder KV.

Encoder memory retains information from past encoder outputs; decoder SWA provides direct access to recent decoder projections.
Evicted decoder entries may still affect later computation through states or activations that read them before eviction.

Full parameter derivatives also include the parameter dependence of
encoder features, memory, every transition, and decoder KV projections.
Let $B_T^E$ denote all encoder-side boundary tensors needed for response
continuation, including encoder continuation KV and encoder-derived
cross-attention memory. For a response loss
$\ell(\Theta,H_T(\Theta),B_T^E(\Theta))$, the chain rule gives
\begin{equation}
 \frac{d\ell}{d\Theta}
 =\frac{\partial\ell}{\partial\Theta}
 +\frac{\partial\ell}{\partial H_T}\frac{\partial H_T}{\partial\Theta}
 +\frac{\partial\ell}{\partial B_T^E}\frac{\partial B_T^E}{\partial\Theta}.
\end{equation}
The first term holds the boundary arguments fixed; the others account for their prefix computation.
In particular,
\begin{equation}
 \frac{\partial\ell}{\partial H_T}\frac{\partial H_T}{\partial\Theta}
 =\frac{\partial\ell}{\partial s_T}\frac{\partial s_T}{\partial\Theta}
 +\frac{\partial\ell}{\partial C_T^D}\frac{\partial C_T^D}{\partial\Theta}.
\end{equation}
Detaching $s_T$, decoder KV, or encoder boundary tensors removes the corresponding gradient terms, even when forward probabilities stay unchanged.

\section{When Cached States Can Be Reused}
A saved prefix can replace forward replay when it matches the current parameters, consumed tokens, positions, initialization, SWA window, and policy execution settings.
Save the recurrent output, each decoder SWA cache, encoder continuation state, encoder memory, and associated metadata.
The saved prefix may come from the sequence beginning or from an exact checkpoint made with the current parameters and execution settings.

A detached cache can hold correct values while lacking the computation graph needed for training.
Full BPTT requires retaining or recomputing the prefix graph so gradients reach every parameter-dependent boundary tensor.

Activation checkpointing within one update recomputes activations under the same parameters and stochastic state.
Mutable cache implementations must restore the saved contents during recomputation and preserve activations needed for backward.

Truncated BPTT preserves numeric values while cutting selected gradient dependencies.
A complete decoder-state detach is
\begin{equation}
 \widetilde H_t=(\sg(s_t),\sg(C_t^D)).
\end{equation}
Detaching only $s_t$ leaves possible paths through decoder KV; detaching only decoder KV leaves paths through the recurrent output.
Encoder-side boundary tensors can provide additional paths across the same boundary.

SWA eviction removes old KV from future attention windows.
Full BPTT still differentiates earlier computations that used those entries, so training may need more activation storage than the inference cache alone.

After a weight update, recompute the prefix values and gradient graph before the next exact training pass.

\section{RLT-2: Chunk-Parallel Hidden-State Feedback}
\label{sec:rlt2}
Section~\ref{sec:feedback-variants} defines the RLT-2 merge and boundary update, and Figure~\ref{fig:architecture-chunk} shows one chunk; this appendix gives its masks, training and prefill procedure, incremental generation, and costs.
Section~\ref{sec:experiments} and Appendix~\ref{app:feedback-variants} report the experiments.

\subsection{Chunk operator and causal masks}
Choose a fixed chunk size $B\ge1$, count BOS as position one, and anchor boundaries to the start of each independent sequence.
For a known prefix of length $T$, define
\begin{equation}
 K=\lceil T/B\rceil,\qquad a_k=(k-1)B+1,\qquad
 b_k=\min(kB,T),\qquad I_k=\{a_k,\ldots,b_k\}.
 \label{eq:chunk-indices}
\end{equation}
Starting from empty decoder caches, the merged inputs $z_{I_k}^0$ of Equation~\eqref{eq:chunk-merge} pass through the chunk operator and readout:
\begin{align}
 y_{I_k},C_{b_k}^D
 &=D_\phi^{\mathrm{chunk}}(z_{I_k}^0;M,C_{a_k-1}^D,I_k),\\
 p_\Theta(x_{t+1}\mid x_{1:t})
 &=\softmax\bigl(W_o\Norm_o(y_t)\bigr)_{x_{t+1}},\qquad t\in I_k.
\end{align}
The chunk operator applies Equation~\eqref{eq:block} layer by layer, using all known positions in $I_k$ together.
At each layer, query $t$ reads decoder keys at $\max(1,t-W+1)\le j\le t$, including retained keys from earlier chunks, and encoder memory only at $j\le t$.
Neither SWA nor encoder positions reset at a chunk boundary.
The chunk may be longer or shorter than the SWA window.

Each decoder layer projects KV for all positions from that layer's inputs before masked attention, so causal SWA can process those positions in parallel.
The last-position output carries the chunk's causal computation to the next chunk, without pooling or an additional summary token.
A sequence fitting in one chunk still receives the learned initial-state merge, whereas RLT-0 removes that merge and its parameters.

\subsection{Training and prefill algorithm}
For teacher-forced training, $x_{1:T}$ denotes the consumed input tokens, with a next-token target wherever one is available.
The same forward procedure computes prompt prefill and known-history policy replay:
\begin{enumerate}
\item Encode the known tokens causally and build encoder KV memory.
Initialize $h=s_\star$, empty decoder SWA caches, and global position one.
\item Visit chunks in increasing order.
Compute $\Norm_s(h)$ and $W_s\Norm_s(h)$ once for the chunk; broadcast them to its token-specific gated merges.
\item For each decoder layer, project the chunk's Q/K/V together, apply causal SWA using retained and current-chunk KV, then apply prefix-masked memory attention and the FFN.
Retain the last $W-1$ decoder KV entries per layer for continuation.
\item Compute logits and the selected next-token losses at every position.
After a complete chunk, set $h$ to its last decoder output.
A partial final chunk leaves $h$ at the preceding complete boundary.
\item For training, accumulate the sequence loss and backpropagate through all chunks before updating parameters.
For prefill, save caches, $h$, the consumed-token count, and the last output or next-token logits.
\end{enumerate}
Full BPTT differentiates through the boundary state, cross-chunk decoder KV, and encoder memory.
Chunking changes the forward dependencies; gradient truncation and optimizer-update frequency are separate choices.
For RL, replay uses the same $B$, boundary anchor, masks, and positions as sampling, rebuilding the state under the current parameters before evaluating action probabilities.

\subsection{Incremental generation and partial chunks}
Let $t$ be the number of consumed tokens and $r=t\bmod B$ the current chunk offset.
The saved feedback register is $h_{\lfloor t/B\rfloor}$, including when the prompt ends inside a chunk.
To consume the next observed or sampled token at position $j=t+1$:
\begin{enumerate}
\item Update the causal encoder and append its memory KV.
Merge $e_j$ with the saved boundary state, keeping that state fixed throughout the chunk.
\item Run one decoder step with the usual causal SWA caches and prefix memory, producing $y_j$, updated KV caches, and next-token logits.
\item If $j\bmod B=0$, replace the feedback register by $y_j$; otherwise preserve it.
Advance the consumed-token count to $j$.
\end{enumerate}
The latest $y_j$ always predicts the next token, whether or not it updates the feedback register.
For example, with $B=4$ and a six-token prompt, tokens 5--8 all merge with $h_1=y_4$; processing tokens 5 and 6 during prefill does not replace that register with $y_6$.
The output $y_6$ predicts token 7, and $y_8$ becomes $h_2$ only after token 8 is consumed.
Known continuation tokens can be batched up to the next fixed boundary, then the procedure continues with the new state.
Padding, message boundaries, and the prompt--response split do not commit a partial chunk.
Independent packed sequences each maintain their own boundary anchor and state.

With fixed weights and exact arithmetic, batched and incremental execution agree if they use the same mathematical operators and stochastic behavior.
Stochastic operations must be disabled or use identical random draws for corresponding operations.
Each position then has the same encoder prefix, preceding boundary state, and causally available layerwise KV in both schedules.
Induction over layers within a chunk and then over chunks gives the same outputs and boundary states.
Changing $B$ or shifting the boundary anchor changes the model computation and invalidates a saved continuation state.

Chunkwise prefill and tokenwise decoding can use different matrix and attention kernels and floating-point reduction orders, altering logits and boundary states even under deterministic execution.
For RL, compare chunkwise replay against a tokenwise reference that reproduces the rollout's encoder, decoder, and readout execution, including precision, batch layout, kernel choices, and stochastic settings.

\subsection{Training and inference efficiency}
Table~\ref{tab:rlt-efficiency} assumes equal encoder and decoder depths, width, windows, and memory groups, with $L_D\ge1$.
Its decoder block stages count sequential dependencies in the layerwise schedules, excluding attention-kernel reductions, communication, and hardware-dependent latency.

At fixed dimensions, the common forward arithmetic with dense encoder and memory attention is
\begin{equation}
 O\bigl((L_E+L_D)Td^2+(L_E+L_D)T^2d
       +GTd^2+L_DT\min(W,T)d\bigr).
\end{equation}
RLT-1 adds $O(Td^2)$ merge work.
RLT-2 also adds $O(Td^2)$ merge work for $T$ token-specific gates, but reuses state normalization and projection within each chunk.
Its $K$ sequential decoder batches permit larger matrix operations and weight reuse across up to $B$ positions.
Total decoder arithmetic still covers all $T$ tokens, and full BPTT requires activations or recomputation across all chunks.

During generation at context length $t$, each variant evaluates its encoder and decoder stacks per token and reads the available encoder memory.
RLT-2 can reuse its feedback projection until the next boundary; the token-specific gate, decoder blocks, and KV writes remain per-token operations.
The shared inference-cache scaling is
\begin{equation}
 O\bigl((L_E+G)t\,d_{\mathrm{KV}}+
 L_D\min(t,W-1)d_{\mathrm{KV}}^D\bigr).
\end{equation}
RLT-1 and RLT-2 add $O(d)$ feedback storage; RLT-2 also retains a chunk offset and may cache $O(d)$ normalized/projected state.
Chunk batches require temporary activations and current-chunk KV during prefill.

A controlled benchmark would sweep $B\in\{1,2,4,8,16,32\}$ and RLT-0 at fixed data, stack dimensions, precision, optimizer, batch size, and attention kernels.
It would measure task accuracy, training throughput, prefill latency by prompt length, generation time per token by cached context length, and peak device memory.

The $B=1$ run would check RLT-1 equivalence, and future-token perturbations would check causality.
At unchanged parameters, maximum and RMS differences in logits, selected-action log-probabilities, and boundary states across chunkwise, partial-chunk, and tokenwise execution would separate execution differences from the effects of a policy update.
A differentiable reference should also compare full-BPTT parameter gradients, including paths through boundary states, decoder KV, and encoder memory.

\section{Experimental Details}
\label{app:experiment-details}

\subsection{Complete six-task study at 2,000 steps}
\label{app:depth8-complete}

This appendix reports all 108 runs of the eight-layer study: six architectures on six tasks with initialization seeds 42, 43, and 44, each trained for 2,000 optimizer steps.
It gives validation curves at training lengths and length generalization from each run's best in-distribution checkpoint.

\subsubsection{Models, tasks, and evaluation protocol}
Models, optimizer settings, batch sizes, and execution follow Section~\ref{sec:experiments}; the cosine schedule reaches $5\times10^{-6}$ at step 2,000.
Validation runs every 100 steps, and each run sees 1,024,000 training examples.

\begin{itemize}
\item \textbf{Addition:} randomly sample 1--8-digit operands and serialize their digits in reverse order.
The 256 validation examples use the same range.
We measure teacher-forced answer-token accuracy: each prediction receives the correct preceding tokens, and scoring includes answer formatting and EOS while excluding prompt and padding positions.
\item \textbf{Parity:} train on binary strings of lengths 3--40 and supervise the final parity bit.
Validation pools 256 examples at each of lengths 32, 36, and 40.
\item \textbf{Modular arithmetic:} evaluate expressions modulo five using $+$, $-$, and $\times$.
The flat variant respects multiplication precedence and trains on odd expression lengths 3--39, with validation lengths 33, 35, and 39.
The bracketed variant uses expression trees, trains on lengths 3--40, and validates at 32, 36, and 40.
Each validation length has 256 examples; only the final value is supervised.
\item \textbf{$S_5$ state tracking:} compose 32 permutations and supervise the running product after each input, following \citet{grazzi2024unlocking}.
Standard inputs range over all 120 permutations; the swaps variant uses the identity and the ten single transpositions.
Each variant has 256 validation sequences of length 32.
We report final-state accuracy as the primary metric and prefix-token accuracy in Table~\ref{tab:experiment-prefix}.
\end{itemize}

For each run, we pool counts across its configured validation lengths before averaging across initializations.
Figure~\ref{fig:experiment-curves} shows all three seeds at each optimizer step, and Table~\ref{tab:experiment-matched} compares every model at step 2,000.

\begin{figure}[p]
\centering
\includegraphics[width=\linewidth]{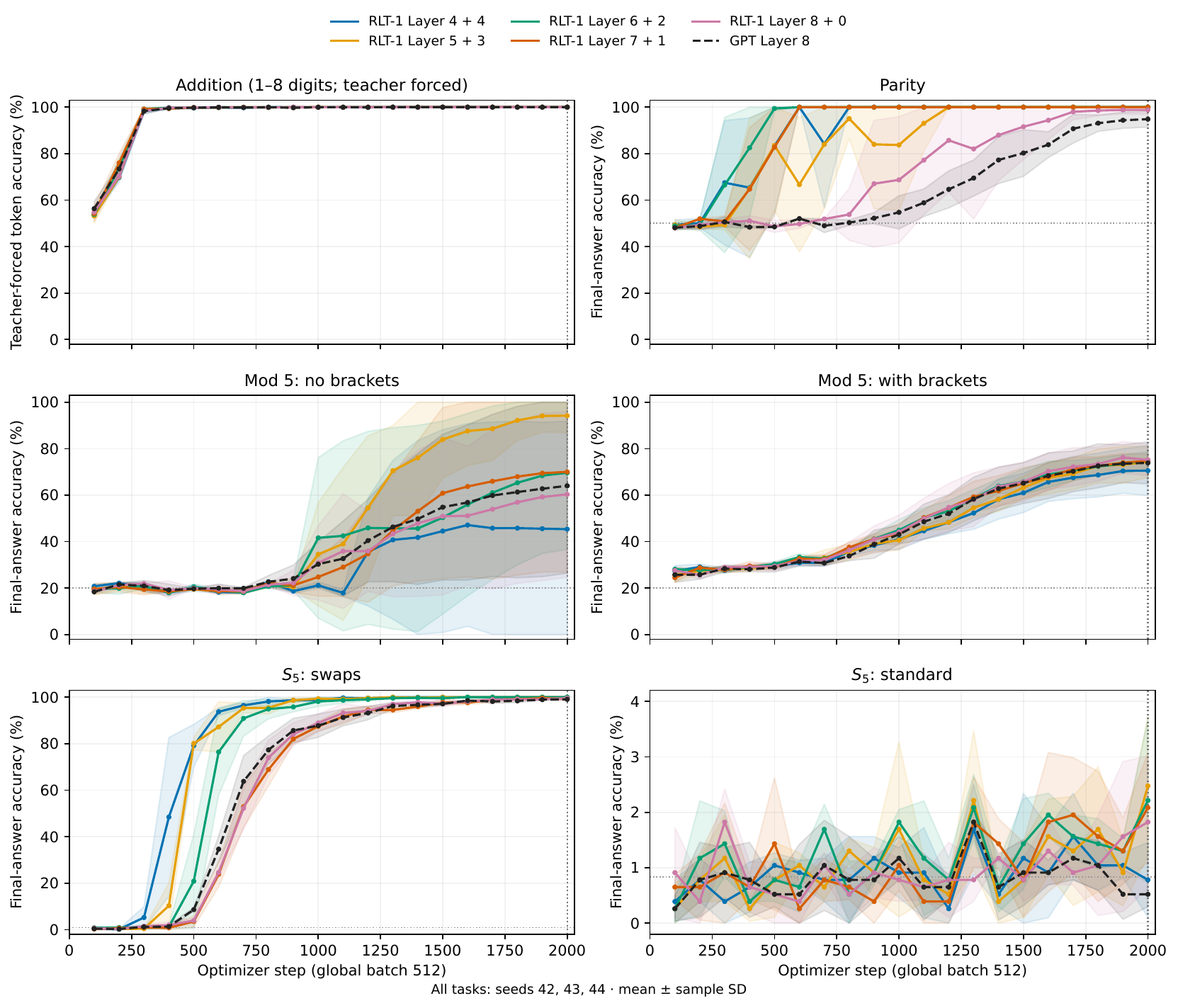}
\caption{\textbf{Validation accuracy during training, three seeds for every task.}
Addition measures teacher-forced answer-token accuracy; the other tasks measure final-label or final-state accuracy.
Lines and bands show mean $\pm$ sample SD across seeds 42, 43, and 44 ($n=3$ at every step); bands are clipped to the accuracy range.
Curves are unsmoothed and extend through 2,000 steps.
Horizontal dotted lines show uniform-prediction accuracy.
Standard $S_5$ uses a narrower vertical scale.}
\label{fig:experiment-curves}
\end{figure}

\begin{table}[t]
\centering
\caption{\textbf{Validation accuracy (\%) after 2,000 optimizer steps.}
All entries are mean $\pm$ sample SD over seeds 42, 43, and 44.
Addition uses teacher-forced answer tokens; the formal tasks use final labels or states.
Every run has consumed 1,024,000 training examples.}
\label{tab:experiment-matched}
\resizebox{\linewidth}{!}{\begin{tabular}{lrrrrrr}
\toprule
Task & \shortstack{RLT-1\\Layer 4 + 4} & \shortstack{RLT-1\\Layer 5 + 3} & \shortstack{RLT-1\\Layer 6 + 2} & \shortstack{RLT-1\\Layer 7 + 1} & \shortstack{RLT-1\\Layer 8 + 0} & \shortstack{GPT\\Layer 8} \\
\midrule
Addition & $100.00\,\pm\,0.00$ & $100.00\,\pm\,0.00$ & $100.00\,\pm\,0.00$ & $100.00\,\pm\,0.00$ & $100.00\,\pm\,0.00$ & $100.00\,\pm\,0.00$ \\
Parity & $100.00\,\pm\,0.00$ & $100.00\,\pm\,0.00$ & $100.00\,\pm\,0.00$ & $100.00\,\pm\,0.00$ & $98.83\,\pm\,1.92$ & $94.84\,\pm\,3.43$ \\
Mod 5, no brackets & $45.36\,\pm\,46.43$ & $94.18\,\pm\,7.29$ & $69.62\,\pm\,33.01$ & $70.01\,\pm\,43.15$ & $60.33\,\pm\,35.70$ & $64.02\,\pm\,37.64$ \\
Mod 5, brackets & $70.53\,\pm\,10.75$ & $74.35\,\pm\,6.71$ & $74.87\,\pm\,3.72$ & $74.78\,\pm\,3.01$ & $75.17\,\pm\,6.78$ & $73.87\,\pm\,9.07$ \\
$S_5$, swaps & $100.00\,\pm\,0.00$ & $100.00\,\pm\,0.00$ & $100.00\,\pm\,0.00$ & $99.35\,\pm\,0.81$ & $99.61\,\pm\,0.68$ & $99.09\,\pm\,0.23$ \\
$S_5$, standard & $0.78\,\pm\,0.68$ & $2.47\,\pm\,1.26$ & $2.21\,\pm\,1.48$ & $2.08\,\pm\,0.98$ & $1.82\,\pm\,1.19$ & $0.52\,\pm\,0.23$ \\
\bottomrule
\end{tabular}
}
\end{table}

\subsubsection{Parity: learning speed and initialization variability}
At 500 steps, RLT-1 Layer $6 + 2$ reaches $99.44\pm0.98\%$ validation accuracy, compared with $48.48\pm0.53\%$ for the Transformer.
RLT-1 Layer $4 + 4$, $5+3$, and $7+1$ average about $83\%$, with SDs of 28--30 percentage points; $8+0$ remains near chance.
All three $6+2$ seeds reach 100\% by step 600.
At step 2,000, splits $4+4$ through $7+1$ reach $100\pm0\%$, while $8+0$ reaches $98.83\pm1.92\%$ and the Transformer $94.84\pm3.43\%$.
Figure~\ref{fig:parity-errorbars} compares the early and final checkpoints.

\subsubsection{Modular arithmetic with and without brackets}
Flat modular arithmetic shows substantial initialization variability at 2,000 steps.
RLT-1 Layer $5 + 3$ has the highest mean, $94.18\pm7.29\%$, compared with $64.02\pm37.64\%$ for the Transformer.
The $4+4$ scores are 17.58\%, 19.53\%, and 98.96\% for seeds 42, 43, and 44; their mean is $45.36\pm46.43\%$.
At this training budget, several architectures produce both successful and near-chance runs.

Bracketed expressions produce closer model means: RLT-1 splits range from 70.53\% to 75.17\%, compared with $73.87\pm9.07\%$ for the Transformer.
The two generators differ in operator structure and label distribution.
Parentheses also consume positions, so equal token lengths can contain different numbers of arithmetic operations.

\subsubsection{Addition and permutation state tracking}
All six architectures reach $100\pm0\%$ teacher-forced token accuracy on the mixed 1--8-digit addition validation set at step 2,000.
On swaps-$S_5$, RLT-1 Layer $4 + 4$, $5+3$, and $6+2$ reach $100\pm0\%$ final-state accuracy.
The $7+1$ and $8+0$ means are $99.35\pm0.81\%$ and $99.61\pm0.68\%$, compared with $99.09\pm0.23\%$ for the Transformer.

On standard $S_5$, mean final-state accuracy is 0.52--2.47\%, against a uniform 120-class reference of 0.83\%.
Mean prefix-token accuracy is 5.40--7.99\% (Table~\ref{tab:experiment-prefix}); this metric includes easier early states.

\subsubsection{Length generalization}
\label{app:depth8-complete-generalization}
Checkpoint selection follows Section~\ref{sec:experiments}.
All 18 model--seed combinations for a task receive the same examples at each length: 256 addition pairs or 1,024 formal-task sequences.
Figure~\ref{fig:length-generalization} shows the full grids; Table~\ref{tab:generalization-longest} reports the longest tested lengths.
Appendix~\ref{app:length-generalization} gives checkpoint steps and supplementary metrics.

\paragraph{Addition and parity.}
Teacher-forced addition accuracy falls when operand widths exceed the 1--8-digit training range.
At nine digits per operand, RLT-1 Layer $7 + 1$ reaches $68.05\pm3.64\%$, compared with $58.61\pm2.44\%$ for the Transformer.
By 32 digits, model means lie between 14.89\% and 16.84\%.

Parity generalization differs across depth splits even when training-length accuracy is perfect.
At 256 bits, $5+3$ and $7+1$ retain $100\pm0\%$, while $6+2$ reaches $84.05\pm27.63\%$ and $4+4$ reaches $66.76\pm28.78\%$.
The $8+0$ mean is $68.91\pm27.39\%$, and the Transformer is near chance at $50.07\pm1.63\%$.

\paragraph{Modular arithmetic.}
On flat expressions of length 63, RLT-1 Layer $5 + 3$ reaches $66.96\pm40.04\%$, compared with $24.48\pm3.89\%$ for the Transformer.
At length 127, $5+3$ retains $39.65\pm19.81\%$; the Transformer is at $19.30\pm1.04\%$.
By length 255, all model means lie between 18.00\% and 21.42\%, near the 20\% uniform reference.
The large intermediate-length SDs reflect the different training outcomes seen in the flat mod-5 validation curves.
Bracketed mod-5 also loses accuracy as expressions grow longer.

\paragraph{Permutation state tracking.}
At 256 operations on swaps-$S_5$, RLT-1 Layer $4 + 4$ reaches $97.30\pm2.76\%$ final-state accuracy, compared with $0.85\pm0.30\%$ for the Transformer.
The $5+3$ and $6+2$ means are $92.71\pm1.87\%$ and $78.91\pm22.26\%$; $7+1$ and $8+0$ are near the uniform reference.
At 512 operations, $4+4$ retains $55.70\pm25.78\%$ final-state accuracy and $91.16\pm6.09\%$ prefix-token accuracy, versus $0.85\pm0.30\%$ and $9.33\pm0.11\%$ for the Transformer.
Standard $S_5$ remains low at this length, with mean final-state accuracies from 0.72\% to 1.43\%.

Figure~\ref{fig:generalization-s5} compares three measures of state-tracking accuracy across test lengths.
Prefix-token accuracy averages correctness over all intermediate states; final-state accuracy scores only the last state.
Whole-sequence accuracy requires every prefix prediction to be correct.
A correct final state can follow incorrect intermediate predictions, while prefix-token accuracy can remain high when errors occur late in the sequence.

\begin{figure}[p]
\centering
\includegraphics[width=\linewidth]{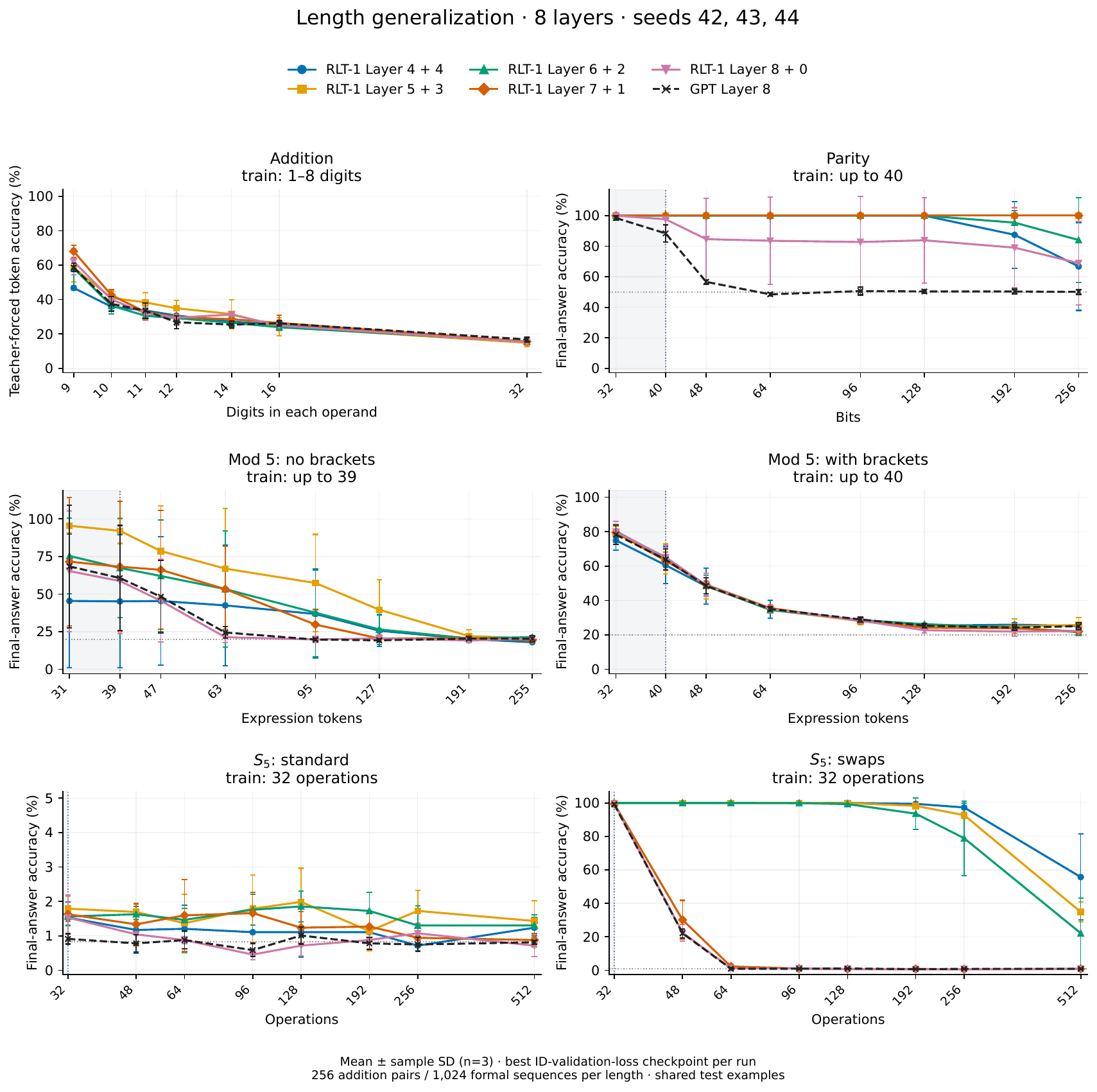}
\caption{\textbf{Length generalization of the best in-distribution checkpoints, three initialization seeds.}
Addition measures teacher-forced answer-token accuracy on 256 pairs per operand width; the other tasks measure final-label or final-state accuracy on 1,024 sequences per length.
Points and error bars show mean $\pm$ sample SD across seeds 42, 43, and 44, using identical test examples.
SD whiskers may extend outside the accuracy range.
Gray regions mark trained lengths; horizontal dotted lines show uniform-prediction accuracy.
Flat mod-5 uses actual odd expression lengths, excluding BOS and the final equals sign.}
\label{fig:length-generalization}
\end{figure}

\begin{figure}[p]
\centering
\includegraphics[width=\linewidth]{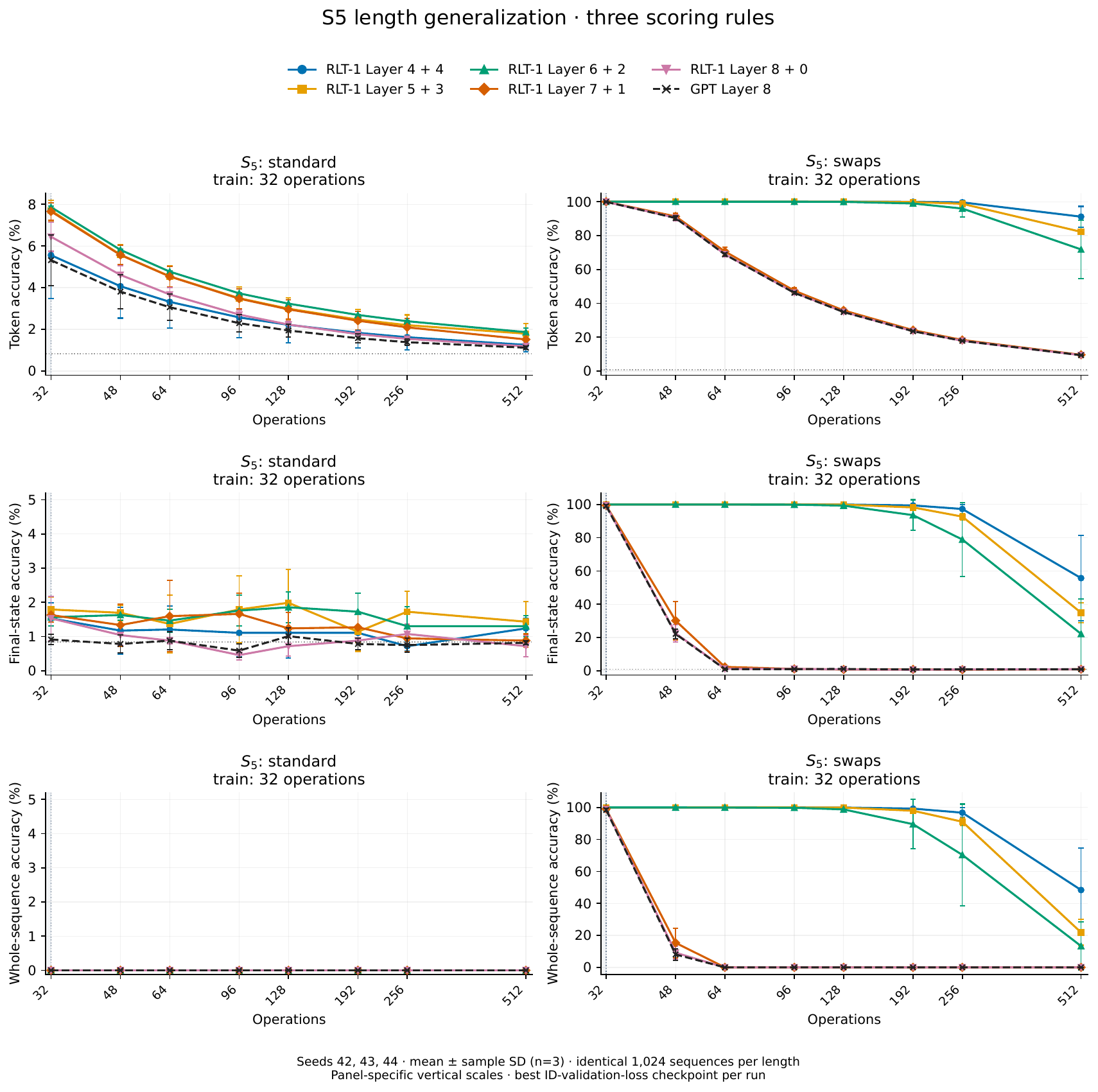}
\caption{\textbf{$S_5$ length generalization under three scoring rules.}
Rows show prefix-token, final-state, and whole-sequence accuracy; columns show standard and swaps inputs.
Every model and seed receives the same 1,024 sequences per length.
Points and error bars show mean $\pm$ sample SD across three initialization seeds.
Whole-sequence success requires all prefix predictions to be correct.
Panel-specific vertical scales expose low accuracies on standard $S_5$.}
\label{fig:generalization-s5}
\end{figure}

\begin{table}[t]
\centering
\caption{\textbf{Accuracy (\%) at the longest tested length.}
Entries are mean $\pm$ sample SD across three initialization seeds, using the checkpoint and scoring rules of Figure~\ref{fig:length-generalization}.
Addition reports teacher-forced token accuracy; formal tasks report final-label or final-state accuracy.}
\label{tab:generalization-longest}
\resizebox{\linewidth}{!}{\begin{tabular}{lrrrrrr}
\toprule
Task (test length) & \shortstack{RLT-1\\Layer 4 + 4} & \shortstack{RLT-1\\Layer 5 + 3} & \shortstack{RLT-1\\Layer 6 + 2} & \shortstack{RLT-1\\Layer 7 + 1} & \shortstack{RLT-1\\Layer 8 + 0} & \shortstack{GPT\\Layer 8} \\
\midrule
Addition (32) & $15.30\,\pm\,0.44$ & $14.89\,\pm\,2.27$ & $15.74\,\pm\,1.58$ & $15.59\,\pm\,1.57$ & $15.31\,\pm\,0.93$ & $16.84\,\pm\,1.45$ \\
Parity (256) & $66.76\,\pm\,28.78$ & $100.00\,\pm\,0.00$ & $84.05\,\pm\,27.63$ & $100.00\,\pm\,0.00$ & $68.91\,\pm\,27.39$ & $50.07\,\pm\,1.63$ \\
Mod 5, no brackets (255) & $18.00\,\pm\,0.39$ & $20.57\,\pm\,0.62$ & $21.42\,\pm\,0.49$ & $19.34\,\pm\,0.54$ & $20.44\,\pm\,0.91$ & $20.35\,\pm\,2.01$ \\
Mod 5, brackets (256) & $25.20\,\pm\,1.71$ & $25.81\,\pm\,4.34$ & $21.58\,\pm\,1.86$ & $22.04\,\pm\,1.72$ & $22.30\,\pm\,1.13$ & $25.07\,\pm\,2.05$ \\
$S_5$, standard (512) & $1.24\,\pm\,0.30$ & $1.43\,\pm\,0.60$ & $1.30\,\pm\,0.31$ & $0.88\,\pm\,0.20$ & $0.72\,\pm\,0.31$ & $0.81\,\pm\,0.06$ \\
$S_5$, swaps (512) & $55.70\,\pm\,25.78$ & $34.86\,\pm\,6.10$ & $22.14\,\pm\,20.95$ & $0.85\,\pm\,0.06$ & $0.85\,\pm\,0.20$ & $0.85\,\pm\,0.30$ \\
\bottomrule
\end{tabular}
}
\end{table}

\subsubsection{Training stability and comparison scope}
Parity RLT-1 Layer $5 + 3$ with seed 42 drops from 100\% at step 500 to 48.05\% at step 600, then returns to 100\% at step 700.
Figure~\ref{fig:experiment-loss} shows the unsmoothed losses, and Appendix~\ref{app:experiment-details} gives individual parity curves.
This temporary regression and the flat mod-5 seed differences show that final averages can hide unstable learning trajectories.

The comparison with a decoder-only Transformer changes feedback, attention structure, parameter count, and compute together.
Appendices~\ref{app:feedback-variants} and~\ref{app:mod5-variants} isolate the feedback with RLT-0 and RLT-2 at every split.

\clearpage

\subsection{Parameter counts}
Table~\ref{tab:experiment-parameters} lists unique trainable parameters.
Each RLT-1 decoder block has both SWA and encoder-memory cross-attention; each Transformer block has one attention sublayer.
RLT-1 Layer $8 + 0$ applies the learned state projection and gated merge at every token.

\begin{table}[htbp]
\centering
\caption{Unique trainable parameters in the eight-layer models.
Each model has the same size across all six tasks.}
\label{tab:experiment-parameters}
\begin{tabular}{lrr}
\toprule
Model & Parameters (M) & Relative to GPT Layer 8 \\
\midrule
RLT-1 Layer 4 + 4 & 28.73 & 1.135$\times$ \\
RLT-1 Layer 5 + 3 & 28.20 & 1.114$\times$ \\
RLT-1 Layer 6 + 2 & 27.68 & 1.093$\times$ \\
RLT-1 Layer 7 + 1 & 27.15 & 1.073$\times$ \\
RLT-1 Layer 8 + 0 & 26.10 & 1.031$\times$ \\
GPT Layer 8 & 25.31 & 1.000$\times$ \\
\bottomrule
\end{tabular}

\end{table}

\subsection{Implementation of RLT-0}
The \texttt{rlt\_parallel\_tiny} control was merged in \href{https://github.com/yifanzhang-pro/nanogptpro-dev/pull/164}{\texttt{nanogptpro-dev} PR~\#164}.
It uses an untied $4+4$ encoder--decoder split, width 512, four attention heads, FFN width 1,365, vocabulary capacity 277, an SWA window of eight, one encoder-memory group, and width-$\mu$P.
It has 27,938,944 unique parameters, and its architecture follows Section~\ref{sec:feedback-variants}.
Training uses full backpropagation, and architecture markers distinguish RLT-0 and RLT-1 checkpoints.

\begin{figure}[p]
\centering
% Training / prefill of known tokens; shared global memory (G=1).
\begingroup
\definecolor{detailInk}{HTML}{172B4D}
\definecolor{detailBlue}{HTML}{007AFF}
\definecolor{detailOrange}{HTML}{E87500}
\definecolor{detailGreen}{HTML}{20856A}
\resizebox{\linewidth}{!}{%
\begin{tikzpicture}[x=1cm,y=1cm,>=Latex,font=\sffamily\fontsize{8}{9.5}\selectfont,text=detailInk,
 box/.style={draw=detailInk!45,fill=white,rounded corners=2pt,minimum height=.65cm,text width=2.65cm,align=center,inner sep=4pt},
 enc/.style={box,draw=detailBlue!65,fill=detailBlue!6},
 dec/.style={box,draw=detailOrange!70,fill=detailOrange!7},
 mem/.style={box,draw=detailGreen,fill=detailGreen!7},
 add/.style={circle,draw=detailInk!65,fill=white,minimum size=.38cm,inner sep=0pt},
 flow/.style={->,line width=.8pt,draw=detailInk!80},
 global/.style={->,line width=1pt,draw=detailBlue},
 local/.style={->,line width=.9pt,draw=detailGreen},
 note/.style={font=\sffamily\fontsize{6.8}{8}\selectfont,align=center}]
\path[use as bounding box] (-.2,-.25) rectangle (17,16.3);
% Stack outlines are drawn first, behind their contents.
\filldraw[rounded corners=5pt,fill=detailBlue!3,draw=detailBlue!30] (.25,3.8) rectangle (4.35,10.7);
\filldraw[rounded corners=5pt,fill=detailOrange!3,draw=detailOrange!35] (8.55,3.8) rectangle (12.75,12.65);
\node[font=\sffamily\bfseries\fontsize{12}{14}\selectfont,text=detailBlue] at (2.3,15.75) {Causal encoder};
\node[font=\sffamily\bfseries\fontsize{12}{14}\selectfont,text=detailOrange] at (10.65,15.75) {RLT-0 decoder};
\node[note] at (2.3,15.15) {Known sequence $x_{1:T}$};
\node[note] at (10.65,15.15) {Parallel positions within each layer};
\node[anchor=west,font=\sffamily\bfseries\fontsize{9}{11}\selectfont] at (.48,10.3) {$L_E$ layers};
\node[anchor=west,font=\sffamily\bfseries\fontsize{9}{11}\selectfont] at (8.78,12.25) {$L_D$ layers};
% Encoder, bottom to top.
\node (token) at (2.3,.65) {Known tokens $x_{1:T}$};
\node[enc] (embed) at (2.3,1.65) {Token embeddings};
\node[enc] (esa) at (2.3,5.15) {Causal self-attention\\[-1pt]{\fontsize{6.8}{8}\selectfont causal mask $j\le t$}};
\node[add] (eadd1) at (2.3,6.55) {$+$};
\node[enc] (effn) at (2.3,7.75) {Feed-forward network};
\node[add] (eadd2) at (2.3,9.15) {$+$};
\node[enc] (feature) at (2.3,11.7) {Encoder outputs\\$e_{1:T}=E_\theta(x_{1:T})$};
\draw[flow] (token)--(embed);
\draw[flow] (embed)--(esa);
\draw[flow] (esa)--(eadd1);
\draw[flow] (eadd1)--(effn);
\draw[flow] (effn)--(eadd2);
\draw[flow] (eadd2)--(feature);
\draw[flow] (2.3,4.35)--(.55,4.35)|-(eadd1);
\draw[flow] (2.3,6.95)--(.55,6.95)|-(eadd2);
\node[note,text width=3.3cm] at (2.3,13.25) {Causal masking allows parallel\\encoder prefill across positions.};
% Projected memory lives between stacks.
\node[enc,text width=2.55cm] (proj) at (6.25,11.7) {Global K/V projection\\[-1pt]{\fontsize{6.8}{8}\selectfont from normalized $e_{1:T}$}};
\node[enc,text width=2.55cm,minimum height=1.1cm] (memory) at (6.25,9.55) {Projected global KV\\$(k_j,v_j)$, $j=1,\ldots,T$};
\draw[global] (feature)--(proj);
\draw[global] (proj)--node[right,note] {project once\\per position}(memory);
\node[note,text width=2.7cm,text=detailBlue] at (6.25,7.85) {Same projected KV\\read by all decoder layers\\($G=1$ shown)};
% Every decoder position receives its own encoder output directly.
\node[dec,text width=3.1cm,minimum height=.9cm] (input) at (10.65,1.65) {Direct encoder input\\$z_{1:T}^0=e_{1:T}$};
\draw[global] (feature.south)--(2.3,10.95)--(4.55,10.95)--(4.55,1.65)--(input.west);
% Decoder: three pre-normalized residual sublayers.
\node[dec] (sqkv) at (10.65,4.65) {RMSNorm\\Q/K/V projections};
\node[dec] (swa) at (10.65,5.85) {Causal SWA\\[-1pt]{\fontsize{6.8}{8}\selectfont window $W$, including $t$}};
\node[add] (dadd1) at (10.65,6.7) {$+$};
\node[dec] (mq) at (10.65,7.65) {Norm + query projection};
\node[dec] (cross) at (10.65,8.75) {Cross-attention\\[-1pt]{\fontsize{6.8}{8}\selectfont prefix mask $j\le t$}};
\node[add] (dadd2) at (10.65,9.6) {$+$};
\node[dec] (dffn) at (10.65,10.55) {RMSNorm + FFN};
\node[add] (dadd3) at (10.65,11.5) {$+$};
\draw[flow] (input)--(sqkv);
\draw[flow] (sqkv)--node[left,note] {$Q_{1:T}^\ell$}(swa);
\draw[flow] (swa)--(dadd1);
\draw[flow] (dadd1)--node[right,note] {$b_{1:T}^\ell$}(mq);
\draw[flow] (mq)--(cross);
\draw[flow] (cross)--(dadd2);
\draw[flow] (dadd2)--node[right,note] {$a_{1:T}^\ell$}(dffn);
\draw[flow] (dffn)--(dadd3);
\draw[flow] (10.65,4.05)--(8.8,4.05)|-(dadd1);
\draw[flow] (10.65,7.05)--(8.8,7.05)|-(dadd2);
\draw[flow] (10.65,9.95)--(8.8,9.95)|-(dadd3);
\draw[global] (memory.east)--(8.15,9.55)--(8.15,8.75)--node[above,note] {K,V}(cross.west);
% Each layer projects KV at all known positions, then applies the SWA mask.
\node[mem,text width=2.4cm,minimum height=1.15cm] (cache) at (14.55,5.2) {Layer $\ell$ SWA KV\\$\{(k_j^{D,\ell},v_j^{D,\ell})\}_{j=1}^T$};
\draw[local] (sqkv.east)--(12.95,4.65)--(12.95,3.4)--(14.55,3.4)--(cache.south);
\draw[local] (cache.west)--(12.95,5.2)--(12.95,5.85)--node[above,note] {K,V}(swa.east);
\node[mem,text width=2.4cm,minimum height=1cm] (nextcache) at (14.55,8.15) {For continuation\\save $C_T^{D,\ell}$\\[-1pt]{\fontsize{6.8}{8}\selectfont last $W-1$ entries}};
\draw[local] (cache)--(nextcache);
\node[note,text width=2.5cm,text=detailGreen] at (14.55,9.55) {Separate KV per layer;\\window at position $t$:\\$j\le t$ and $t-j<W$};
% Outputs are read out independently at each position, with no feedback path.
\node (output) at (10.65,13.2) {$z_{1:T}^{L_D}$};
\node[dec] (head) at (10.65,14.15) {Readout + softmax\\$p_\Theta(x_{t+1}\mid x_{1:t})$};
\draw[flow] (dadd3)--node[right,note,pos=.45] {$z_{1:T}^\ell$}(output);
\draw[flow] (output)--(head);
\node[note,text width=2.5cm] at (14.55,14.15) {One prediction\\per position $t$};
\node[dec,text width=2.4cm,minimum height=1cm] at (14.55,12.1) {Training / prefill\\[-1pt]{\fontsize{6.8}{8}\selectfont Positions in parallel;\\layers in sequence.}};
\node[note,text width=2.5cm] at (14.55,1.65) {Generation remains\\one token at a time.};
\node[note,text width=3.4cm] at (6.25,4.8) {Attention includes output projection.\\$+$ denotes a residual addition.\\Keys include positional\\transformations.};
\end{tikzpicture}%
}
\endgroup
\caption{\textbf{RLT-0 during training or prompt prefill.}
The decoder receives the encoder outputs directly, $z_{1:T}^0=e_{1:T}$, with the gated merge and final-hidden-state feedback path removed.
Each outlined stack repeats for its indicated depth.
For $G=1$, all decoder layers read the same projected global KV; query position $t$ can read encoder positions $j\le t$.
Each decoder layer constructs its own SWA KV and applies a causal window of size $W$, including the current position.
Known positions run together within each layer, while layers follow depth order.
Incremental generation retains the encoder cache, global memory, and each layer's last $W-1$ SWA entries and proceeds one token at a time.}
\label{fig:architecture-no-feedback}
\end{figure}
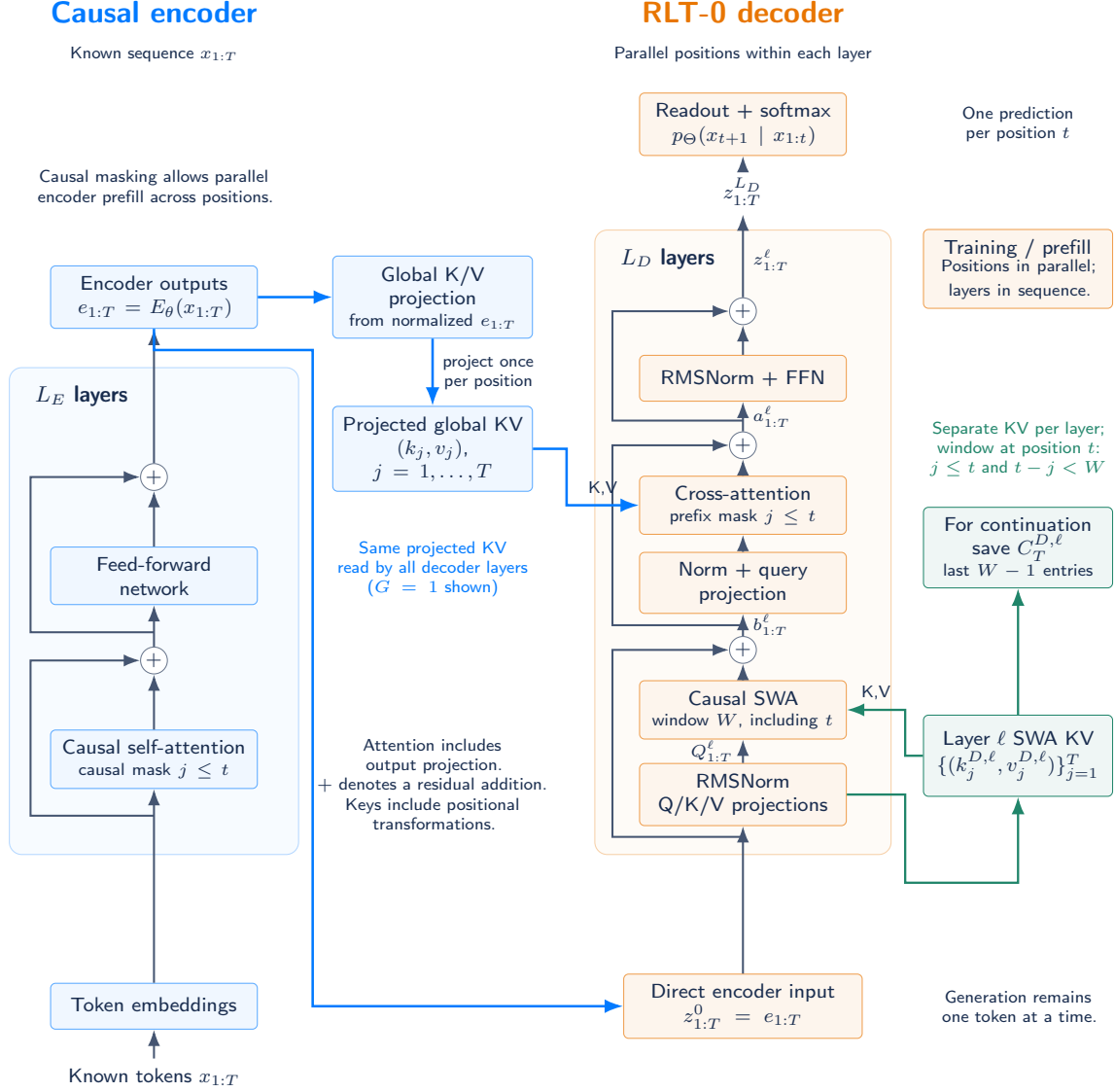

\subsection{Seed aggregation and metric definitions}
Let $a_{m,s,t}$ be validation accuracy for model $m$, initialization seed $s$, and optimizer step $t$.
For every task we report
\begin{equation}
 \bar a_{m,t}=\frac{1}{3}\sum_{s\in\{42,43,44\}}a_{m,s,t},
 \qquad
 \mathrm{SD}_{m,t}=\sqrt{\frac{1}{2}\sum_{s\in\{42,43,44\}}(a_{m,s,t}-\bar a_{m,t})^2}.
\end{equation}
We aggregate all three seeds at the same step, without interpolation or imputation.
All curves extend through step 2,000, the common comparison step in Table~\ref{tab:experiment-matched}.

Parity and mod-5 supervise one final label per example.
Their three equally sized validation groups are pooled within a run before averaging across seeds.
$S_5$ supervises each prefix; final-state accuracy counts the last state in each sequence, while prefix-token accuracy pools all supervised positions.
Addition uses teacher forcing on the full prompt and correct preceding answer tokens.
Its denominator includes all answer tokens---digits, spaces, brackets, and EOS---and excludes prompt and padding positions.
Addition loss is teacher-forced cross-entropy normalized per example.

\begin{table}[htbp]
\centering
\caption{Final-state and prefix-token accuracy (\%) on the fixed $S_5$ validation sets at step 2,000, reported as mean $\pm$ sample SD over seeds 42, 43, and 44.}
\label{tab:experiment-prefix}
\resizebox{\linewidth}{!}{\begin{tabular}{llrrrrrr}
\toprule
Task & Metric & \shortstack{RLT-1\\Layer 4 + 4} & \shortstack{RLT-1\\Layer 5 + 3} & \shortstack{RLT-1\\Layer 6 + 2} & \shortstack{RLT-1\\Layer 7 + 1} & \shortstack{RLT-1\\Layer 8 + 0} & \shortstack{GPT\\Layer 8} \\
\midrule
$S_5$, swaps & Final state & $100.00\,\pm\,0.00$ & $100.00\,\pm\,0.00$ & $100.00\,\pm\,0.00$ & $99.35\,\pm\,0.81$ & $99.61\,\pm\,0.68$ & $99.09\,\pm\,0.23$ \\
$S_5$, swaps & Prefix token & $100.00\,\pm\,0.00$ & $100.00\,\pm\,0.00$ & $100.00\,\pm\,0.00$ & $99.96\,\pm\,0.05$ & $99.95\,\pm\,0.06$ & $99.91\,\pm\,0.06$ \\
$S_5$, standard & Final state & $0.78\,\pm\,0.68$ & $2.47\,\pm\,1.26$ & $2.21\,\pm\,1.48$ & $2.08\,\pm\,0.98$ & $1.82\,\pm\,1.19$ & $0.52\,\pm\,0.23$ \\
$S_5$, standard & Prefix token & $5.57\,\pm\,2.11$ & $7.82\,\pm\,0.55$ & $7.99\,\pm\,0.29$ & $7.67\,\pm\,0.54$ & $6.47\,\pm\,0.65$ & $5.40\,\pm\,1.21$ \\
\bottomrule
\end{tabular}
}
\end{table}

\subsection{Reproducing the completed training snapshot}
The formal tasks use data seeds 20260914 for training, 20260915 for validation, and 20260916 for test sets.
Addition uses data seed 42.
All 108 runs completed 2,000 steps using the implementation based on upstream commit \texttt{d1a4516}.
Native histories were frozen on September 17, 2026; the manifest records exact collection timestamps and source hashes.

\texttt{manifest.json} and \texttt{configs.json} in \texttt{data/experiments-depth8/} record run identifiers, settings, and source/corpus fingerprints.
\texttt{plot\_depth8.py} in \texttt{scripts/} rebuilds training figures and tables from the 2,160 validation records and 216,000 unsmoothed training-step records.
The fixed-length token-accuracy export contains 1,800 validation records from 90 completed formal-task runs; \texttt{plot\_token\_accuracy.py} in the same directory rebuilds its plot.

\begin{center}
\begin{minipage}{.94\linewidth}
\captionsetup{type=figure}
\centering
\includegraphics[width=\linewidth]{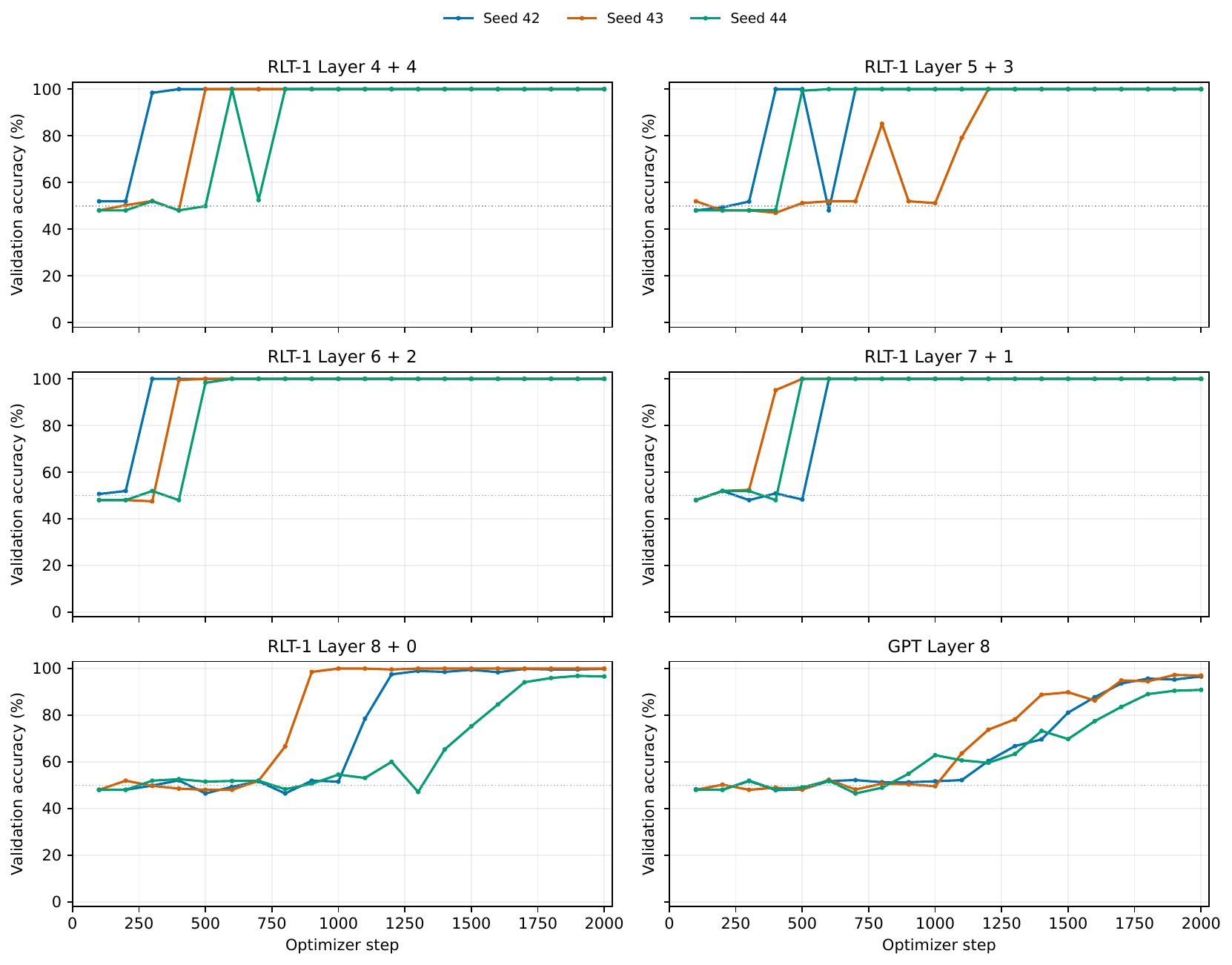}
\caption{Parity accuracy for each initialization seed through 2,000 steps.
All runs share the indexed training stream and validation set; curves show temporary regressions and differences in learning speed.}
\label{fig:parity-seeds}
\end{minipage}
\end{center}

\subsection{Length-generalization protocol and supplementary metrics}
\label{app:length-generalization}
The study uses all 108 completed runs, with seeds 42, 43, and 44 for every architecture and task.
We select each run's checkpoint by minimum native ID-validation example loss, pooling the configured formal-task validation lengths and taking the earliest step on ties.
Table~\ref{tab:generalization-checkpoints} lists the selected steps; checkpoint and training-result hashes accompany the data.
Test results do not participate in selection.

\begin{table}[p]
\centering
\caption{Checkpoint steps selected by minimum in-distribution validation loss for length generalization.
Every run completed 2,000 steps; seeds 42, 43, and 44 are listed separately.}
\label{tab:generalization-checkpoints}
\resizebox{\linewidth}{!}{\begin{tabular}{llrrrrrr}
\toprule
Task & Seed & \shortstack{RLT-1\\Layer 4 + 4} & \shortstack{RLT-1\\Layer 5 + 3} & \shortstack{RLT-1\\Layer 6 + 2} & \shortstack{RLT-1\\Layer 7 + 1} & \shortstack{RLT-1\\Layer 8 + 0} & \shortstack{GPT\\Layer 8} \\
\midrule
Addition & 42 & 2000 & 2000 & 2000 & 2000 & 2000 & 2000 \\
 & 43 & 2000 & 2000 & 2000 & 2000 & 2000 & 2000 \\
 & 44 & 2000 & 2000 & 2000 & 1900 & 2000 & 2000 \\
\addlinespace[2pt]
Parity & 42 & 900 & 1600 & 800 & 1500 & 2000 & 2000 \\
 & 43 & 1100 & 2000 & 900 & 1500 & 1900 & 1900 \\
 & 44 & 1400 & 1000 & 1200 & 1100 & 2000 & 2000 \\
\addlinespace[2pt]
Mod 5, no brackets & 42 & 800 & 2000 & 2000 & 2000 & 2000 & 400 \\
 & 43 & 800 & 2000 & 2000 & 800 & 400 & 2000 \\
 & 44 & 2000 & 2000 & 2000 & 2000 & 2000 & 1900 \\
\addlinespace[2pt]
Mod 5, brackets & 42 & 2000 & 2000 & 2000 & 2000 & 2000 & 1900 \\
 & 43 & 2000 & 2000 & 2000 & 2000 & 2000 & 2000 \\
 & 44 & 2000 & 2000 & 2000 & 2000 & 2000 & 2000 \\
\addlinespace[2pt]
$S_5$, standard & 42 & 1800 & 2000 & 2000 & 2000 & 2000 & 2000 \\
 & 43 & 2000 & 2000 & 2000 & 2000 & 2000 & 2000 \\
 & 44 & 2000 & 2000 & 1800 & 2000 & 2000 & 2000 \\
\addlinespace[2pt]
$S_5$, swaps & 42 & 2000 & 2000 & 2000 & 1900 & 2000 & 2000 \\
 & 43 & 2000 & 2000 & 2000 & 2000 & 2000 & 1900 \\
 & 44 & 2000 & 2000 & 2000 & 1900 & 2000 & 2000 \\
\addlinespace[2pt]
\bottomrule
\end{tabular}
}
\end{table}

The addition grid is 9, 10, 11, 12, 14, 16, and 32 digits in each operand.
Parity and bracketed mod-5 use lengths 32, 40, 48, 64, 96, 128, 192, and 256; both $S_5$ tasks use 32, 48, 64, 96, 128, 192, 256, and 512.
Flat expressions alternate operands and binary operators, so each requested even length $L$ maps to actual expression length $L-1$: 31, 39, 47, 63, 95, 127, 191, and 255.
Formal-task lengths exclude BOS and any final equals sign; flat mod-5 has $L+1$ input tokens after these markers are added.
There are 47 shared test sets and 846 model--seed--length evaluations.

Addition uses 256 unique unordered operand pairs per length, sampled from the native test partition with seed 20260916.
Both operands have exactly the indicated width and no leading zero; reversed-digit serialization matches training.
The evaluation computes teacher-forced answer-token accuracy without autoregressive generation.
Each formal-task set contains 1,024 sequences from the original generator and test seed 20260916.
All models and initializations use the same examples at a task and length.
Shared holdouts are reused from the previous benchmark; additional formal-task sets exclude duplicates and existing validation/test examples.

After matching checkpoints and holdout fingerprints, we reuse 697 model--seed--length results and evaluate the remaining 149 combinations.
All reported metrics are recomputed from saved per-example predictions or teacher-token counts.
Evaluation uses CPU FP32, four intra-op threads, and the recorded training runtime; checkpoint weights remain fixed.
New evaluations use batch 32, and reused native formal evaluations used batch four.
A batch-size check with RLT-1 Layer $4 + 4$ matched all 32,768 state predictions across the two batch sizes on 32 length-512 sequences per $S_5$ task.
The teacher-forced addition evaluator also reproduces the saved per-example token counts for RLT-1 Layer $4 + 4$ and GPT Layer 8 on 32 shared 32-digit pairs.

At each test length, we first compute a score for each initialization and then apply the mean/SD formula above with length in place of training step.
Error bars show sample SD across the three initializations on shared test examples.
Figure~\ref{fig:generalization-token} reports supplementary token accuracy for all six tasks; the three $S_5$ scoring rules are compared in Appendix~\ref{app:depth8-complete-generalization} (Figure~\ref{fig:generalization-s5}).

The directory \path{data/experiments-depth8-generalization/} contains 846 per-seed rows, 282 aggregates, checkpoint/dataset fingerprints, and a verification report.
Before regenerating figures and tables, \texttt{scripts/plot\_generalization.py} checks seed/length coverage, scores computed from counts, sample SDs, and checkpoint correspondence with the training snapshot.

\begin{figure}[p]
\centering
\includegraphics[width=\linewidth]{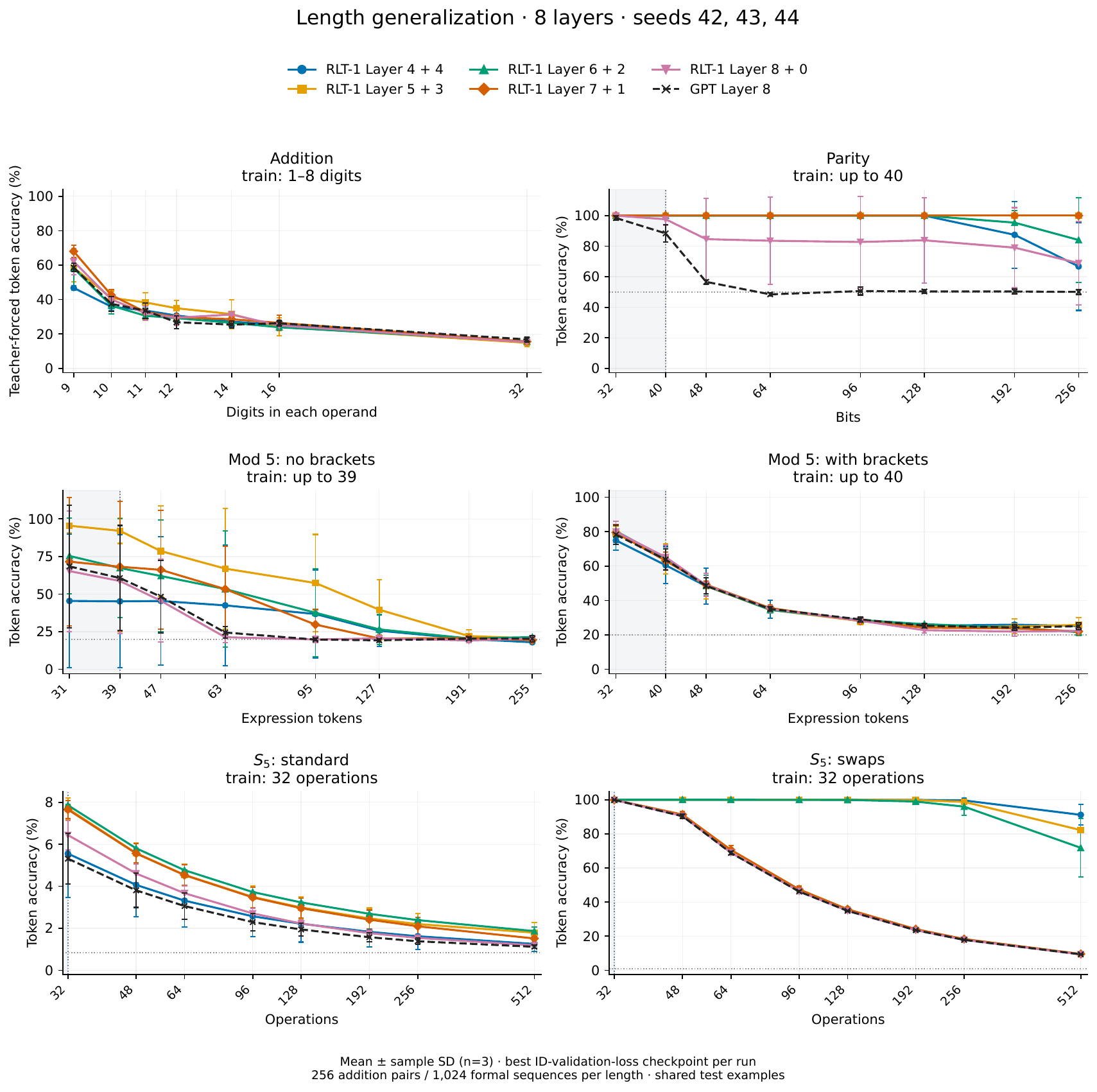}
\caption{\textbf{Token accuracy across test lengths, three initialization seeds.}
The checkpoints and shared examples are those of Figure~\ref{fig:length-generalization}.
Addition uses teacher forcing on answer tokens; $S_5$ scores every prefix state; parity and mod-5 each score one final label.
Points and error bars show mean $\pm$ sample SD across seeds 42, 43, and 44.
Gray regions mark trained lengths.}
\label{fig:generalization-token}
\end{figure}

\begin{figure}[p]
\centering
\includegraphics[width=\linewidth]{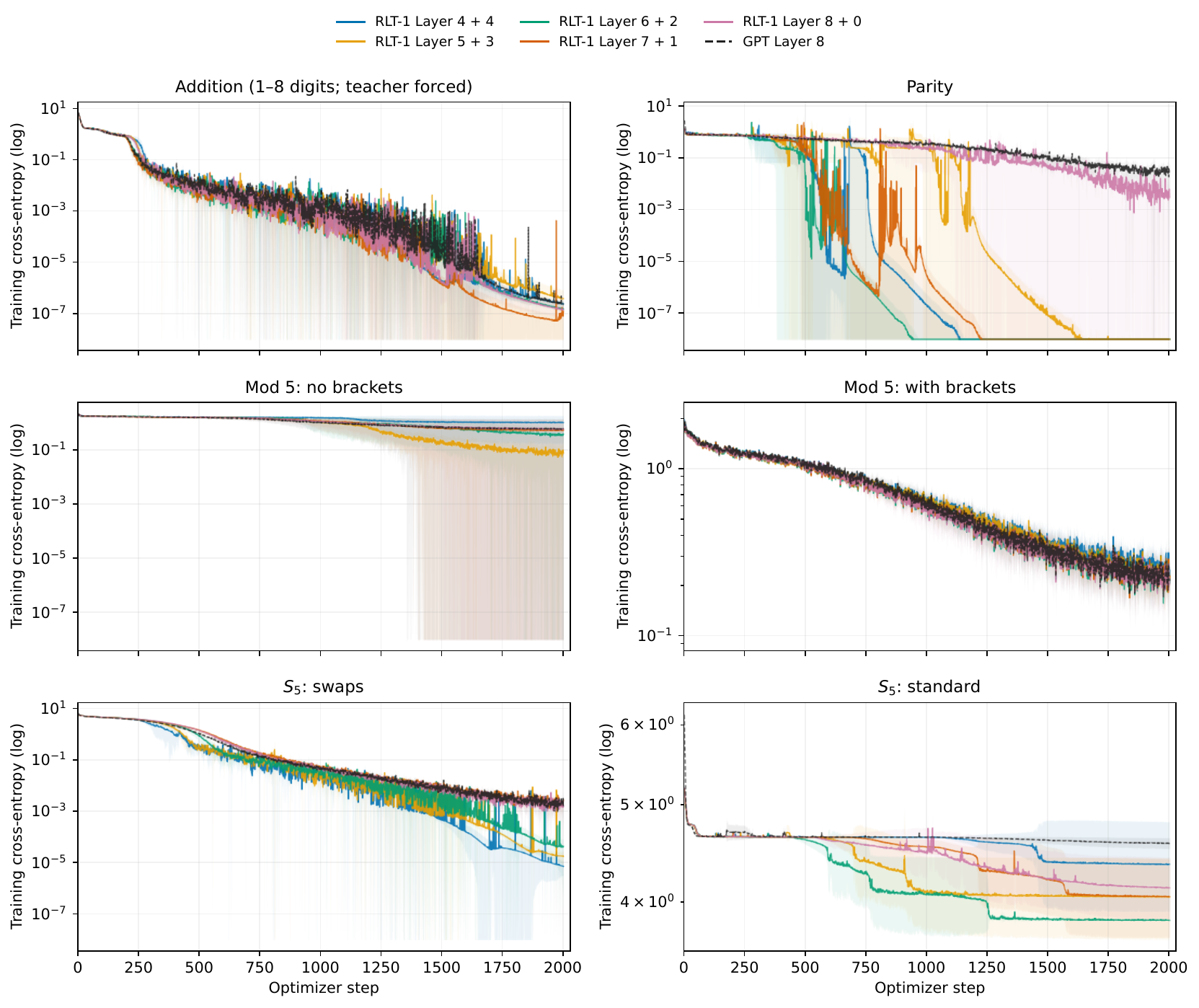}
\caption{Unsmoothed global-batch training cross-entropy, shown as mean $\pm$ sample SD across three seeds for every task.
Means and band limits below $10^{-8}$ are displayed at $10^{-8}$ on the logarithmic axis.
Loss is measured before the optimizer update; logged gradient norms are measured before clipping.
Compare losses within each task, since supervision differs across tasks.}
\label{fig:experiment-loss}
\end{figure}
\clearpage

\subsection{Token-accuracy trajectories at lengths 32 and 33}
Figure~\ref{fig:validation-token-accuracy} compares token accuracy at length 32 for parity, bracketed mod-5, and both $S_5$ tasks, and at length 33 for flat mod-5.
Flat expressions alternate numbers and binary operators, so their payload lengths are odd; 33 is the nearest configured validation length to 32.
The completed snapshot contains 1,800 validation checkpoints from 90 runs: six architectures, five formal tasks, and three seeds, each trained for 2,000 steps.
At length 32, $S_5$ token accuracy scores 8,192 prefix states across 256 sequences; parity and mod-5 each score one final label per sequence.

\begin{center}
\begin{minipage}{\linewidth}
\captionsetup{type=figure}
\centering
\includegraphics[width=.9\linewidth]{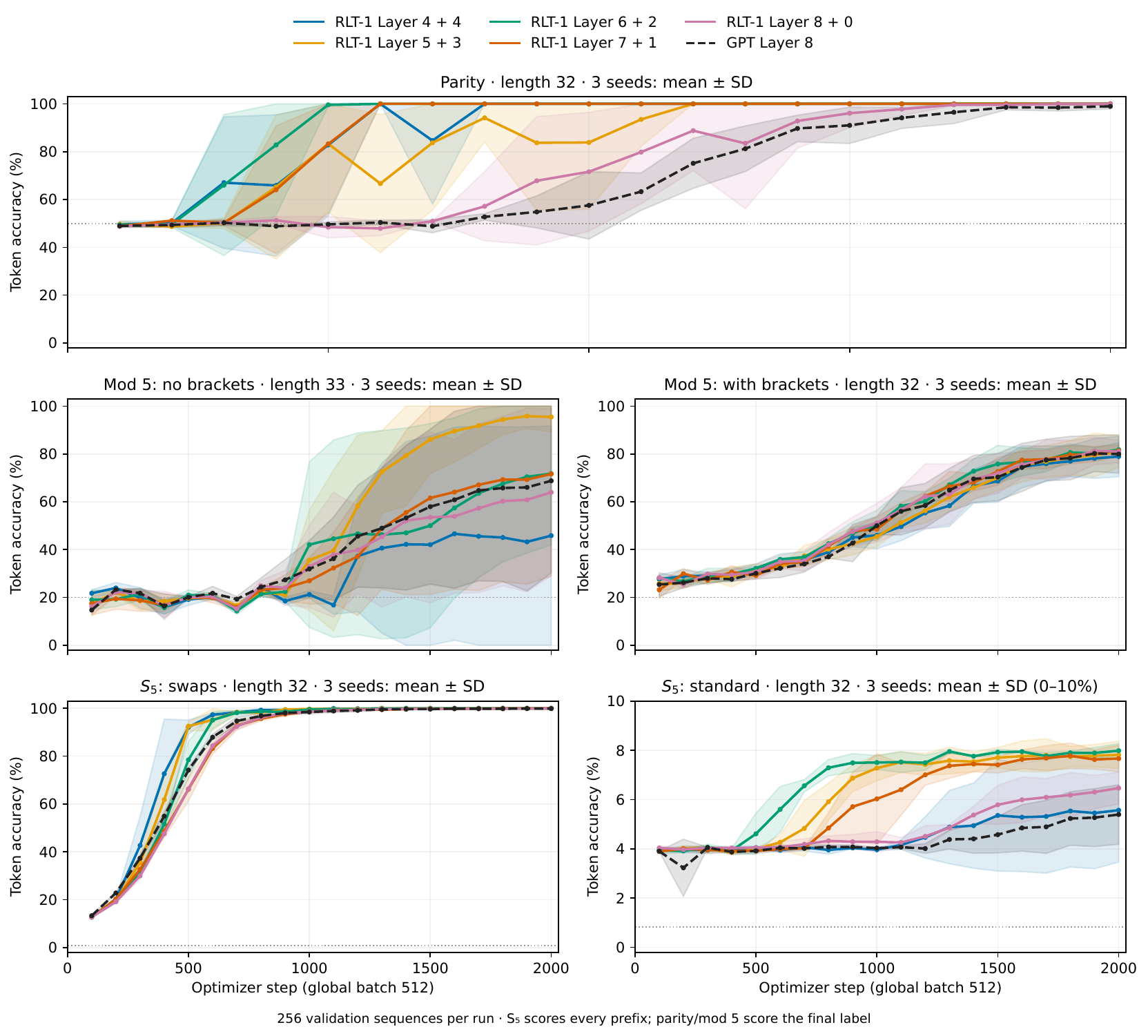}
\caption{\textbf{Validation token accuracy on five formal tasks, three seeds per task.}
Flat mod-5 uses payload length 33; all other panels use length 32.
Every point uses all three initialization seeds at the same optimizer step; bands show mean $\pm$ sample SD, clipped to the accuracy range.
$S_5$ scores all prefix states; parity and mod-5 score the final label.
All curves extend through 2,000 steps.
Standard $S_5$ uses a narrower vertical scale; horizontal dotted lines show uniform-prediction accuracy.}
\label{fig:validation-token-accuracy}
\end{minipage}
\end{center}
\clearpage
\subsection{Sixteen-layer parity: best and final checkpoints}
\label{app:parity-depth16}
All ten seed-42 parity runs completed 2,000 updates with training lengths 3--40, global batch 1,024, and microbatch 32.
Figure~\ref{fig:parity-depth16} compares the best ID-loss checkpoint with the fixed step-2,000 checkpoint on the same 1,024 examples per length.
Exact ID-loss ties select the earliest checkpoint, giving steps 1,400 and 1,300 for $8+8$ and $9+7$.
Paired evaluations use matching batch sizes; their best-checkpoint predictions agree with the previously reported predictions.

Accuracy is unchanged in 74 of 80 model--length comparisons.
At 256 bits, $10+6$ changes from 62.70\% to 66.11\%, $14+2$ from 63.57\% to 62.11\%, and $15+1$ from 94.14\% to 94.82\%.
RLT-1 Layer $8 + 8$, $9+7$, $11+5$, and $16+0$ retain 100\% at both checkpoints, compared with 49.41\% for GPT Layer 16.

\begin{table}[t]
\centering
\caption{\textbf{Sixteen-layer parity at 256 bits, seed 42.}
Final-answer accuracy (\%) at the best ID-loss checkpoint and at step 2,000.}
\label{tab:depth16-refresh}
\small
\begin{tabular}{lrrr}
\toprule
Model & Best step & Best checkpoint & Step 2,000 \\
\midrule
RLT-1 Layer 8 + 8 & 1400 & 100.00 & 100.00 \\
RLT-1 Layer 9 + 7 & 1300 & 100.00 & 100.00 \\
RLT-1 Layer 10 + 6 & 1800 & 62.70 & 66.11 \\
RLT-1 Layer 11 + 5 & 1300 & 100.00 & 100.00 \\
RLT-1 Layer 12 + 4 & 1000 & 98.83 & 98.83 \\
RLT-1 Layer 13 + 3 & 2000 & 99.41 & 99.41 \\
RLT-1 Layer 14 + 2 & 1500 & 63.57 & 62.11 \\
RLT-1 Layer 15 + 1 & 1100 & 94.14 & 94.82 \\
RLT-1 Layer 16 + 0 & 1900 & 100.00 & 100.00 \\
GPT Layer 16 & 2000 & 49.41 & 49.41 \\
\bottomrule
\end{tabular}

\end{table}

\begin{figure}[p]
\centering
\includegraphics[width=\linewidth]{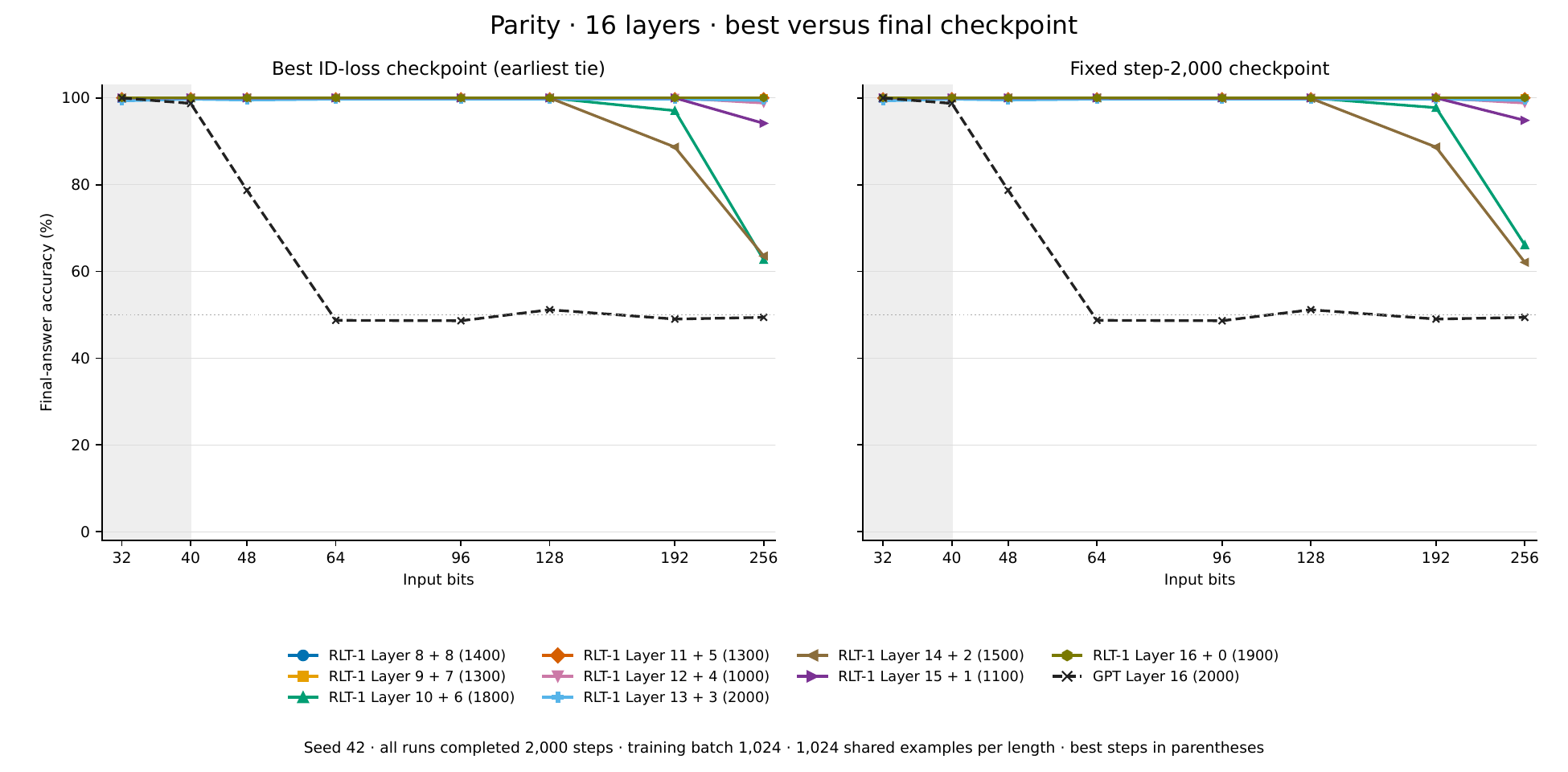}
\caption{\textbf{Sixteen-layer parity: best versus final checkpoint.}
Seed 42, 1,024 shared examples per length; parentheses give best ID-loss checkpoint steps.
Gray shading marks the training range; the dotted line marks 50\% chance.}
\label{fig:parity-depth16}
\end{figure}

\clearpage
\subsection{Three-seed comparison of feedback variants}
\label{app:feedback-variants}
We extend the eight-layer study to RLT-0 and RLT-2 chunk4, each at splits $4+4$ through $8+0$ and seeds 42, 43, and 44.
These 180 runs train for 5,000 updates on mod-5 and 2,000 on the other four tasks.
Width, data generation, optimizer, batch sizes, and CPU execution match Section~\ref{sec:experiments}; the cosine schedule ends at the corresponding training budget.
RLT-1 and RLT-2 use feedback scale 0.1 and TBPTT128; RLT-0 removes the feedback merge, has 787,968 fewer parameters per split, and uses full BPTT.
The three-seed RLT-1 and Transformer controls for addition, parity, and $S_5$ are reused from the original study.

Table~\ref{tab:variant-validation} reports validation at the end of training for all 288 runs.
Figure~\ref{fig:variant-validation} shows $4+4$ training curves; the data repository provides curves for every split.
Every validation mean uses seeds 42, 43, and 44 at the same step, including the Transformer mod-5 controls.

Formal-task tests select the minimum ID validation example loss, with earliest-step tie breaking.
We recount final-state, prefix-token, and whole-sequence scores from saved predictions and verify identical holdouts across variants and seeds.
Each length has 1,024 examples; this comparison uses the native test grid through length 256 (255 for flat mod-5).
The separate 512-operation $S_5$ results in Section~\ref{sec:length-generalization} remain specific to RLT-1 and Transformer.

\begin{table}[t]
\centering
\caption{\textbf{End-of-training validation for all eight-layer variants.}
Mean $\pm$ sample SD over seeds 42--44, in percent; addition uses teacher-forced tokens, other tasks use final answers or states.
Addition, parity, and $S_5$ use 2,000 steps; mod-5 uses 5,000.}
\label{tab:variant-validation}
\resizebox{\linewidth}{!}{\begin{tabular}{llrrrrrr}
\toprule
Family & Layers & Addition & Parity & $S_5$ swaps & $S_5$ standard & Mod-5 flat & Mod-5 brackets \\
\midrule
RLT-1 & 4 + 4 & $100.00 \pm 0.00$ & $100.00 \pm 0.00$ & $100.00 \pm 0.00$ & $0.78 \pm 0.68$ & $46.48 \pm 46.35$ & $97.31 \pm 1.14$ \\
RLT-1 & 5 + 3 & $100.00 \pm 0.00$ & $100.00 \pm 0.00$ & $100.00 \pm 0.00$ & $2.47 \pm 1.26$ & $99.13 \pm 1.28$ & $98.26 \pm 0.49$ \\
RLT-1 & 6 + 2 & $100.00 \pm 0.00$ & $100.00 \pm 0.00$ & $100.00 \pm 0.00$ & $2.21 \pm 1.48$ & $99.65 \pm 0.33$ & $98.00 \pm 1.48$ \\
RLT-1 & 7 + 1 & $100.00 \pm 0.00$ & $100.00 \pm 0.00$ & $99.35 \pm 0.81$ & $2.08 \pm 0.98$ & $98.39 \pm 2.78$ & $97.14 \pm 1.16$ \\
RLT-1 & 8 + 0 & $100.00 \pm 0.00$ & $98.83 \pm 1.92$ & $99.61 \pm 0.68$ & $1.82 \pm 1.19$ & $81.12 \pm 30.90$ & $96.57 \pm 0.85$ \\
\midrule
RLT-0 & 4 + 4 & $100.00 \pm 0.00$ & $66.45 \pm 18.12$ & $98.96 \pm 0.23$ & $0.65 \pm 0.81$ & $93.49 \pm 1.92$ & $94.27 \pm 2.00$ \\
RLT-0 & 5 + 3 & $100.00 \pm 0.00$ & $69.66 \pm 10.15$ & $98.05 \pm 0.68$ & $0.39 \pm 0.39$ & $73.52 \pm 24.88$ & $92.53 \pm 2.76$ \\
RLT-0 & 6 + 2 & $100.00 \pm 0.00$ & $58.51 \pm 1.50$ & $98.57 \pm 0.60$ & $1.56 \pm 1.17$ & $63.72 \pm 38.67$ & $95.88 \pm 1.64$ \\
RLT-0 & 7 + 1 & $100.00 \pm 0.00$ & $83.12 \pm 25.82$ & $99.09 \pm 0.81$ & $1.43 \pm 0.60$ & $97.48 \pm 1.01$ & $97.74 \pm 0.53$ \\
RLT-0 & 8 + 0 & $100.00 \pm 0.00$ & $85.68 \pm 18.25$ & $98.96 \pm 0.81$ & $0.91 \pm 0.60$ & $97.44 \pm 1.73$ & $97.35 \pm 0.67$ \\
\midrule
RLT-2 chunk4 & 4 + 4 & $100.00 \pm 0.00$ & $99.87 \pm 0.23$ & $100.00 \pm 0.00$ & $2.47 \pm 1.19$ & $73.87 \pm 44.47$ & $96.40 \pm 1.09$ \\
RLT-2 chunk4 & 5 + 3 & $100.00 \pm 0.00$ & $83.12 \pm 29.13$ & $99.48 \pm 0.23$ & $1.04 \pm 0.60$ & $71.74 \pm 45.58$ & $98.00 \pm 0.60$ \\
RLT-2 chunk4 & 6 + 2 & $100.00 \pm 0.00$ & $95.49 \pm 7.82$ & $99.48 \pm 0.23$ & $1.95 \pm 1.70$ & $98.91 \pm 1.33$ & $97.09 \pm 1.66$ \\
RLT-2 chunk4 & 7 + 1 & $100.00 \pm 0.00$ & $78.43 \pm 16.65$ & $98.83 \pm 1.70$ & $0.65 \pm 0.23$ & $97.70 \pm 2.56$ & $97.01 \pm 1.38$ \\
RLT-2 chunk4 & 8 + 0 & $100.00 \pm 0.00$ & $76.52 \pm 23.05$ & $99.48 \pm 0.90$ & $0.78 \pm 0.78$ & $98.78 \pm 0.08$ & $96.57 \pm 0.42$ \\
\midrule
GPT & 8 & $100.00 \pm 0.00$ & $94.84 \pm 3.43$ & $99.09 \pm 0.23$ & $0.52 \pm 0.23$ & $99.13 \pm 0.78$ & $97.01 \pm 1.04$ \\
\bottomrule
\end{tabular}
}
\end{table}

\begin{figure}[p]
\centering
\includegraphics[width=\linewidth]{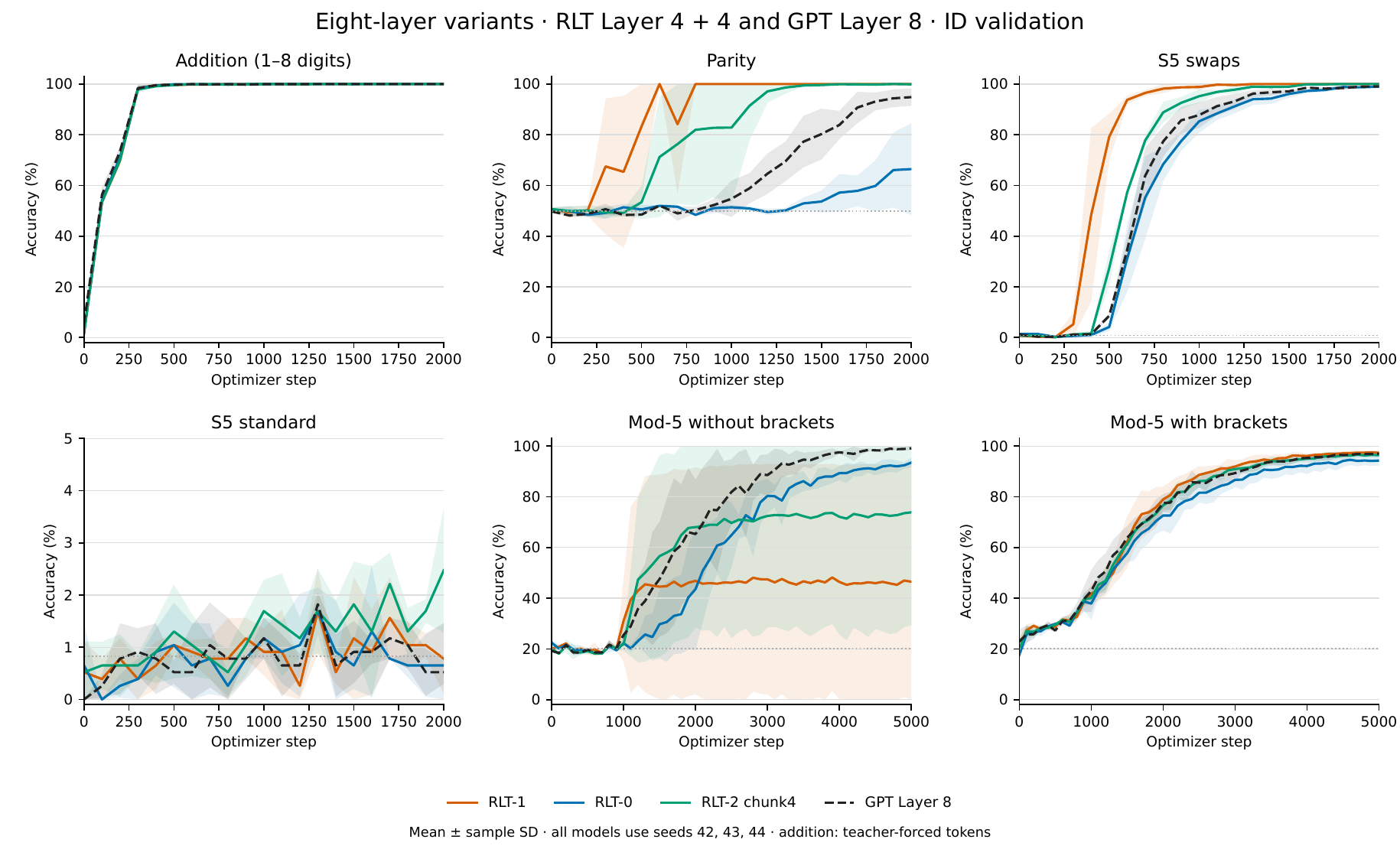}
\caption{\textbf{Feedback variants at $4+4$: validation during training.}
Lines and bands show mean $\pm$ sample SD over seeds 42--44 for every model.
Addition measures teacher-forced answer tokens; other tasks measure final answers or states.
Bands are clipped to 0--100\%; standard $S_5$ uses a narrower vertical scale.}
\label{fig:variant-validation}
\end{figure}

Section~\ref{sec:experiments} discusses the parity and swaps-$S_5$ results; Table~\ref{tab:variant-generalization} gives the $4+4$ values at selected lengths, including both mod-5 tasks.
All 30 RLT-0 and RLT-2 addition runs reach 100\% teacher-forced token accuracy on mixed 1--8-digit validation at step 2,000.
Their length generalization beyond eight digits was not evaluated; the 9--32-digit figures compare RLT-1 with the Transformer.

\begin{table}[t]
\centering
\caption{\textbf{Best-checkpoint test accuracy at $4+4$.}
Final-answer/state accuracy (\%), mean $\pm$ sample SD over three seeds, on the same 1,024 examples per length.
Standard $S_5$ at length 32 is in distribution; the other rows are out of distribution.}
\label{tab:variant-generalization}
\small
\begin{tabular}{lrrrr}
\toprule
Task / length & RLT-1 & RLT-0 & RLT-2 chunk4 & GPT Layer 8 \\
\midrule
Parity / 64 & $100.00 \pm 0.00$ & $50.23 \pm 3.80$ & $98.99 \pm 1.66$ & $48.47 \pm 1.18$ \\
$S_5$ swaps / 48 & $100.00 \pm 0.00$ & $7.26 \pm 3.17$ & $68.29 \pm 8.75$ & $22.10 \pm 2.76$ \\
$S_5$ standard / 32 & $1.53 \pm 0.46$ & $1.20 \pm 0.72$ & $1.37 \pm 0.87$ & $0.91 \pm 0.15$ \\
Mod-5 flat / 63 & $44.43 \pm 43.81$ & $21.84 \pm 2.17$ & $50.78 \pm 38.02$ & $33.20 \pm 2.33$ \\
Mod-5 brackets / 64 & $64.10 \pm 3.49$ & $39.45 \pm 0.61$ & $62.76 \pm 4.42$ & $46.71 \pm 1.21$ \\
\bottomrule
\end{tabular}

\end{table}

\begin{figure}[p]
\centering
\includegraphics[width=\linewidth]{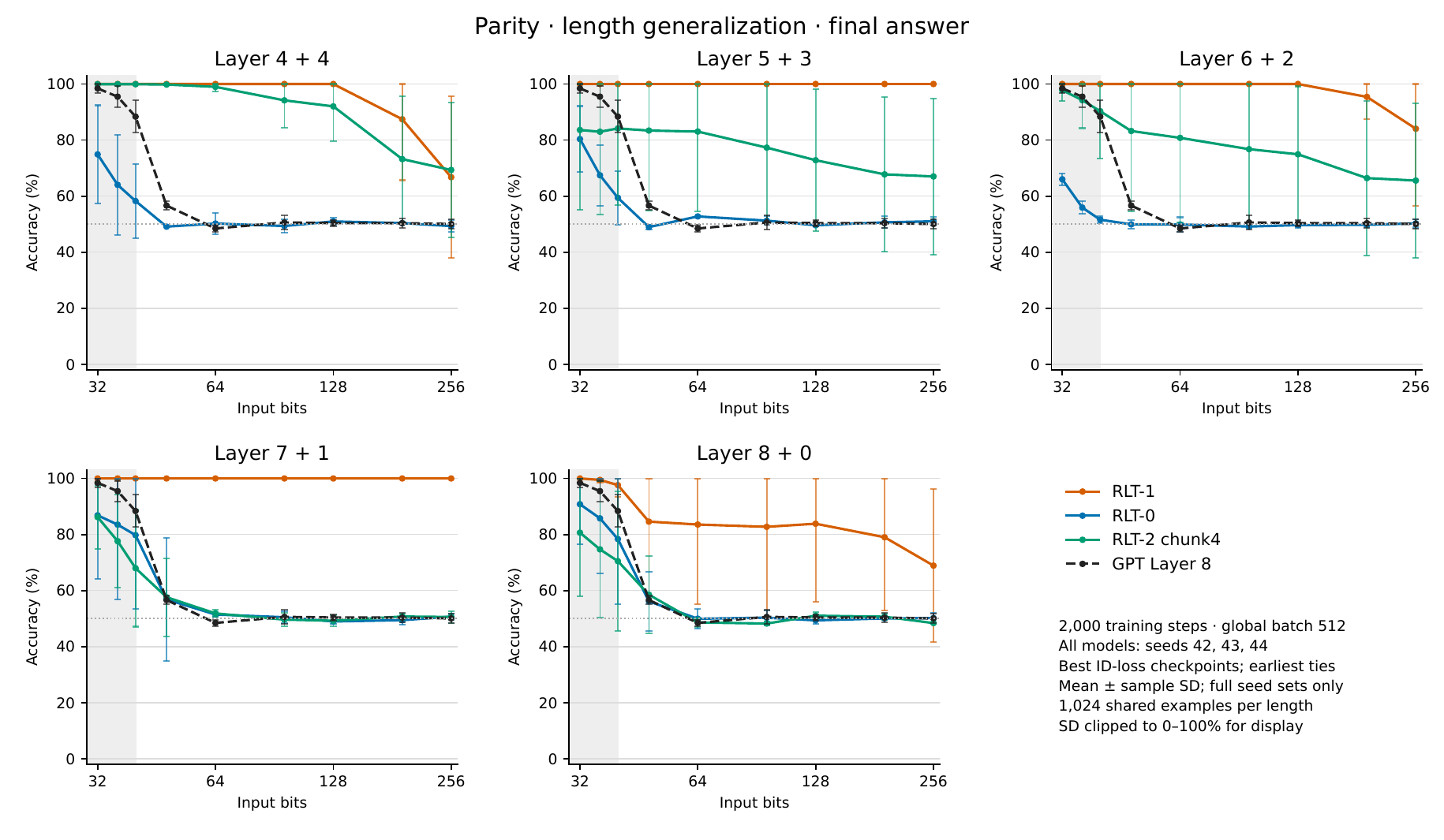}
\caption{\textbf{Parity length generalization across feedback variants and depth splits.}
All models use three seeds, 2,000 training steps, and their best ID-loss checkpoint.
Error bars show sample SD, clipped to 0--100\% for display; gray shading marks training lengths and the dotted line marks chance.}
\label{fig:variant-parity}
\end{figure}

\begin{figure}[p]
\centering
\includegraphics[width=\linewidth]{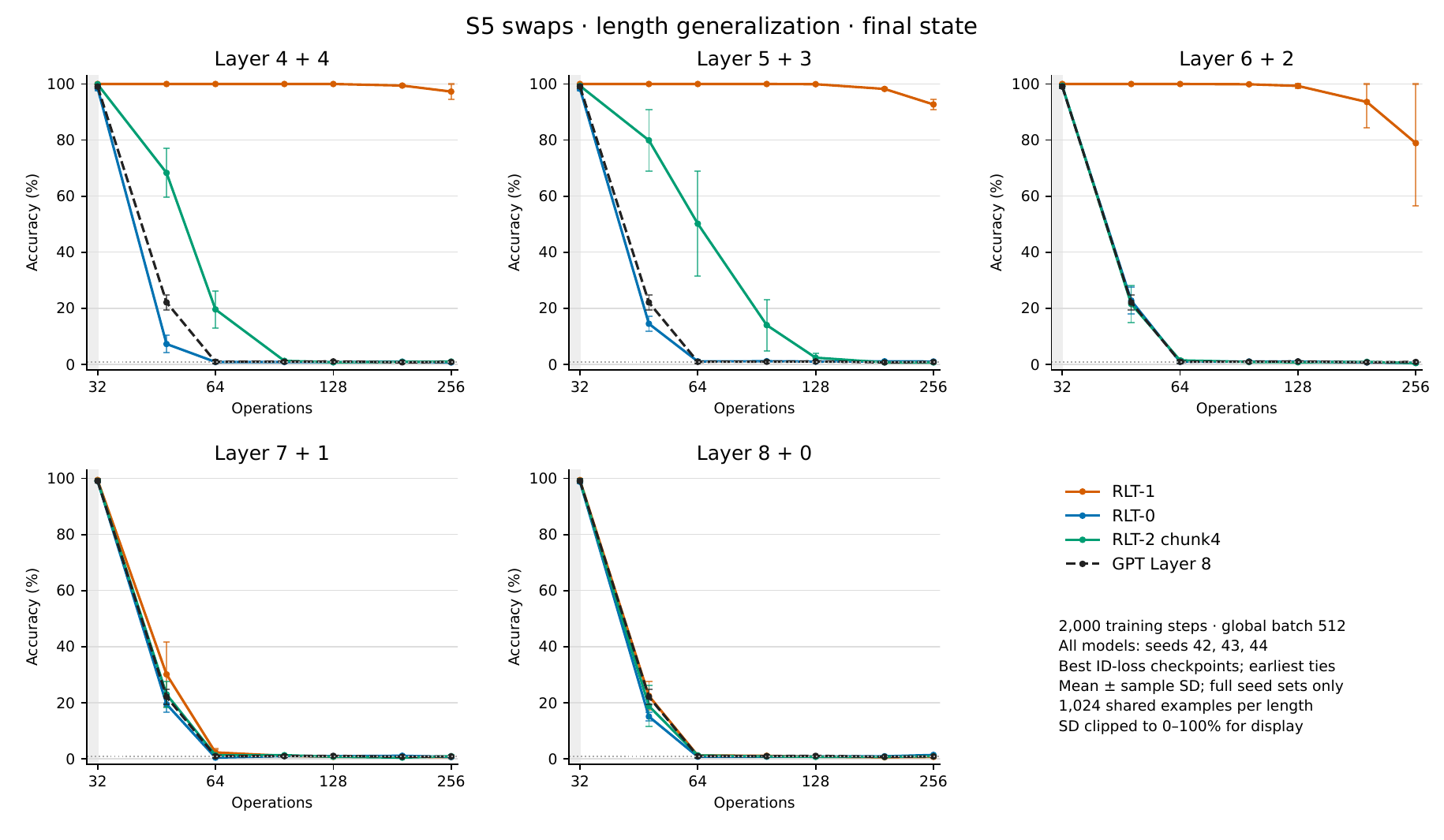}
\caption{\textbf{Swaps-$S_5$ length generalization across feedback variants.}
Final-state accuracy, mean $\pm$ sample SD over three seeds, using each run's best ID-loss checkpoint from 2,000 training steps.
Error bars are clipped to 0--100\% for display.}
\label{fig:variant-s5-swaps}
\end{figure}

\begin{figure}[p]
\centering
\includegraphics[width=\linewidth]{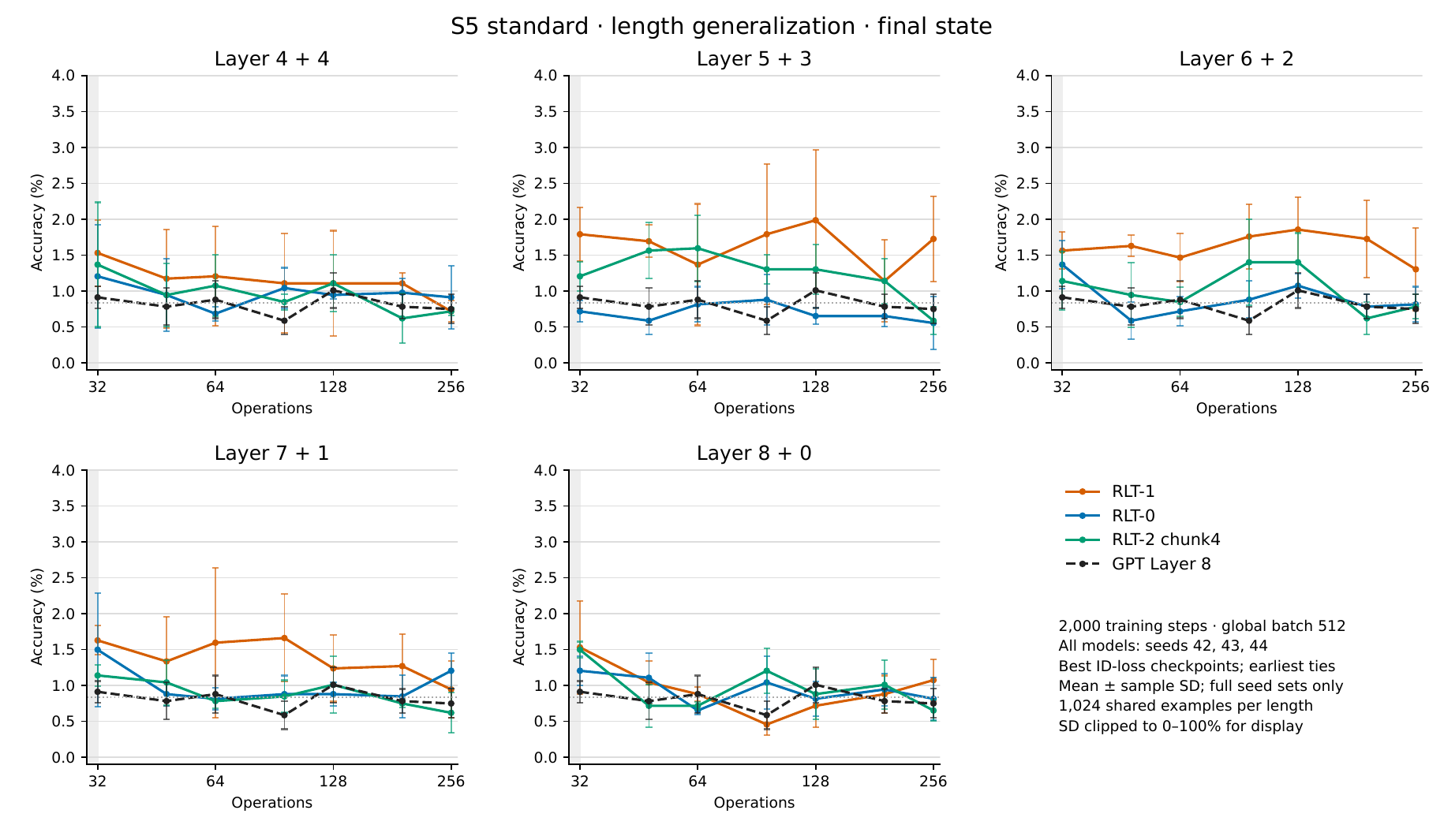}
\caption{\textbf{Standard $S_5$ length generalization across feedback variants.}
Final-state accuracy, mean $\pm$ sample SD over three seeds; all models train for 2,000 steps.
The vertical axis is expanded near chance ($1/120$).}
\label{fig:variant-s5-standard}
\end{figure}

\clearpage
\subsection{Mod-5 feedback variants at 5,000 steps}
\label{app:mod5-variants}
All 96 mod-5 runs train for 5,000 updates.
RLT-0, RLT-1, and RLT-2 chunk4 each contribute 30 runs across two tasks, five splits, and seeds 42--44.
The Transformer contributes six runs, with the same three seeds on both tasks.

Flat mod-5 remains sensitive to initialization after 5,000 steps.
For chunk4 $4+4$, best-checkpoint ID accuracy at length 33 is 100.00\%, 20.21\%, and 98.73\% for seeds 42, 43, and 44.
At length 63, its three-seed mean is $50.78\pm38.02\%$, compared with $44.43\pm43.81\%$ for RLT-1, $21.84\pm2.17\%$ for RLT-0, and $33.20\pm2.33\%$ for GPT Layer 8.

On bracketed mod-5 at length 64, RLT-1 Layer $4 + 4$ reaches $64.10\pm3.49\%$, chunk4 $62.76\pm4.42\%$, RLT-0 $39.45\pm0.61\%$, and GPT Layer 8 $46.71\pm1.21\%$.

In the seed-42 CPU timing comparison at $4+4$, chunk4 training steps are 2.27 times as fast as RLT-1, and RLT-0 steps 4.17--4.30 times as fast, across the two mod-5 tasks.
These single-seed means cover steps 1,001--2,000 on the shared cluster and exclude validation, checkpoint writes, and logging.
Frozen scores, per-seed records, provenance, and figure regeneration instructions are in \texttt{data/variants-20260925/}.

\begin{figure}[p]
\centering
\includegraphics[width=\linewidth]{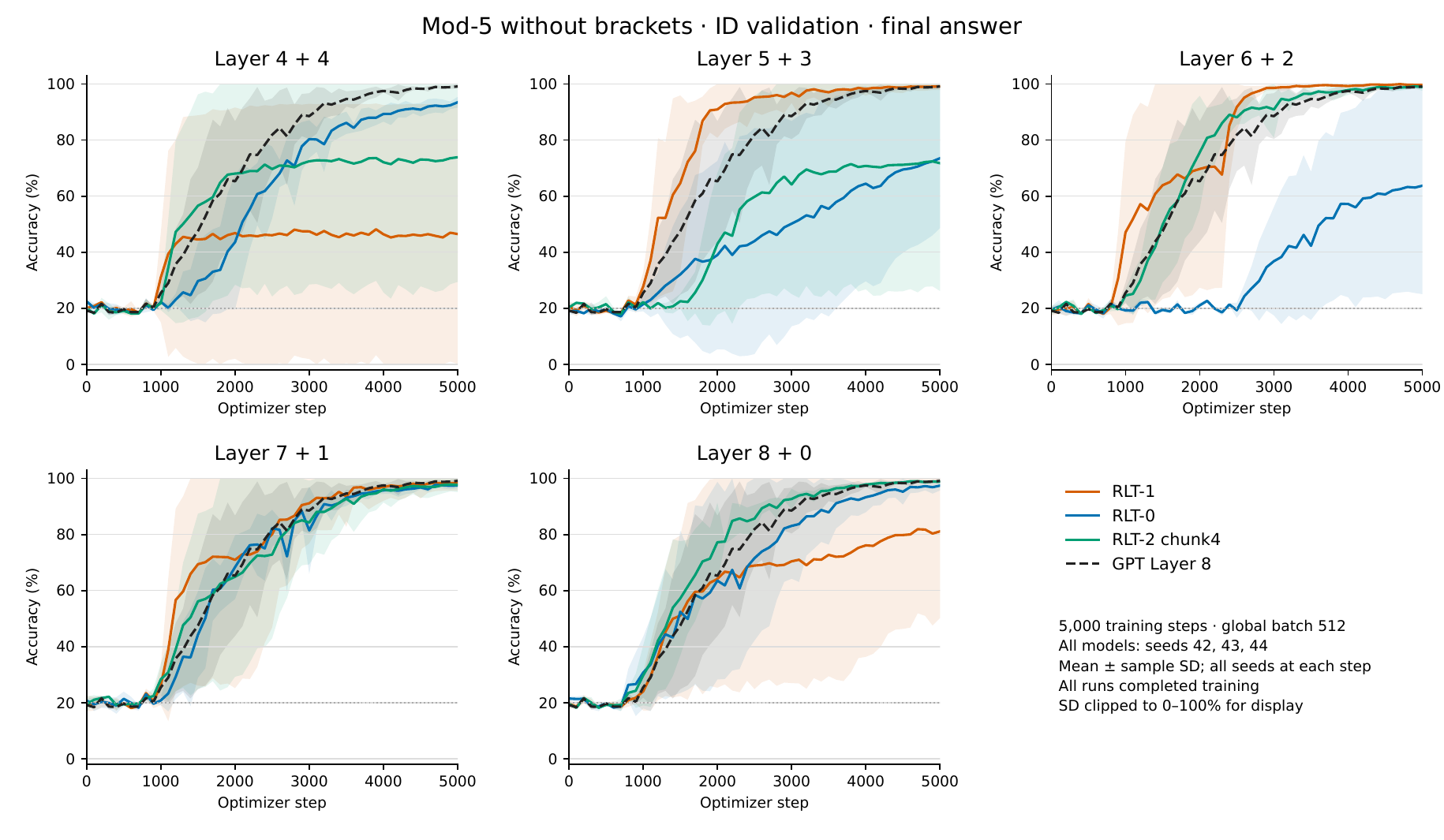}
\caption{\textbf{Mod-5 without brackets: ID validation.}
Mean $\pm$ sample SD over seeds 42--44 for every model; all runs train for 5,000 steps.
Bands are clipped to 0--100\%.}
\label{fig:mod5-no-brackets-validation}
\end{figure}

\begin{figure}[p]
\centering
\includegraphics[width=\linewidth]{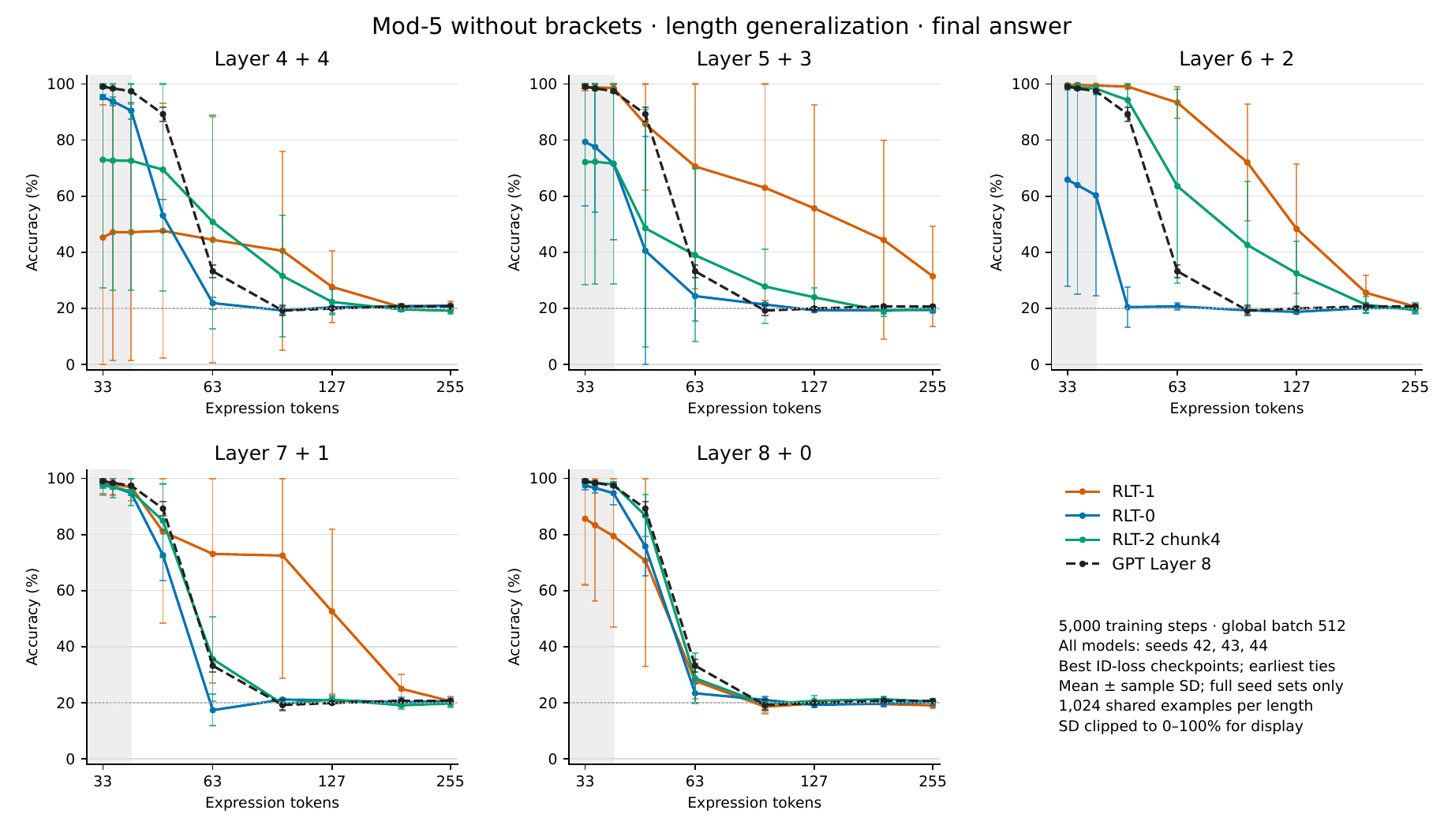}
\caption{\textbf{Mod-5 without brackets: length generalization.}
5,000-step runs at best ID-loss checkpoints; every curve shows mean $\pm$ sample SD over seeds 42--44.
Error bars are clipped to 0--100\% for display.}
\label{fig:mod5-no-brackets-generalization}
\end{figure}

\begin{figure}[p]
\centering
\includegraphics[width=\linewidth]{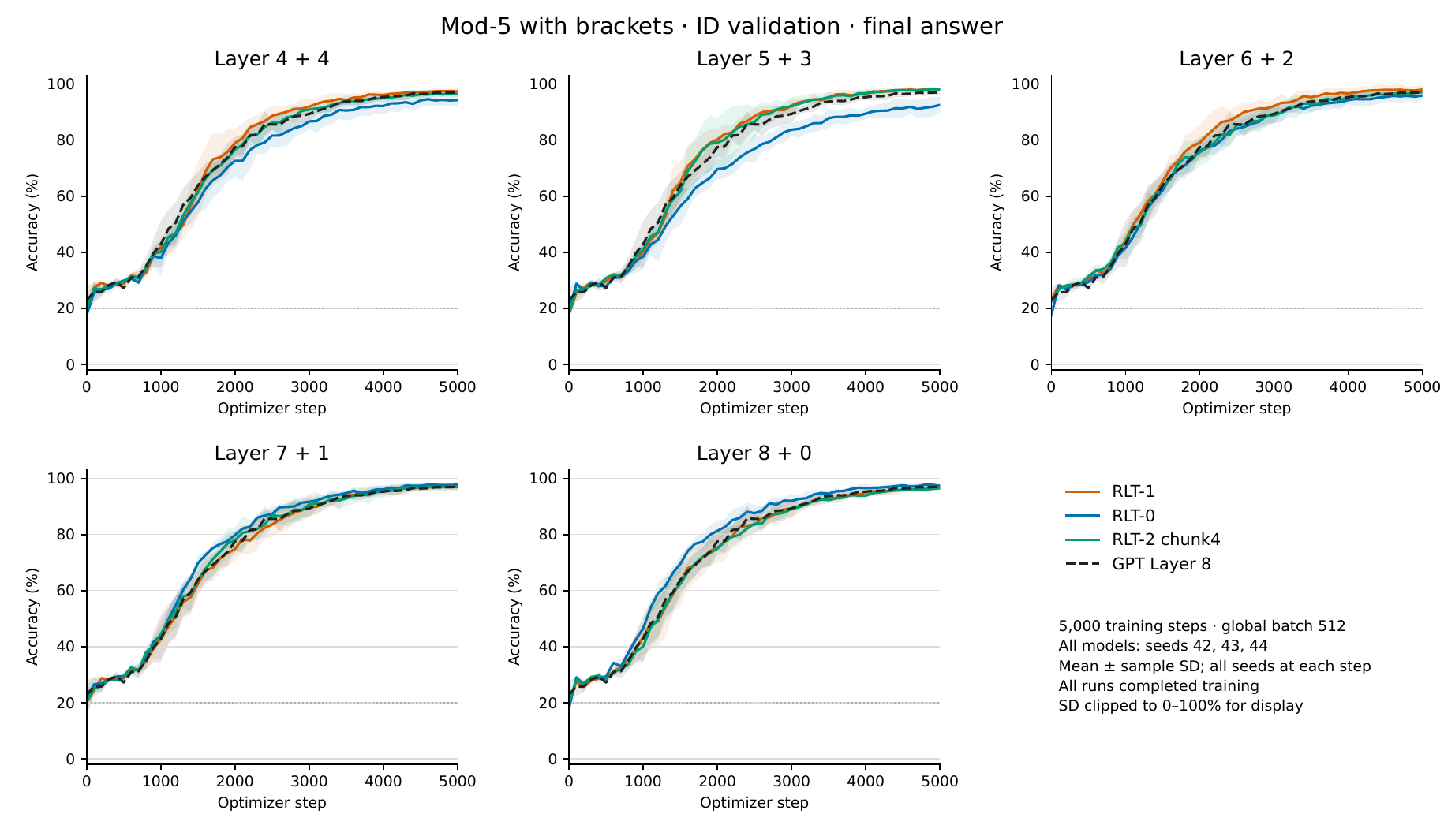}
\caption{\textbf{Mod-5 with brackets: ID validation.}
Mean $\pm$ sample SD over seeds 42--44 for every model; all runs train for 5,000 steps.
Bands are clipped to 0--100\%.}
\label{fig:mod5-with-brackets-validation}
\end{figure}

\begin{figure}[p]
\centering
\includegraphics[width=\linewidth]{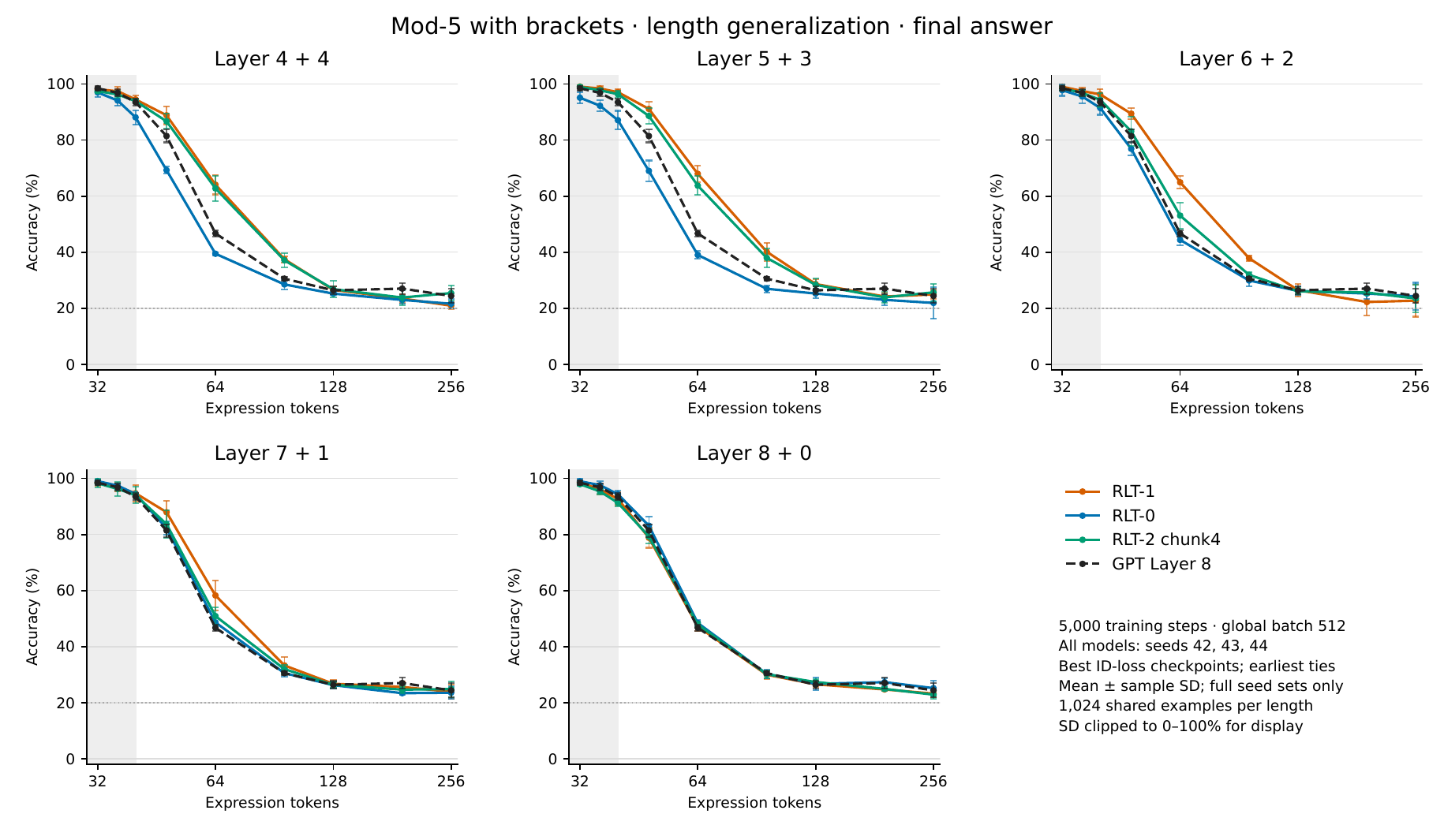}
\caption{\textbf{Mod-5 with brackets: length generalization.}
5,000-step runs at best ID-loss checkpoints; every curve shows mean $\pm$ sample SD over seeds 42--44.
Error bars are clipped to 0--100\% for display.}
\label{fig:mod5-with-brackets-generalization}
\end{figure}

\clearpage
\subsection{Addition feedback scale at a fixed learning-rate schedule}
\label{app:addition-feedback}
We vary only the RLT-1 feedback scale $\alpha$, keeping the initialization, architecture, training data, and optimizer schedule fixed within each split and seed.
All runs train on mixed 1--8-digit addition for 2,000 updates, with global batch 512, microbatch 32, and TBPTT128.
The learning rate warms up for 200 steps to $10^{-4}$ and decays to $5\times10^{-6}$.
We test $\alpha=0.03$ at every split and $\alpha=0.01$ at every split except $7+1$, with seeds 42, 43, and 44.
The comparison includes the $\alpha=0.1$ RLT-1 runs and GPT Layer 8 at the same seeds, for 45 models in total.

We select each model's checkpoint by minimum ID validation loss, taking the earliest step on exact ties; the $7+1$ run with $\alpha=0.1$ and seed 44 selects step 1,900, and the other 44 models select step 2,000.
Figure~\ref{fig:addition-feedback} and Table~\ref{tab:addition-feedback} report teacher-forced answer-token accuracy on 256 shared examples at each of seven widths from 9 to 32 digits, as mean $\pm$ sample SD over the three seeds.
Counts include answer formatting and EOS, and exclude prompts and padding.

No tested feedback reduction improves accuracy at every width for any split.
Across the 63 comparisons with $\alpha=0.1$ at the same split and width, 37 means rise and 26 fall, by $-3.86$ to $+5.60$ percentage points; only three differences exceed both sample SDs.
For $8+0$, $\alpha=0.01$ raises the 12-digit mean from $29.49\pm4.38\%$ to $32.21\pm5.13\%$, while the 16-digit mean moves from $24.92\pm0.39\%$ to $24.66\pm4.88\%$.
At 32 digits, RLT-1 means range from 14.44\% to 15.91\%, compared with $16.84\pm1.45\%$ for GPT Layer 8.

\begin{table}[htbp]
\centering
\caption{\textbf{Addition feedback-scale ablation, seeds 42--44.}
Teacher-forced answer-token accuracy (\%) at selected widths, mean $\pm$ sample SD over three seeds, after 2,000 training steps with the same learning-rate schedule.
Each width uses 256 shared examples; all seven widths are available in the data repository.}
\label{tab:addition-feedback}
{\small\begin{tabular}{lrrrrr}
\toprule
Model & $\alpha$ & 9 digits & 12 digits & 16 digits & 32 digits \\
\midrule
RLT-1 Layer 4 + 4 & 0.01 & $47.25 \pm 4.04$ & $33.11 \pm 4.56$ & $25.93 \pm 1.88$ & $14.44 \pm 0.17$ \\
RLT-1 Layer 4 + 4 & 0.03 & $48.75 \pm 5.34$ & $32.00 \pm 2.07$ & $26.22 \pm 0.74$ & $14.74 \pm 1.25$ \\
RLT-1 Layer 4 + 4 & 0.1 & $46.77 \pm 1.34$ & $30.72 \pm 0.58$ & $25.28 \pm 1.48$ & $15.30 \pm 0.44$ \\
RLT-1 Layer 5 + 3 & 0.01 & $57.96 \pm 0.67$ & $32.36 \pm 2.93$ & $23.01 \pm 5.22$ & $15.29 \pm 2.19$ \\
RLT-1 Layer 5 + 3 & 0.03 & $56.30 \pm 2.41$ & $32.60 \pm 5.61$ & $23.54 \pm 3.75$ & $14.66 \pm 1.49$ \\
RLT-1 Layer 5 + 3 & 0.1 & $58.51 \pm 8.12$ & $34.97 \pm 4.54$ & $24.20 \pm 5.26$ & $14.89 \pm 2.27$ \\
RLT-1 Layer 6 + 2 & 0.01 & $56.52 \pm 9.87$ & $30.97 \pm 4.64$ & $26.35 \pm 2.98$ & $15.21 \pm 1.46$ \\
RLT-1 Layer 6 + 2 & 0.03 & $63.85 \pm 10.15$ & $30.39 \pm 2.85$ & $27.18 \pm 4.26$ & $14.86 \pm 0.65$ \\
RLT-1 Layer 6 + 2 & 0.1 & $58.25 \pm 10.29$ & $29.22 \pm 2.28$ & $23.92 \pm 1.58$ & $15.74 \pm 1.58$ \\
RLT-1 Layer 7 + 1 & 0.03 & $66.43 \pm 2.53$ & $30.10 \pm 2.29$ & $26.27 \pm 1.42$ & $15.91 \pm 0.63$ \\
RLT-1 Layer 7 + 1 & 0.1 & $68.05 \pm 3.64$ & $29.80 \pm 4.03$ & $26.45 \pm 4.31$ & $15.59 \pm 1.57$ \\
RLT-1 Layer 8 + 0 & 0.01 & $62.20 \pm 9.85$ & $32.21 \pm 5.13$ & $24.66 \pm 4.88$ & $15.45 \pm 0.66$ \\
RLT-1 Layer 8 + 0 & 0.03 & $59.18 \pm 4.47$ & $28.81 \pm 3.57$ & $25.92 \pm 1.10$ & $14.90 \pm 0.54$ \\
RLT-1 Layer 8 + 0 & 0.1 & $62.13 \pm 7.75$ & $29.49 \pm 4.38$ & $24.92 \pm 0.39$ & $15.31 \pm 0.93$ \\
GPT Layer 8 & -- & $58.61 \pm 2.44$ & $26.84 \pm 3.69$ & $26.22 \pm 1.54$ & $16.84 \pm 1.45$ \\
\bottomrule
\end{tabular}
}
\end{table}

\begin{figure}[p]
\centering
\includegraphics[width=\linewidth]{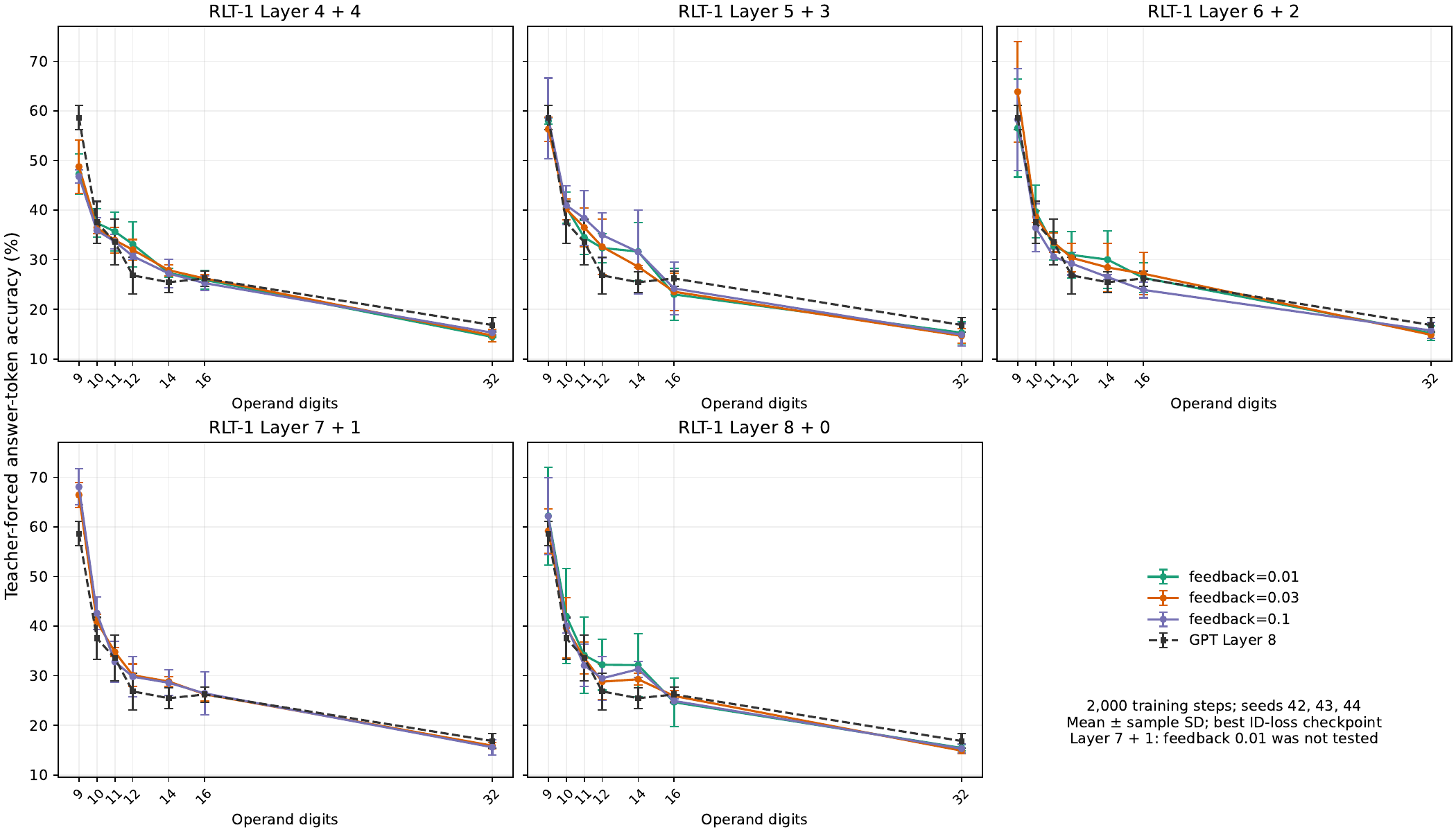}
\caption{\textbf{Addition generalization as feedback scale changes.}
Each panel compares the tested feedback scales with GPT Layer 8 at one RLT-1 split; $\alpha=0.01$ was not tested at $7+1$.
Points and error bars show mean $\pm$ sample SD over seeds 42, 43, and 44; all models use 2,000 training steps and the same learning-rate schedule.
Accuracy measures teacher-forced answer tokens on the same 256 examples at each width.}
\label{fig:addition-feedback}
\end{figure}

\end{document}